\documentclass[twoside,11pt]{article}

\usepackage{jmlr2e}

\usepackage{amsmath,amsbsy}
\usepackage{algpseudocode}
\usepackage{algorithm}
\usepackage{footnote}
\usepackage{color,xcolor}
\usepackage{graphicx,caption,subcaption}
\usepackage[title]{appendix}

\def\abovestrut#1{\rule[0in]{0in}{#1}\ignorespaces}
\def\belowstrut#1{\rule[-#1]{0in}{#1}\ignorespaces}
\def\abovespace{\abovestrut{0.20in}}

\def\belowspace{\belowstrut{0.10in}}

\newcommand{\vct}{\boldsymbol }
\newcommand{\mat}{\mathbf}

\newcommand{\nml}{\mathcal{N}}

\newcommand{\Span}{\mathrm{span}}

\newcommand{\argmax}{\mathrm{argmax}}
\newcommand{\vol}{\mathrm{vol}}
\newcommand{\bern}{\mathrm{Bernoulli}}

\newcommand{\rank}{\mathrm{rank}}
\newcommand{\complete}{\mathrm{mc}}
\newcommand{\css}{\mathrm{css}}
\newcommand{\diag}{\mathrm{diag}}

\newtheorem{thm}{Theorem}
\newtheorem{lem}{Lemma}
\newtheorem{cor}{Corollary}
\newtheorem{prop}{Proposition}

\newtheorem{defn}{Definition}

\ShortHeadings{Column Subset Selection with Missing Data via Selective Sampling}{Wang and Singh}
\firstpageno{1}

\begin{document}

\jmlrheading{18}{2017}{1-42}{5/15; Revised 9/16}{1/18}{15-233}{Yining Wang and Aarti Singh}

\title{Provably Correct Algorithms for Matrix Column Subset Selection with Selectively Sampled Data}

\author{\name Yining Wang \email yiningwa@cs.cmu.edu \\
       \name Aarti Singh \email aarti@cs.cmu.edu\\
       \addr Machine Learning Department, School of Computer Science\\
       Carnegie Mellon University\\
       5000 Forbes Avenue, Pittsburgh, PA 15213, USA}

\editor{Sujay Sanghavi}

\maketitle

\begin{abstract}
We consider the problem of matrix column subset selection, which selects a subset of columns from an input matrix
such that the input can be well approximated by the span of the selected columns.
Column subset selection has been applied to numerous real-world data applications such as population genetics summarization, 
electronic circuits testing and recommendation systems.
In many applications the complete data matrix is unavailable and one needs to select representative columns by inspecting only a small portion 
of the input matrix.
In this paper we propose the first provably correct column subset selection algorithms for partially observed data matrices.
Our proposed algorithms exhibit different merits and limitations in terms of statistical accuracy, computational efficiency, sample complexity and sampling schemes,
which provides a nice exploration of the tradeoff between these desired properties for column subset selection.
The proposed methods employ the idea of feedback driven sampling and
are inspired by several sampling schemes previously introduced for low-rank matrix approximation tasks \citep{cur_relative,norm2-css,approximate-volume-sampling,power-adaptivity}.
Our analysis shows that, under the assumption that the input data matrix has incoherent rows but possibly coherent columns,
all algorithms provably converge to the best low-rank approximation of the original data as number of selected columns increases.
Furthermore, two of the proposed algorithms enjoy a relative error bound,
which is preferred for column subset selection and matrix approximation purposes.
We also demonstrate through both theoretical and empirical analysis the power of feedback driven sampling compared to uniform random sampling
on input matrices with highly correlated columns.
\end{abstract}


\section{Introduction}

Given a matrix $\mat M\in\mathbb R^{n_1\times n_2}$, the \emph{column subset selection} problem aims to find $s$ exact columns in $\mat M$
that capture as much of $\mat M$ as possible.
More specifically, we want to select $s$ columns of $\mat M$ to form a column sub-matrix $\mat C\in\mathbb R^{n_1\times s}$ to minimize the norm of the following residue
\begin{equation}
\min_{\mat X\in\mathbb R^{s\times n_2}}{\|\mat M-\mat C\mat X\|_{\xi}} = \|\mat M-\mat C\mat C^\dagger\mat M\|_{\xi},
\label{eq_residue}
\end{equation}
where $\mat C^\dagger$ is the Moore-Penrose pseudoinverse of $\mat C$
and $\xi=2$ or $F$ denotes the spectral or Frobenius norm.
In this paper we mainly focus on the Frobenius norm, as was the case in previous theoretical analysis for sampling based column subset selection algorithms
\citep{cur_relative,norm2-css,approximate-volume-sampling,volume-sampling-css}.
To evaluate the performance of column subset selection,
one compares the residue norm defined in Eq. (\ref{eq_residue}) with $\|\mat M-\mat M_k\|_{\xi}$,
where $\mat M_k$ is the best rank-$k$ approximation of $\mat M$.
Usually the number of selected columns $s$ is larger than or equal to the target rank $k$.
Two forms of error guarantee are common: additive error guarantee in Eq. (\ref{eq_additive_error}) and relative error guarantee in Eq. (\ref{eq_relative_error}),
with $0<\epsilon<1$ and $c>1$ (ideally $c=1+\epsilon$).
\begin{eqnarray}
\|\mat M-\mat C\mat C^\dagger\mat M\|_{\xi} &\leq& \|\mat M-\mat M_k\|_{\xi} + \epsilon\|\mat M\|_{F};\label{eq_additive_error}\\
\|\mat M-\mat C\mat C^\dagger\mat M\|_{\xi} &\leq& c\|\mat M-\mat M_k\|_{\xi}.\label{eq_relative_error}
\end{eqnarray}
In general, relative error bound is much more appreciated because $\|\mat M\|_{\xi}$ is usually large in practice,
while $\|\mat M-\mat M_k\|_2$ is expected to be small when the goal is low-rank approximation.
In addition, when $\mat M$ is an exact low-rank matrix Eq. (\ref{eq_relative_error}) implies perfect reconstruction,
while the error in Eq. (\ref{eq_additive_error}) remains non-zero.
The column subset selection problem can be considered as a form of \emph{unsupervised feature selection} or \emph{prototype selection}, 
which arises frequently in the analysis of large data sets.
For example, column subset selection has been applied to various tasks such as summarizing population genetics,
testing electronic circuits, recommendation systems, etc.
Interested readers should refer to \citep{sampling-rrqr,missing-css-blockomp} for further motivations.

Many methods have been proposed for the column subset selection problem \citep{rrqr1,rrqr2,norm2-css,volume-sampling-css,cur_relative,boutsidis2014near}.
An excellent summarization of these methods and their theoretical guarantee is available in Table 1 in \citep{sampling-rrqr}.
Most of these methods can be roughly categorized into two classes.
One class of algorithms are based on \emph{rank-revealing QR} (RRQR) decomposition \citep{rrqr1,rrqr2}
and it has been shown in \citep{sampling-rrqr} that RRQR is nearly optimal in terms of residue norm under the $s=k$ setting,
that is, exact $k$ columns are selected to reconstruct an input matrix.
On the other hand, sampling based methods \citep{norm2-css,volume-sampling-css,cur_relative} try to select columns by sampling from certain distributions over all columns of an input matrix.
{Extension of sampling based methods to general low-rank matrix approximation problems is also investigated \citep{cohen2015uniform,bhojanapalli2015tighter}.}
These algorithms are much faster than RRQR and achieves comparable performance if the sampling distribution is carefully selected 
and slight over-sampling (i.e., $s>k$) is allowed \citep{volume-sampling-css,cur_relative}.
In \citep{sampling-rrqr} sampling based and RRQR based algorithms are unified to arrive at an efficient column subset selection method that uses exactly $s=k$ columns and is nearly optimal.

Although the column subset selection problem with access to the full input matrix has been extensively studied,
often in practice it is hard or even impossible to obtain the complete data.
For example, for the genetic variation detection problem it could be expensive and time-consuming
to obtain full DNA sequences of an entire population.
Several heuristic algorithms have been proposed recently for column subset selection with missing data,
including the Block OMP algorithm \citep{missing-css-blockomp} and the group Lasso formulation explored in \citep{missing-css-glasso}.
Nevertheless, no theoretical guarantee or error bounds have been derived for these methods.
The presence of missing data poses new challenges for column subset selection,
as many well-established algorithms seem incapable of handling missing data in an elegant way.
Below we identify a few key challenges that prevent application of previous theoretical results on column subset selection
under the missing data setting:
\begin{itemize}
\item \textbf{Coherent matrix design}: most previous results on the completion or recovery of low rank matrices with incomplete data
assume the underlying data matrix is \emph{incoherent} \citep{simpler-matrix-completion,noisy-matrix-completion1,noisy-matrix-completion2},
which intuitively assumes all rows and columns in the data matrix are weakly correlated.
\footnote{The precise definition of incoherence is given in Section \ref{subsec:notation}.}
On the other hand, previous algorithms on column subset selection and matrix CUR decomposition
spent most efforts on dealing with coherent matrices \citep{volume-sampling-css,cur_relative,sampling-rrqr,optimal-cur}.
In fact, one can show that under standard incoherence assumptions of matrix completion algorithms a high-quality column subset can be obtained
by sampling each column uniformly at random, which trivializes the problem \citep{missing-cur-rongjin}.
Such gap in problem assumptions renders column subset selection on incomplete coherent matrices particularly difficult.
In this paper, we explore the possibility of a weaker incoherence assumption that bridges the gap.
We present and discuss detailed assumptions considered in this paper in Sec.~\ref{subsec:asmp}.

\item \textbf{Limitation of existing sampling schemes}:
previous matrix completion methods usually assume the observed data are sampled uniformly at random.
However, in \citep{power-adaptivity} it is proved that uniform sampling (in fact any sampling scheme with apriori fixed sampling distribution)
is not sufficient to complete a coherent matrix.
Though in \citep{complete-any-matrix} a provably correct sampling scheme was proposed for any matrix based on statistical leverage scores,
which is also the key ingredient of many previous column subset selection and matrix CUR decomposition algorithms \citep{cur_relative,sampling-rrqr,optimal-cur},
it is very difficult to approximate the leverage scores of an incomplete coherent matrix.
Common perturbation results on singular vector space (e.g., Wedin's theorem) fail because
closeness between two subspaces does not imply closeness in their leverage scores
since the latter are defined in an infinity norm manner (see Section \ref{subsec:lscore_def} for details).

\item \textbf{Limitation of zero filling}:
A straightforward algorithm for missing data column subset selection
is to first fill all unobserved entries with zero and then properly scale the observed ones so that
the completed matrix is consistent with the underlying data matrix in expectation \citep{l2-lowrank-approx,l1-lowrank-approx}.
Column subset selection algorithms designed for fully observed data could be applied afterwards on the zero-filled matrix.
However, the zero filling procedure can change the underlying subspace of a matrix drastically \citep{subspace-detection}
and usually leads to additive error bounds as in Eq. (\ref{eq_additive_error}).
To achieve stronger relative error bounds we need an algorithm that goes beyond the zero filling idea.
\end{itemize}

In this paper, we propose three column subset selection algorithms based on the idea of \emph{active sampling} of the input matrix.
In our algorithms, observed matrix entries are chosen sequentially and in a feedback-driven manner.
We motivate this sampling setting from both practical and theoretical perspectives. In applications where each entry of a data matrix $\mat M$ represents results from an expensive or time-consuming experiment, it makes sense to carefully select which entry to query (experiment), possibly in a feedback-driven manner, so as to reduce experimental cost. For example, if $\mat M$ has drugs as its columns and targets (proteins) as its rows, it makes sense to cautiously select drug-target pairs for sequential experimental study in order to find important drugs/targets with typical drug-target interactions. From a theoretical perspective, we show in Section 7.1 that no passive sampling scheme is capable of achieving relative-error column subset selection with high probability, even if the column space of $\mat M$ is incoherent. Such results suggest that active/adaptive sampling is to some extent unavoidable, unless both row and column spaces of $\mat M$ are incoherent.

We also remark that the algorithms we consider make very few measurements of the input matrix,
which differs from previous feedback-driven re-sampling methods in the theoretical computer science literature (e.g., \citep{adaptive-cur}) that requires access to the entire input matrix.
Active sampling has been shown to outperform all passive schemes in several settings (cf. \citep{distilled-sensing,kolar2011minimax}),
and furthermore it works for completion of matrices with incoherent rows/columns under which passive learning provably fails \citep{akshay-nips,power-adaptivity}.
To the best of our knowledge, the algorithms proposed in this paper are the first column subset selection algorithms for coherent matrices that enjoy theoretical error guarantee with missing data, whether passive or active.
Furthermore, two of our proposed methods achieve relative error bounds.


{
\subsection{Assumptions}\label{subsec:asmp}

Completing/approximating partially observed low-rank matrices using a subset of columns requires certain assumptions on the input data matrix $\mat M$ \citep{noisy-matrix-completion1,complete-any-matrix,simpler-matrix-completion,missing-cur-rongjin}.
To see this, consider the extreme-case example where the input data matrix $\mat M$ consists of \emph{exactly} one non-zero element
(i.e., $\mat M_{ij}=1\{i=i^*,j=j^*\}$ for some $i^*\in [n_1]$ and $j^*\in[n_2]$).
In this case, the relative approximation quality $c=\|\mat M-\mat C\mat C^\dagger\mat M\|_\xi/\|\mat M-\mat M_1\|_\xi$ in Eq.~(\ref{eq_relative_error})
would be infinity if column $j^*$ is not selected in $\mat C$.
In addition, it is clearly impossible to correctly identify $j^*$ using $o(n_1n_2)$ observations even with active sampling strategies.
Therefore, additional assumptions on $\mat M$ are required to provably approximate a partially observed matrix using column subsets.

In this work we consider the assumption that the top-$k$ \emph{column space} of the input matrix $\mat M$ is incoherent (detailed mathematical definition given in Sec.~\ref{subsec:lscore_def}),
while placing no incoherence or spikiness assumptions on the actual \emph{columns}, rows or the row space of $\mat M$.
In addition to the necessity of incoherence assumptions for incomplete matrix approximation problems discussed above, 
we further motivate the ``one-sided'' incoherence assumption from two perspectives: 
\begin{itemize}
\item[-] Column subset selection with incomplete observation remains a non-trivial problem even if the column space is assumed to be incoherent.
Due to the possible heterogeneity of the columns, 
naive methods such as column subsets sampled uniformly at random are in general bad approximations of the original data matrix $\mat M$.
Existing column subset selection algorithms for fully-observed matrices also need to be majorly revised to accommodate missing matrix components.

\item[-] Compared to existing work on approximating low-rank incomplete matrices, our assumptions (one-sided incoherence) are arguably weaker. 
\cite{missing-cur-rongjin} analyzed matrix CUR approximation of partially observed matrices, but assumed that both column and row spaces are incoherent; 
\cite{power-adaptivity} derived an adaptive sampling procedure to complete a low-rank matrix with only one-sided incoherence assumptions, but only achieved additive error bounds for noisy low-rank matrices.

\item[-] Finally, the one-sided incoherence assumption is reasonable in a number of practical scenarios. For example, in the application of drug-target interaction prediction,
the one-sided incoherence assumption allows for highly specialized or diverse drugs while assuming some predictability between target protein responses.
\end{itemize}

}

\subsection{Our contributions}

The main contribution of this paper is three provably correct algorithms for column subset selection, which are inspired by existing work on column subset selection for fully-observed matrices, but only inspect a small portion of the input matrix.
The sampling schemes for the proposed algorithms and their main merits and drawbacks are summarized below:
\begin{enumerate}
\item \textbf{Norm sampling}: The algorithm is simple and works for any input matrix with incoherent column subspace.
However, it only achieves an additive error bound as in Eq. (\ref{eq_additive_error}).
It is also inferior than the other two proposed methods in terms of residue error on both synthetic and real-world data sets.

\item \textbf{Iterative norm sampling}:
The iterative norm sampling algorithm enjoys relative error guarantees as in Eq. (\ref{eq_relative_error})
at the expense of being much more complicated and computationally expensive.
In addition, its correctness is only proved for low-rank matrices with incoherent column space corrupted with i.i.d. Gaussian noise.

\item \textbf{Approximate leverage score sampling}:
The algorithm enjoys relative error guarantee for general (high-rank) input matrices with incoherent column space.
However, it requires more over-sampling and its error bound is worse than the one for iterative norm sampling on noisy low-rank matrices.
Moreover, to actually reconstruct the data matrix 
\footnote{See Section \ref{subsec:notation} for the distinction between selection and reconstruction.}
the approximate leverage score sampling scheme requires sampling a subset of both entire rows and columns,
while both norm based algorithms only require sampling of some entire columns.
\end{enumerate}

In summary, our proposed algorithms offer a rich, provably correct toolset for column subset selection with missing data.
Furthermore, a comprehensive understanding of the design tradeoffs among statistical accuracy, computational efficiency, sample complexity,
and sampling scheme, etc. is achieved by analyzing different aspects of the proposed methods.
Our analysis could provide further insights into other matrix completion/approximation tasks on partially observed data.

We also perform comprehensive experimental study of column subset selection with missing data
using the proposed algorithms as well as modifications of heuristic algorithms proposed recently \citep{missing-css-blockomp,missing-css-glasso}
on synthetic matrices and two real-world applications:
tagging Single Nucleotide Polymorphisms (tSNP) selection and column based image compression.
Our empirical study verifies most of our theoretical results and reveals a few interesting observations that are previously unknown.
For instance, though leverage score sampling is widely considered as the state-of-the-art for matrix CUR approximation and column subset selection,
our experimental results show that under certain low-noise regimes (meaning that the input matrix is very close to low rank)
iterative norm sampling is more preferred and achieves smaller error.
These observations open new questions and suggest the need for new analysis in related fields, even for the fully observed case.

\subsection{Notations}\label{subsec:notation}

For any matrix $\mat M$ we use $\mat M^{(i)}$ to denote the $i$-th column of $\mat M$.
Similarly, $\mat M_{(i)}$ denotes the $i$-th row of $\mat M$.
All norms $\|\cdot \|$ are $\ell_2$ norms or the matrix spectral norm unless otherwise specified.

We assume the input matrix is of size $n_1\times n_2$, $n=\max(n_1,n_2)$.
We further assume that $n_1\leq n_2$.
We use $\vct x_i=\mat M^{(i)}\in\mathbb R^{n_1}$ to denote the $i$-th column of $\mat M$.
Furthermore, for any column vector $\vct x_i\in\mathbb R^{n_1}$ and index subset $\Omega\subseteq[n_1]$, define the subsampled vector $\vct x_{i,\Omega}$
and the scaled subsampled vector $\mathcal R_{\Omega}(\vct x_i)$ as
\begin{equation}
\vct x_{i,\Omega} = \vct 1_{\Omega}\circ \vct x_i,\quad
\mathcal R_{\Omega}(\vct x_i) = \frac{n_1}{|\Omega|}\vct 1_{\Omega}\circ \vct x_i,
\label{eq_subsample_vector}
\end{equation}
where $\vct 1_{\Omega}\in\mathbb \{0,1\}^{n_1}$ is the indicator vector of $\Omega$ and $\circ$ is the Hadamard product (entrywise product).
We also generalize the definition in Eq. (\ref{eq_subsample_vector}) to matrices by applying the same operator on each column.

We use $\|\mat M-\mat C\mat C^\dagger\mat M\|_\xi$ to denote the \emph{selection error}
and $\|\mat M-\mat C\mat X\|_\xi$ to denote the \emph{reconstruction error}.
The difference between the two types of error is that for selection error an algorithm is only required to output 
indices of the selected columns
while for reconstruction error an algorithm needs to output both the selected columns $\mat C$ and the coefficient matrix $\mat X$
so that $\mat C\mat X$ is close to $\mat M$.
We remark that the reconstruction error always upper bounds the selection error due to Eq. (\ref{eq_residue}).
On the other hand, there is no simple procedure to compute $\mat C^\dagger\mat M$ when $\mat M$ is not fully observed.

\subsection{Outline of the paper}

The paper is organized as follows: in Section \ref{sec:preliminary} we provide background knowledge and review several concepts
that are important to our analysis.
We then present main results of the paper, the three proposed algorithms and their theoretical guarantees in Section \ref{sec:main_result}.
Proofs for main results given in Section \ref{sec:main_result} are sketched in Section \ref{sec:proofs}
and some technical lemmas and complete proof details are deferred to the appendix.
In Section \ref{sec:related_work} we briefly describe previously proposed heuristic based algorithm
for column subset selection with missing data and their implementation details.
Experimental results are presented in Section \ref{sec:experiments}
and we discuss several aspects including the limitation of passive sampling and time complexity of proposed algorithms
in Section \ref{sec:discussion}.

\section{Preliminaries}\label{sec:preliminary}


This section provides necessary background knowledge for the analysis in this paper.
We first review the concept of \emph{coherence}, which plays an important row in sampling based matrix algorithms.
We then summarize three matrix sampling schemes proposed in previous literature.

\subsection{Subspace and vector incoherence}\label{subsec:lscore_def}

Incoherence plays a crucial role in various matrix completion and approximation tasks \citep{simpler-matrix-completion,power-adaptivity,noisy-matrix-completion1,noisy-matrix-completion2}.
For any matrix $\mat M\in\mathbb R^{n_1\times n_2}$ of rank $k$, singular value decomposition yields
$\mat M=\mat U\mat\Sigma\mat V^\top$, where $\mat U\in\mathbb R^{n_1\times k}$ and $\mat V\in\mathbb R^{n_2\times k}$ have orthonormal columns.
Let $\mathcal U=\Span(\mat U)$ and $\mathcal V=\Span(\mat V)$ be the column and row space of $\mat M$.
The \emph{column space coherence} is defined as
\begin{equation}
\mu(\mathcal U) := \frac{n_1}{k}\max_{i=1}^{n_1}{\|\mat U^\top\vct e_i\|_2^2} = \frac{n_1}{k}\max_{i=1}^{n_1}{\|\mat U_{(i)}\|_2^2}.
\label{eq_mu_subspace}
\end{equation}
Note that $\mu(\mathcal U)$ is always between 1 and $n_1/k$.
Similarly, the \emph{row space coherence} is defined as
\begin{equation}
\mu(\mathcal V) := \frac{n_2}{k}\max_{i=1}^{n_2}{\|\mat V^\top\vct e_i\|_2^2} = \frac{n_2}{k}\max_{i=1}^{n_2}{\|\mat V_{(i)}\|_2^2}.
\end{equation}

In this paper we also make use of incoherence level of vectors, which previously appeared in \citep{subspace-detection,akshay-nips,power-adaptivity}.
For a column vector $\vct x\in\mathbb R^{n_1}$, its incoherence is defined as
\begin{equation}
\mu(\vct  x) := \frac{n_1\|\vct x\|_{\infty}^2}{\|\vct x\|_2^2}.
\label{eq_mu_vector}
\end{equation}

It is an easy observation that if $\vct x$ lies in the subspace $\mathcal U$ then $\mu(\vct x) \leq k\mu(\mathcal U)$.
In this paper we adopt incoherence assumptions on the column space $\mathcal U$, which subsequently yields incoherent row vectors $\vct x_i$.
No incoherence assumption on the row space $\mathcal V$ or row vectors $\mat M_{(i)}$ is made.

\subsection{Matrix sampling schemes}\label{subsec:matrix-sampling-review}

\emph{Norm sampling}:
Norm sampling for column subset selection was proposed in \citep{norm2-css} and has found applications in a number of matrix computation tasks,
e.g., approximate matrix multiplication \citep{matmult_additive} and low-rank or compressed matrix approximation \citep{matapprox_additive,cur_additive}.
The idea is to sample each column with probability proportional to its squared $\ell_2$ norm,
i.e., $\Pr[i\in C] \propto \|\mat M^{(i)}\|_2^2$ for $i\in\{1,2,\cdots,n_2\}$.
These types of algorithms usually come with an additive error bound on their approximation performance.

\emph{Volume sampling}:
For volume sampling \citep{volume-sampling-css}, a subset of columns $C$ is picked with probability proportional
to the volume of the simplex spanned by columns in $C$.
That is, $\Pr[C] \propto \vol(\Delta(C))$ where $\Delta(C)$ is the simplex spanned by $\{\mat M^{(C(1))},\cdots,\mat M^{(C(k))}\}$.
{
Computationally efficient volume sampling algorithms exist \citep{deshpande2010efficient,anari2016monte}.
These methods are based on the computation of characteristic polynomials of the projected data matrix \citep{deshpande2010efficient}
or an MCMC sampling procedure \citep{anari2016monte}. 
Under the partially observed setting, both approaches are difficult to apply.
For the characteristic polynomials approach, one has to estimate the characteristic polynomial and essentially the least singular value of the target matrix $\mat M$
up to relative error bounds.
This is not possible unless the matrix is very well-conditioned, which violates the setting that $\mat M$ is approximately low-rank.
For the MCMC sampling procedure, it was shown in \citep{anari2016monte} that $O(kn_2)$ iterations are needed for the sampling Markov chain to mix.
As each sampling iteration requires observing one entire column, performing $O(kn_2)$ iterations essentially requires observing $O(kn_2)$ columns, i.e., the entire matrix $\mat M$.
On the other hand, an iterative norm sampling procedure is known to perform \emph{approximate volume sampling} and therefore enjoys
multiplicative approximation bounds for column subset selection \citep{approximate-volume-sampling}.
In this paper we generalize the iterative norm sampling scheme to the partially observed setting and demonstrate similar multiplicative approximation error guarantees.
}

\emph{Leverage score sampling}:
The leverage score sampling scheme was introduced in \citep{cur_relative} to get relative error bounds for CUR matrix approximation
 and has later been applied to coherent matrix completion \citep{complete-any-matrix}.
For each row $i\in\{1,\cdots,n_1\}$ and column $j\in\{1,\cdots,n_2\}$ define $\mu_i:=\frac{n_1}{k}\|\mat U^\top\vct e_i\|_2^2$
and $\nu_j:=\frac{n_2}{k}\|\mat V^\top\vct e_j\|_2^2$ to be their \emph{unnormalized leverage scores},
where $\mat U\in\mathbb R^{n_1\times k}$ and $\mat V\in\mathbb R^{n_2\times k}$ are the top-k left and right singular vectors of an input matrix $\mat M$.
It was shown in \citep{cur_relative} that if rows and columns are sampled with probability proportional to their leverage scores
then a relative error guarantee is possible for matrix CUR approximation and column subset selection.

\section{Column subset selection via active sampling}\label{sec:main_result}

\begin{table}[t]
\centering
\caption{Summary of theoretical guarantees of proposed algorithms.
$s$ denotes the number of selected columns and $m$ denotes the expected number of observed matrix entries.
Dependency on failure probability $\delta$ and other poly-logarithmic dependency is omitted.
$\mathcal U$ represents the column space of $\mat A$.}
\scalebox{0.7}{
\begin{tabular}{lcccccc}
\hline
\abovespace\belowspace
& Error type &  Error bound& $s$& $m$& Assumptions\\
\hline
\abovespace
\textsc{Norm}& $\|\mat M-\mat C\mat C^\dagger\mat M\|_F$& $\|\mat M-\mat M_k\|_F+\epsilon\|\mat M\|_F$& $\Omega(k/\epsilon^2)$& $\widetilde\Omega(\mu_1 n)$& $\max_{i=1}^{n_2}\mu(\mat M_{(i)})\leq \mu_1$\\
\belowspace
& $\|\mat M-\mat C\mat X\|_F$& $\|\mat M-\mat M_k\|_F+2\epsilon\|\mat M\|_F$& $\Omega(k/\epsilon^2)$& $\widetilde\Omega(k\mu_1 n/\epsilon^4)$& same as above\\
\abovespace
\textsc{Iter. norm}& $\|\mat M-\mat C\mat C^\dagger\mat M\|_F$& $\sqrt{2.5^k(k+1)!}\|\mat M-\mat M_k\|_F$& $k$& $\widetilde\Omega(k^2\mu_0 n)$& $\mat M=\mat A+\mat R$; $\mu(\mathcal U)\leq\mu_0$\\
& $\|\mat M-\mat C\mat C^\dagger\mat M\|_F$& $\sqrt{1+3\epsilon}\|\mat M-\mat M_k\|_F$& $\Theta(k^2\log k+k/\epsilon)$& $\widetilde\Omega\left(\frac{k\mu_0 n}{\epsilon}\left(k+\frac{1}{\epsilon}\right)\right)$& same as above\\
\belowspace
& $\|\mat M-\mat C\mat X\|_F$& $\sqrt{1+3\epsilon}\|\mat M-\mat M_k\|_F$& $\Theta(k^2\log k+k/\epsilon)$& $\widetilde\Omega\left(\frac{k\mu_0 n}{\epsilon}\left(k+\frac{1}{\epsilon}\right)\right)$& same as above\\
\abovespace\belowspace
\textsc{Lev. score}& $\|\mat M-\mat C\mat C^\dagger\mat M\|_F$& $3(1+\epsilon)\|\mat M-\mat M_k\|_F$& $\Omega(k^2/\epsilon^2)$& $\Omega(k^2\mu_0n/\epsilon^2)$& $\mu(\mathcal U)\leq\mu_0$\\
\hline
\end{tabular}
}
\label{tab_theoretical_result}
\end{table}

In this section we propose three column subset selection algorithms that only observe a small portion of an input matrix.
All algorithms employ the idea of active sampling to handle matrices with coherent rows.
While Algorithm \ref{alg_additive_cx} achieves an additive reconstruction error guarantee for any matrix,
Algorithm \ref{alg_relative_cx} achieves a relative-error reconstruction guarantee when the input matrix has certain structure.
Finally, Algorithm \ref{alg_lscore_cx} achieves a relative-error selection error bound for any general input matrix
at the expense of slower error rate and more sampled columns.
Table \ref{tab_theoretical_result} summarizes the main theoretical guarantees for the proposed algorithms.

\subsection{$\ell_2$ norm sampling}\label{subsec:norm}

\begin{algorithm*}[t]
\caption{Active norm sampling for column subset selection with missing data}
\begin{algorithmic}[1]
\State \textbf{Input}: size of column subset $s$, expected number of samples per column $m_1$ and $m_2$.
\State \textbf{Norm estimation}: For each column $i$, sample each index in $\Omega_{1,i}\subseteq [n_1]$ i.i.d. from $\bern(m_1/n_1)$.
observe $\vct x_{i,\Omega_{1,i}}$ and compute
$\hat c_i = \frac{n_1}{m_1}\|\vct x_{i,\Omega_{1,i}}\|_2^2$. Define $\hat f = \sum_i{\hat c_i}$.
\State \textbf{Column subset selection}: Set $\mat C=\mat 0\in\mathbb R^{n_1\times s}$.
\begin{itemize}
\item For $t\in[s]$: sample $i_t\in[n_2]$ such that $\Pr[i_t = j] = \hat c_j/\hat f$. 
Observe $\mat M^{(i_t)}$ in full and set $\mat C^{(t)} = \mat M^{(i_t)}$.
\end{itemize}
\State \textbf{Matrix approximation}: Set $\widehat{\mat M} = \mat 0\in\mathbb R^{n_1\times n_2}$.
\begin{itemize}
\item For each column $\vct x_i$, sample each index in $\Omega_{2,i}\subseteq [n_1]$ i.i.d. from $\bern(m_{2,i}/n_1)$, where $m_{2,i}=m_2n_2\hat c_i/\hat f$;
observe $\vct x_{i,\Omega_{2,i}}$.
\item Update: $\widehat{\mat M} = \widehat{\mat M} + (\mathcal R_{\Omega_{2,i}}(\vct x_i))\vct e_i^\top$. 
\end{itemize}
\State \textbf{Output}: selected columns $\mat C$ and coefficient matrix ${\mat X} = \mat C^{\dagger}\widehat{\mat M}$.
\end{algorithmic}
\label{alg_additive_cx}
\end{algorithm*}

We first present an active norm sampling algorithm (Algorithm \ref{alg_additive_cx}) for column subset selection under the missing data setting.
The algorithm is inspired by the norm sampling work for column subset selection by Frieze et al. \citep{norm2-css}
and the low-rank matrix approximation work by Krishnamurthy and Singh \citep{power-adaptivity}.

The first step of Algorithm \ref{alg_additive_cx} is to estimate the $\ell_2$ norm for each column by uniform subsampling.
Afterwards, $s$ columns of $\mat M$ are selected independently with probability proportional to their $\ell_2$ norms.
Finally, the algorithm constructs a sparse approximation of the input matrix by sampling each matrix entry with probability proportional to 
the square of the corresponding column's norm
and then a $\mat C\mat X$ approximation is obtained.

When the input matrix $\mat M$ has incoherent columns,
the selection error as well as $\mat C\mat X$ reconstruction error can be bounded as in Theorem \ref{thm_additive_cx}.
\begin{thm}
Suppose $\max_{i=1}^{n_2}{\mu(\vct x_i)} \leq \mu_1$ for some positive constant $\mu_1$.
Let $\mat C$ and $\mat X$ be the output of Algorithm \ref{alg_additive_cx}.
Denote $\mat M_k$ as the best rank-$k$ approximation of $\mat M$.
Fix $\delta = \delta_1+\delta_2+\delta_3 > 0$.
With probability at least $1-\delta$, we have
\begin{equation}
\|\mat M-\mat C\mat C^\dagger\mat M\|_F \leq \|\mat M-\mat M_k\|_F + \epsilon\|\mat M\|_F
\label{eq_additive_css}
\end{equation}
provided that $s=\Omega(k\epsilon^{-2}/\delta_2)$, $m_1=\Omega(\mu_1\log(n/\delta_1))$.
{Furthermore, if $m_2=\Omega(\mu_1s\log^2(n/\delta_3)/(\delta_2\epsilon^2))$
then with probability $\geq 1-\delta$ we have the following bound on reconstruction error:}
\begin{equation}
\|\mat M - \mat C{\mat X}\|_F \leq \|\mat M-\mat M_k\|_F + 2\epsilon\|\mat M\|_F.
\label{eq_additive_cx}
\end{equation}
\label{thm_additive_cx}
\end{thm}

As a remark, Theorem \ref{thm_additive_cx} shows that one can achieve $\epsilon$ additive reconstruction error using Algorithm \ref{alg_additive_cx} with expected sample complexity
(omitting dependency on $\delta$)
\begin{equation*}
\Omega\left(\mu_1n_2\log(n) + \frac{kn_1}{\epsilon^2} + \frac{k\mu_1 n_2\log^2(n)}{\epsilon^4}\right)
=\Omega(k\mu_1\epsilon^{-4}n\log^2 n).
\end{equation*}

\subsection{Iterative norm sampling}\label{subsec:iternorm}

\begin{algorithm*}[t]
\caption{Active iterative norm sampling for column subset selection for data corrupted by Gaussian noise}

\begin{algorithmic}[1]
\State \textbf{Input}: target rank $k < \min(n_1,n_2)$, error tolerance parameter $\epsilon$, $\delta$ and expected number of samples per column $m$.
\State \textbf{Entrywise sampling}: For each column $i$, sample each index in an index set $\Omega_i\subseteq[n_1]$ i.i.d. from $\bern(m/n_1)$.
Observe $\vct x_{i,\Omega_i}$.
\State \textbf{Approximate volume sampling}: Set $C = \mathcal U = \emptyset$. Let $\mat U$ be an orthonormal basis of $\mathcal U$.
\For{$t=1,2,\cdots, k$}
	\State For $i\in\{1,\cdots,n_2\}$, compute $\hat c_i^{(t)} = \frac{n_1}{m}\|\vct x_{i,\Omega_i}-\mat U_{\Omega_i}(\mat U_{\Omega_i}^\top\mat U_{\Omega_i})^{-1}\mat U_{\Omega_i}^\top\vct x_{i,\Omega_i}\|_2^2$. 
	\State Set $\hat f^{(t)} = \sum_{i=1}^{n_2}\hat c_i^{(t)}$.
	\State Select a column $i_t$ at random, with probability $\Pr[i_t=j] = \hat p_j^{(t)} = \hat c_j^{(t)} / \hat f^{(t)}$.
	\State Observe $\mat M^{(i_t)}$ in full and update: $C \gets C\cup\{i_t\}$, $\mathcal U \gets \Span(\mathcal U, \{\mat M^{(i_t)}\})$.
\EndFor
\State \textbf{Active norm sampling}: set $T=(k+1)\log(k+1)$ and $s_1=s_2=\cdots=s_{T-1}=5k$, $s_T=10k/\epsilon\delta$; $S=\emptyset$, $\mathcal S=\emptyset$.
Suppose $\mat U$ is an orthonormal basis of $\Span(\mathcal U,\mathcal S)$.
\For{$t=1,2,\cdots,T$}
	\State For $i\in\{1,\cdots,n_2\}$, compute $\hat c_i^{(t)} = \frac{n_1}{m}\|\vct x_{i,\Omega_i}-\mat U_{\Omega_i}(\mat U_{\Omega_i}^\top\mat U_{\Omega_i})^{-1}\mat U_{\Omega_i}^\top\vct x_{i,\Omega_i}\|_2^2$.
	\State Set $\hat f^{(t)} = \sum_{i=1}^{n_2}\hat c_i^{(t)}$.
	\State Select $s_t$ columns $S_t=(i_1,\cdots,i_{s_t})$ independently at random, with probability $\Pr[j\in S_t] = \hat q_j^{(t)} = \hat c_j^{(t)}/\hat f^{(t)}$.
	\State Observe $\mat M^{(S_t)}$ in full and update: $S\gets S\cup S_t$, $\mathcal S\gets\Span(\mathcal S,\{\mat M^{(S_t)}\})$.
\EndFor
\State \textbf{Matrix approximation}: $\widehat{\mat M}=\sum_{i=1}^{n_2}{\mat U(\mat U_{\Omega_i}^\top\mat U_{\Omega_i})^{-1}\mat U_{\Omega_i}\vct x_{i,\Omega_i}\vct e_i^\top}$,
where $\mat U\in\mathbb R^{n_1\times (s+k)}$ is an orthonormal basis of $\Span(\mathcal U_0,\mathcal U_1)$.
\State \textbf{Output}: selected column subsets $\mat C=(\mat M^{(C(1))}, \cdots,\mat M^{(C(k))})\in\mathbb R^{n_1\times k}$,
$\mat S=(\mat M^{(C)},\mat M^{(S_1)}, \cdots,\mat M^{(S_T)})\in\mathbb R^{n_1\times s}$ where $s=k+s_1+\cdots+s_T$
and $\mat X=\mat S\mat S^\dagger\widehat{\mat M}$.
\end{algorithmic}
\label{alg_relative_cx}
\end{algorithm*}


In this section we present Algorithm \ref{alg_relative_cx}, another active sampling algorithm
based on the idea of iterative norm sampling and approximate volume sampling introduced in \citep{approximate-volume-sampling}.
Though Algorithm \ref{alg_relative_cx} is more complicated than Algorithm \ref{alg_additive_cx}, 
it achieves a relative error bound on inputs that are noisy perturbation of some underlying low-rank matrix.

Algorithm \ref{alg_relative_cx} employs the idea of \emph{iterative norm sampling}.
That is, after selecting $l$ columns from $\mat M$, the next column (or next several columns depending on the error type)
is sampled according to column norms of a \emph{projected} matrix $\mathcal P_{\mathcal C^\perp}(\mat M)$,
where $\mathcal C$ is the subspace spanned by currently selected columns.
It can be shown that iterative norm sampling serves as an approximation of \emph{volume sampling},
a sampling scheme that is known to have relative error guarantees \citep{volume-sampling-css,approximate-volume-sampling}.

Theorem \ref{thm_relative_cx} shows that when the input matrix $\mat M$ is the sum of an exact low rank matrix $\mat A$
and a stochastic noise matrix $\mat R$, then by selecting exact $k$ columns from $\mat M$ using iterative norm sampling
one can upper bound the selection error $\|\mat M-\mat C\mat C^\dagger\mat M\|_F$ by the best rank-$k$ approximation error
$\|\mat M-\mat M_k\|_F$ within a multiplicative factor that does not depend on the matrix size $n$.
Such relative error guarantee is much stronger than the additive error bound provided in Theorem \ref{thm_additive_cx}
as when $\mat M$ is exactly low rank the error is eliminated with high probability.
In fact, when the input matrix $\mat M$ is exactly low rank the first phase of the proposed algorithm (Line 1 to Line 9 in Algorithm \ref{alg_relative_cx})
resembles the adaptive sampling algorithm proposed in \citep{akshay-nips,power-adaptivity} for matrix and tensor completion
in the sense that at each iteration all columns falling exactly onto the span of already selected columns
will have zero norm after projection and hence will never be sampled again.
However, we are unable to generalize our algorithm to general full-rank inputs
because it is difficult to bound the incoherence level of projected columns (and hence the projection accuracy itself later on)
without a stochastic noise model.
We present a new algorithm with slightly worse error bounds in Section \ref{subsec:lscore} which can handle general high-rank inputs.

Though Eq. (\ref{eq_relative_css}) is a relative error bound, the multiplicative factor scales exponentially with the intrinsic rank $k$,
which is not completely satisfactory.
As a remedy, we show that by slightly over-sampling the columns ($\Theta(k^2\log k+k/\epsilon\delta)$ instead of $k$ columns)
the selection error as well as the $\mat C\mat X$ reconstruction error could be upper bounded by $\|\mat M-\mat M_k\|_F$
within only a $(1+3\epsilon)$ factor,
{which implies that the error bounds are nearly optimal when the number of selected columns $s$ is sufficiently large,
for example, $s=\Omega(k^2\log k+k/\epsilon\delta)$.}

\begin{thm}
Fix $\delta > 0$.
Suppose $\mat M=\mat A+\mat R$, where $\mat A$ is a rank-$k$ deterministic matrix with incoherent column space
(i.e., $\mu(\mathcal U(\mat A)) \leq \mu_0$)
and $\mat R$ is a random matrix with i.i.d. zero-mean Gaussian distributed entries.
Suppose $k= O(n_1/\log(n_2/\delta))$.
Let $\mat C,\mat S$ and $\mat X$ be the output of Algorithm \ref{alg_relative_cx}.
Then the following holds:
\begin{enumerate}
\item If $m = \Omega(k^2\mu_0\log^2(n/\delta))$ then with probability $\geq 1-\delta$
\begin{equation}
\|\mat M-\mat C\mat C^\dagger\mat M\|_F^2 \leq \frac{2.5^k(k+1)!}{\delta}\|\mat R\|_F^2.
\label{eq_relative_css}
\end{equation}
The column subset size is $k$ and the corresponding sample complexity is $\Omega(k^2\mu_0 n\log^2(n/\delta))$.
\item If $m=\Omega(\epsilon^{-1}s\mu_0\log^2(n/\delta))$ with $s=\Theta(k^2\log k+k/\epsilon\delta)$, then with probability $\geq 1-\delta$
\begin{equation}
\|\mat M-\mat S\mat S^\dagger\mat M\|_F^2 \leq \|\mat M-\mat S\mat X\|_F^2 \leq (1+3\epsilon)\|\mat R\|_F^2.
\label{eq_relative_cx}
\end{equation}
The column subset size is $\Theta(k^2\log k+k/\epsilon\delta)$ and the sample complexity is (omitting dependence on $\delta$)
$$
\Omega\left(\frac{k^2\mu_0 n\log k\log^2(n)}{\epsilon} + \frac{k\mu_0 n\log^2(n)}{\epsilon^2}\right).
$$
\end{enumerate}
\label{thm_relative_cx}
\end{thm}

{
\subsection{Approximate leverage score sampling}\label{subsec:lscore}

\begin{algorithm*}[t]
\caption{Approximate leverage score sampling for column subset selection on general input matrices}
\begin{algorithmic}[1]
\State \textbf{Input}: target rank $k$, size of column subset $s$, expected number of row samples $m$.
\State \textbf{Leverage score estimation}: Set $\mathcal S=\emptyset$.
\begin{itemize}
\item For each row $i$, with probability $m/n_1$ observe the row $\mat M_{(i)}$ in full and update $\mathcal S\gets\Span(\mathcal S, \{\mat M_{(i)}\})$.
\item Compute the first $k$ right singular vectors of $\mathcal S$ (denoted by $\mat S_k\in\mathcal R^{n_2\times k}$) and 
estimate the unnormalized row space leverage scores as $\tilde l_j=\|{\mat S_k}^\top\vct e_j\|_2^2$, $j\in\{1,2,\cdots,n_1\}$.
\end{itemize}
\State \textbf{Column subset selection}: Set $C=\emptyset$.
\begin{itemize}
\item For $t\in\{1,2,\cdots,s\}$ select a column $i_t\in[n_2]$ with probability $Pr[i_t=j]=\hat p_j=\tilde l_j/k$;
update $C\gets C\cup\{i_t\}$.
\end{itemize}
\State \textbf{Output}: the selected column indices $C\subseteq \{1,2,\cdots,n_2\}$ and actual columns $\mat C=(\mat M^{(C(1))},\cdots,\mat M^{(C(s))})$.
\end{algorithmic}
\label{alg_lscore_cx}
\end{algorithm*}

The third sampling-based column subset selection algorithm for partially observed matrices is presented in Algorithm \ref{alg_lscore_cx}.
The proposed algorithm was based on the leverage score sampling scheme for matrix CUR approximation introduced in \citep{cur_relative}.
To compute the sampling distribution (i.e., leverage scores) from partially observed data,
the algorithm subsamples a small number of rows from the input matrix and uses leverage scores of the row space of the subsampled matrix to form the sampling distribution.
Note that we do not attempt to approximate leverage scores of the original input matrix directly;
instead, we compute leverage scores of another matrix that is a good approximation of the original data.
Such technique was also explored in \citep{approx-lev-score} to approximate statistical leverages in a fully observed setting.
Afterwards, column sampling distribution is constructed using the estimated leverage scores and a subset of columns are selected according to the constructed sampling distribution.

We bound the selection error $\|\mat M-\mat C\mat C^\dagger\mat M\|_F$ of the approximate leverage score sampling algorithm in Theorem \ref{thm_lscore_cx}.
Note that unlike Theorem \ref{thm_additive_cx} and \ref{thm_relative_cx}, 
only selection error bound is provided
since for deterministic full-rank input matrices it is challenging to approximately compute the projection of $\mat M$ onto $\Span(\mat C)$
because the projected vector may no longer be incoherent
(this is in fact the reason why Theorem \ref{thm_relative_cx} holds only for low-rank matrices perturbed by Gaussian noise,
and we believe similar conclusion should also hold for Algorithm \ref{alg_lscore_cx} is the stronger assumption of Gaussian noise perturbation is made).
It remains an open problem to approximately compute $\mat C^\dagger\mat M$ given $\mat C$ with provable guarantee for general matrix $\mat M$ without observing it in full.
Eq. (\ref{thm_lscore_cx}) shows that Algorithm \ref{alg_lscore_cx} enjoys a relative error bound on the selection error.
In fact, when the input matrix $\mat M$ is exactly low rank then Algorithm \ref{alg_lscore_cx} is akin to the two-step matrix completion method proposed in \citep{complete-any-matrix}
for column incoherent inputs.

Although Theorem \ref{thm_lscore_cx} shows that Algorithm \ref{alg_lscore_cx} generalizes the relative selection error bound in Theorem \ref{thm_relative_cx}
to general input matrices,
it also reveals several drawbacks of the approximate leverage score sampling algorithm compared to the iterative norm sampling method.
First, Algorithm \ref{alg_lscore_cx} always needs to over-sample columns (at the level of $\Theta(k^2/\epsilon^2)$, which is even more than Algorithm \ref{alg_relative_cx}
for a $(1+\epsilon)$ reconstruction error bound);
in contrast, the iterative norm sampling algorithm only requires exact $k$ selected columns to guarantee a relative error bound.
In addition, Eq. (\ref{eq_lscore_cx}) shows that the selection error bound is suboptimal even if $s$ is sufficiently large because of the $(3+3\epsilon)$ multiplicative term.

\begin{thm}
Suppose $\mat M$ is an input matrix with incoherent top-$k$ column space (i.e., $\mu(\mathcal U_k(\mat M)) \leq \mu_0$)
and $C$ is the column indices output by Algorithm \ref{alg_lscore_cx}.
If $m=\Omega(\epsilon^{-2}\mu_0k^2\log(1/\delta))$ and $s=\Omega(\epsilon^{-2}k^2\log(1/\delta))$ then with probability $\geq 1-\delta$
the following holds:
\begin{equation}
\|\mat M-\mat C\mat C^\dagger\mat M\|_F \leq 3(1+\epsilon)\|\mat M-\mat M_k\|_F,
\label{eq_lscore_cx}
\end{equation}
where $\mat C=[\mat M^{(C(1))},\cdots,\mat M^{(C(s))}]\in\mathbb R^{n_1\times s}$ are the selected columns
and $\mat M_k$ is the best rank-$k$ approximation of $\mat M$.
\label{thm_lscore_cx}
\end{thm}

}

\section{Proofs}\label{sec:proofs}

In this section we provide proof sketches of the main results (Theorem \ref{thm_additive_cx}, \ref{thm_relative_cx} and \ref{thm_lscore_cx}).
Some technical lemmas and complete proof details are deferred to Appendix \ref{appsec:alg1} and \ref{appsec:alg2}.

\subsection{Proof sketch of Theorem \ref{thm_additive_cx}}

{The proof of Theorem \ref{thm_additive_cx} can be divided into two steps.
First, in Lemma \ref{lem_l2css} we show that (approximate) column sampling yields an additive error bound for column subset selection.
Its proof is very similar to the one presented in \citep{norm2-css} and we defer it to Appendix \ref{appsec:alg1}.
Second, we cite a lemma from \citep{power-adaptivity}
to show that with high probability the first pass in Algorithm \ref{alg_additive_cx} gives accurate estimates of column norms of the input matrix $\mat M$.
}

\begin{lem}
Provided that $(1-\alpha)\|\vct x_i\|_2^2\leq \hat c_i\leq (1+\alpha)\|\vct x_i\|_2^2$ for $i=1,2,\cdots,n_2$,
with probability $\geq 1-\delta$ we have
\begin{equation}
\|\mat M - \mathcal P_C(\mat M)\|_F \leq \|\mat M-\mat M_k\|_F + \sqrt{\frac{(1+\alpha)k}{(1-\alpha)\delta s}}\|\mat M\|_F,
\end{equation}
where $\mat M_k$ is the best rank-$k$ approximation of $\mat M$.
\label{lem_l2css}
\end{lem}

\begin{lem}[\citep{power-adaptivity}, Lemma 10]
Fix $\alpha,\delta\in(0,1)$.
Assume $\mu(\vct x_i)\leq\mu_0$ holds for $i=1,2,\cdots,n_2$.
For some fixed $i\in\{1,\cdots,n_2\}$
with probability $\geq 1-2\delta$ we have
\begin{equation}
(1-\alpha)\|\vct x_i\|_2^2 \leq \hat c_i\leq (1+\alpha)\|\vct x_i\|_2^2
\label{eq_norm_estimation_additive}
\end{equation}
with $\alpha =\sqrt{\frac{2\mu_0}{m_1}\log(1/\delta)} + \frac{2\mu_0}{3m_1}\log(1/\delta)$.
Furthermore, if $m_1=\Omega(\mu_0\log(n_2/\delta))$ with carefully chosen constants then
Eq. (\ref{eq_norm_estimation_additive}) holds uniformly for all columns with $\alpha=0.5$.
\label{lem_norm_estimation_additive}
\end{lem}

Combining Lemma \ref{lem_l2css} and Lemma \ref{lem_norm_estimation_additive}
and setting $s=\Omega(k\epsilon^{-2}/\delta)$ for some target accuracy threshold $\epsilon$
we have that with probability $1-3\delta$ the selection error bound Eq. (\ref{eq_additive_css}) holds.

In order to bound the reconstruction error $\|\mat M-\mat C\mat X\|_F^2$, we cite another lemma from \citep{power-adaptivity}
that analyzes the performance of the second pass of Algorithm \ref{alg_additive_cx}.
{
At a higher level, Lemma \ref{lem_approx} is a consequence of matrix Bernstein inequality \citep{user-friendly}
which asserts that the \emph{spectral norm} of a matrix can be preserved by a sum of properly scaled randomly sampled sub-matrices.
}
\begin{lem}[\citep{power-adaptivity}, Lemma 9]
Provided that $(1-\alpha)\|\vct x_i\|_2^2\leq \hat c_i\leq (1+\alpha)\|\vct x_i\|_2^2$ for $i=1,2,\cdots,n_2$,
with probability $\geq 1-\delta$ we have
\begin{multline}
\|\mat M-\widehat{\mat M}\|_2 \leq \|\mat M\|_F\sqrt{\frac{1+\alpha}{1-\alpha}}\left(\frac{4}{3}\sqrt{\frac{n_1\mu_0}{m_2n_2}}\log\left(\frac{n_1+n_2}{\delta}\right)\right.
\left.+\sqrt{\frac{4}{m_2}\max\left(\frac{n_1}{n_2}, \mu_0\right)\log\left(\frac{n_1+n_2}{\delta}\right)}\right).
\label{eq_approx}
\end{multline}
\label{lem_approx}
\end{lem}

The complete proof of Theorem \ref{thm_additive_cx} is deferred to Appendix \ref{appsec:alg1}.

\subsection{Proof sketch of Theorem \ref{thm_relative_cx}}

In this section we give proof sketch of Eq. (\ref{eq_relative_css}) and Eq. (\ref{eq_relative_cx}) separately.

\subsubsection{Proof sketch of $\|\mat M-\mat C\mat C^\dagger\mat M\|_F$ error bound}

We take three steps to prove the $\|\mat M-\mat C\mat C^\dagger\mat M\|_F$ error bound in Theorem \ref{thm_relative_cx}.
At the first step, we show that when the input matrix has a low rank plus noise structure
then with high probability \emph{for all small subsets of columns} the spanned subspace has incoherent column space (assuming the low-rank matrix has incoherent column space)
and furthermore, the projection of the other columns onto the orthogonal complement of the spanned subspace are incoherent, too.
Given the incoherence condition we can easily prove a norm estimation result similar to Lemma \ref{lem_norm_estimation_additive},
which is the second step.
Finally, we note that the approximate iterative norm sampling procedure is an approximation of \emph{volume sampling},
a column sampling scheme that is known to yield a relative error bound.

\emph{STEP 1}: We first prove that when the input matrix $\mat M$ is a noisy low-rank matrix with incoherent column space,
with high probability \emph{a fixed} column subset also has incoherent column space.
This is intuitive because the Gaussian perturbation matrix is highly incoherent with overwhelming probability.
A more rigorous statement is shown in Lemma \ref{lem_preserve_incoherence}.
\begin{lem}
{Suppose $\mat A$ has incoherent column space, i.e., $\mu(\mathcal U(\mat A)) \leq\mu_0$.}
Fix $C\subseteq [n_2]$ to be any subset of column indices that has $s$ elements and $\delta > 0$.
Let $\mat C=[\mat M^{(C(1))},\cdots,\mat M^{(C(s))}]\in\mathbb R^{n_1\times s}$ be the compressed matrix 
and $\mathcal U(C)=\Span(\mat C)$ denote the subspace spanned by the selected columns.
{ Suppose $\max(s,k)\leq n_1/4-k$ and $\log(4n_2/\delta) \leq n_1/64$.}
Then with probability $\geq 1-\delta$ over the random drawn of $\mat R$ we have
\begin{equation}
\mu(\mathcal U(C)) = \frac{n_1}{s}\max_{1\leq i\leq n_1}{\|\mathcal P_{\mathcal U(C)}\vct e_i\|_2^2} 
= O\left(\frac{k\mu_0 + s + \sqrt{s\log(n_1/\delta)} + \log(n_1/\delta)}{s}\right);
\label{eq_preserve_subspace_incoherence}
\end{equation}
furthermore, with probability $\geq 1-\delta$ the following holds:
\begin{equation}
\mu(\mathcal P_{\mathcal U(C)^\perp}(\mat M^{(i)})) = O(k\mu_0+\log(n_1n_2/\delta)),\quad
\forall i\notin C.
\label{eq_preserve_project_incoherence}
\end{equation}
\label{lem_preserve_incoherence}
\end{lem}

{
At a higher level, Lemma \ref{lem_preserve_incoherence} is a consequence of the Gaussian white noise being highly incoherent,
and the fact that the randomness imposed on each column of the input matrix is independent.}
The complete proof can be found in Appendix \ref{appsec:alg2}.

{
Given Lemma \ref{lem_preserve_incoherence}, Corollary \ref{cor_preserve_incoherence_uniform} holds
by taking a uniform bound over all $\sum_{j=1}^s{\binom{n_2}{j}}=O(s(n_2)^s)$ column subsets that contain no more than $s$ elements.
The $2s\log(4n_2/\delta)\leq n_1/64$ condition is only used to ensure that the desired failure probability $\delta$ is not exponentially small.
Typically, in practice the intrinsic dimension $k$ and/or the target column subset size $s$ is much smaller than the ambient dimension $n_1$. 
\begin{cor}
Fix $\delta > 0$ and $s\geq k$.
{Suppose $s\leq n_1/8$ and $2s\log(4n_2/\delta)\leq n_1/64$.}
With probability $\geq 1-\delta$ the following holds:
for any subset $C\subseteq [n_2]$ with at most $s$ elements, the spanned subspace $\mathcal U(C)$ satisfies
\begin{equation}
\mu(\mathcal U(C)) \leq O((k+s)|C|^{-1}\mu_0\log(n/\delta));
\label{eq_preserve_subspace_incoherence_uniform}
\end{equation}
furthermore, 
\begin{equation}
\mu(\mathcal P_{\mathcal U(C)^\perp}(\mat M^{(i)})) = O((k+s)\mu_0\log(n/\delta)),\quad \forall i\notin C.
\label{eq_preserve_project_incoherence_uniform}
\end{equation}
\label{cor_preserve_incoherence_uniform}
\end{cor}
}

\emph{STEP 2}: In this step, we prove that the norm estimation scheme in Algorithm \ref{alg_relative_cx} works
when the incoherence conditions in Eq. (\ref{eq_preserve_subspace_incoherence_uniform}) and (\ref{eq_preserve_project_incoherence_uniform}) are satisfied.
More specifically, we have the following lemma bounding the norm estimation error:
\begin{lem}
Fix $i\in\{1,\cdots,n_2\}$, $t\in\{1,\cdots,k\}$ and $\delta,\delta'>0$.
Suppose Eq. (\ref{eq_preserve_subspace_incoherence_uniform}) and (\ref{eq_preserve_project_incoherence_uniform}) hold with probability $\geq 1-\delta$.
Let $\mathcal S_t$ be the subspace spanned by selected columns at the $t$-th round
and let $\hat c_i^{(t)}$ denote the estimated squared norm of the $i$th column.
If $m$ satisfies
\begin{equation}
m = \Omega(k\mu_0\log(n/\delta)\log(k/\delta')),
\label{eq_m_lowerbound}
\end{equation}
then with probability $\geq 1-\delta-4\delta'$ we have
\begin{equation}
\frac{1}{2}\|[\mat E_t]_{(i)}\|_2^2 \leq \hat c_i^{(t)}\leq \frac{5}{4}\|[\mat E_t]_{(i)}\|_2^2.
\label{eq_norm_estimation}
\end{equation}
Here $\mat E_t = \mathcal P_{\mathcal S_t^\perp}(\mat M)$ denotes the projected matrix at the $t$-th round.
\label{lem_norm_estimation}
\end{lem}
Lemma \ref{lem_norm_estimation} is similar with previous results on subspace detection \citep{subspace-detection} and matrix approximation \citep{power-adaptivity}.
The intuition behind Lemma \ref{lem_norm_estimation} is that one can accurately estimate the $\ell_2$ norm of a vector by uniform subsampling
entries of the vector, provided that the vector itself is incoherent.
The proof of Lemma \ref{lem_norm_estimation} is deferred to Appendix \ref{appsec:alg2}.

Similar to the first step, by taking a union bound over all possible subsets of picked columns and $n_2-k$ unpicked columns 
we can prove a stronger version of Lemma \ref{lem_norm_estimation}, as shown in Corollary \ref{cor_norm_estimation_uniform}.
\begin{cor}
Fix $\delta,\delta' > 0$.
Suppose Eq. (\ref{eq_preserve_subspace_incoherence_uniform}) and (\ref{eq_preserve_project_incoherence_uniform}) hold with probability $\geq 1-\delta$.
If 
\begin{equation}
m\geq\Omega(k^2\mu_0\log(n/\delta)\log(n/\delta'))
\label{eq_m_lowerbound_uniform}
\end{equation}
then with probability $\geq 1-\delta-4\delta'$ the following property holds for any selected column subset by Algorithm \ref{alg_relative_cx}:
\begin{equation}
\frac{2}{5}\frac{\|[\mat E_t]_{(i)}\|_2^2}{\|\mat E_t\|_F^2} \leq \hat p_i^{(t)}\leq \frac{5}{2}\frac{\|[\mat E_t]_{(i)}\|_2^2}{\|\mat E_t\|_F^2},\forall i\in[n_2],t\in[k],
\label{eq_norm_estimation_uniform}
\end{equation}
where $\hat p_i^{(t)}=\hat c_i^{(t)}/\hat f^{(t)}$ is the sampling probability of the $i$th column at round $t$.
\label{cor_norm_estimation_uniform}
\end{cor}

\emph{STEP 3}: To begin with, we define \emph{volume sampling} distributions:
\begin{defn}[volume sampling, \citep{volume-sampling-css}]
A distribution $p$ over column subsets of size $k$ is a volume sampling distribution if 
\begin{equation}
p(C) = \frac{\vol(\Delta(C))^2}{\sum_{T:|T|=k}{\vol(\Delta(T))^2}},\quad\forall |C|=k.
\label{eq_volume_sampling}
\end{equation}
\label{def_volume_sampling}
\end{defn}
Volume sampling has been shown to achieve a relative error bound for column subset selection,
which is made precise by Theorem \ref{thm_volume_sampling_css} cited from \citep{approximate-volume-sampling,volume-sampling-css}.
\begin{thm}[\citep{approximate-volume-sampling}, Theorem 4]
Fix a matrix $\mat M$ and let $\mat M_k$ denote the best rank-$k$ approximation of $\mat M$.
If the sampling distribution $p$ is a volume sampling distribution defined in Eq. (\ref{eq_volume_sampling}) then 
\begin{equation}
\mathbb E_C\left[\|\mat M-\mathcal P_{\mathcal V(C)}(\mat M)\|_F^2\right] \leq (k+1)\|\mat M-\mat M_k\|_F^2;
\label{eq_volume_sampling_css_e}
\end{equation}
furthermore, applying Markov's inequality one can show that with probability $\geq 1-\delta$
\begin{equation}
\|\mat M-\mathcal P_{\mathcal V(C)}(\mat M)\|_F^2 \leq \frac{k+1}{\delta}\|\mat M-\mat M_k\|_F^2.
\label{eq_volume_sampling_css_whp}
\end{equation}
\label{thm_volume_sampling_css}
\end{thm}

{
In general, exact volume sampling is difficult to employ under partial observation settings, as we explained in Sec.~\ref{subsec:matrix-sampling-review}.}
However, in \citep{approximate-volume-sampling} it was shown that iterative norm sampling serves as an approximate of volume sampling
and achieves a relative error bound as well.
In Lemma \ref{lem_approximate_volume_sampling} we present an extension of this result. 
Namely, \emph{approximate} iterative column norm sampling is an approximate of volume sampling, too.
Its proof is very similar to the one presented in \citep{approximate-volume-sampling} and we defer it to Appendix \ref{appsec:alg2}.
\begin{lem}
Let $p$ be the volume sampling distribution defined in Eq. (\ref{eq_volume_sampling}). 
Suppose the sampling distribution of a $k$-round sampling strategy $\hat p$ satisfies Eq. (\ref{eq_norm_estimation_uniform}).
Then we have
\begin{equation}
\hat p_C \leq 2.5^k k!p_C,\quad \forall |C|=k.
\label{eq_approximate_volume_sampling}
\end{equation}
\label{lem_approximate_volume_sampling}
\end{lem}

We can now prove the error bound for selection error $\|\mat M-\mat C\mat C^\dagger\mat M\|_F$ of Algorithm \ref{alg_relative_cx}
by combining Corollary \ref{cor_preserve_incoherence_uniform}, \ref{cor_norm_estimation_uniform}, Lemma \ref{lem_approximate_volume_sampling} and 
Theorem \ref{thm_volume_sampling_css},
{with failure probability $\delta,\delta'$ set at $O(1/k)$ to facilitate a union bound argument across all iterations.}
In particular, Corollary \ref{cor_preserve_incoherence_uniform} and \ref{cor_norm_estimation_uniform}
guarantees that Algorithm \ref{alg_relative_cx} estimates column norms accurately with high probability;
then one can apply Lemma \ref{lem_approximate_volume_sampling} to show that the sampling distribution employed in the algorithm
is actually an approximate volume sampling distribution, which is known to achieve relative error bounds (by Theorem \ref{thm_volume_sampling_css}).

{
\subsubsection{Proof sketch of $\|\mat M-\mat S\mat X\|_F$ error bound}

We first present a theorem, which is a generalization of Theorem 2.1 in \citep{volume-sampling-css}.
\begin{thm}[\citep{volume-sampling-css}, Theorem 2.1]
Suppose $\mat M\in\mathbb R^{n_1\times n_2}$ is the input matrix and $\mathcal U\subseteq\mathbb R^{n_1}$ is an arbitrary vector space.
Let $\mat S\in\mathbb R^{n_1\times s}$ be a random sample of $s$ columns in $\mat M$ from a distribution $q$ such that
\begin{equation}
\frac{(1-\alpha)\|\mat E^{(i)}\|_2^2}{(1+\alpha)\|\mat E\|_F^2} \leq q_i \leq \frac{(1+\alpha)\|\mat E^{(i)}\|_2^2}{(1-\alpha)\|\mat E\|_F^2},
\quad\forall i\in\{1,2,\cdots,n_2\},
\label{eq_norm_estimation_ver3}
\end{equation}
where $\mat E=\mathcal P_{\mathcal U^\perp}(\mat M)$ is the projection of $\mat M$ onto the orthogonal complement of $\mathcal U$.
Then 
\begin{equation}
\mathbb E_S\left[\|\mat M-\mathcal P_{\Span(\mathcal U,\mat S),k}(\mat M)\|_F^2\right]
\leq \|\mat M-\mat M_k\|_F^2 + \frac{(1+\alpha)k}{(1-\alpha)s}\|\mat E\|_F^2,
\label{eq_adaptive_norm_one_iter}
\end{equation}
where $\mat M_k$ denotes the best rank-$k$ approximation of $\mat M$.
\label{thm_adaptive_norm_lemma}
\end{thm}
Intuitively speaking, Theorem \ref{thm_adaptive_norm_lemma} states that relative estimation of residues $\mathcal P_{\mathcal U^\perp}(\mat M)$
would yield relative estimation of the data matrix $\mat M$ itself.

In the remainder of the proof we assume $s=\Omega(k^2\log(k)+k/\epsilon\delta)$ is the number of columns selected in $\mat S$ in Algorithm \ref{alg_relative_cx}.
Corollary \ref{cor_preserve_incoherence_uniform} asserts that with high probability $\mu(\mathcal U(\mat S))=O(s|C|^{-1}\mu_0\log(n/\delta))$
and $\mu(\mathcal P_{\mathcal U(\mat S)^\perp}(\mat M^{(i)})) = O(s\mu_0\log(n/\delta))$ for any subset $S$ with $|S|\leq s$.
Subsequently, we can apply Lemma \ref{lem_norm_estimation} and a union bound over $n_2$ columns and $T$ rounds to obtain the following proposition:
\begin{prop}
Fix $\delta,\delta'>0$.
If $m=\Omega(s\mu_0\log(n/\delta)\log(nT/\delta'))$ then with probability $\geq 1-\delta-\delta'$ 
\begin{equation}
\frac{2\|\mat E_t^{(i)}\|_2^2}{5\|\mat E_t\|_F^2} \leq \hat q_i^{(t)}
\leq \frac{5\|\mat E_t^{(i)}\|_2^2}{2\|\mat E_t\|_F^2},\quad\forall i\in\{1,2,\cdots,n_2\}, t\in\{1,2,\cdots,T\}.
\end{equation}
Here $\mat E_t=\mat M-\mathcal P_{\Span(\mathcal U\cup\mathcal S_1\cup\cdots\cup\mathcal S_{t-1})}(\mat M)$ is the residue at round $t$ of the active norm sampling procedure.
\label{prop_norm_estimation_ver3}
\end{prop}
Note that we do not need to take a union bound over all $\binom{n_2}{s}$ column subsets because this time we do not require the sampling distribution of Algorithm \ref{alg_relative_cx}
to be close uniformly to the true active norm sampling procedure.

Consequently, combining Theorem \ref{thm_adaptive_norm_lemma} and Proposition \ref{prop_norm_estimation_ver3} we obtain Lemma \ref{lem_adaptive_exponential}.
Its proof is deferred to Appendix \ref{appsec:alg2}.
\begin{lem}
Fix $\delta,\delta'>0$.
If $m=\Omega(s\mu_0\log^2(n/\delta))$ and $s_1=\cdots=s_{T-1}=5k$, $s_T=10k/\epsilon\delta'$
then with probability $\geq 1-2\delta-\delta''$
\begin{equation}
\|\mat M-\mathcal P_{\mathcal U\cup\mathcal S_1\cup\cdots\cup\mathcal S_T}(\mat M)\|_F^2 \leq (1+\epsilon/2)\|\mat M-\mat M_k\|_F^2 + \frac{\epsilon/2}{2^T}\|\mat M-\mathcal P_{\mathcal U}(\mat M)\|_F^2.
\end{equation}
\label{lem_adaptive_exponential}
\end{lem}

Applying Theorem \ref{thm_volume_sampling_css}, Lemma \ref{lem_approximate_volume_sampling} and note that $2^{(k+1)\log(k+1)} = (k+1)^{(k+1)} \geq (k+1)!$, 
we immediately have Corollary \ref{cor_adaptive_exponential}.
\begin{cor}
Fix $\delta > 0$.
Suppose $T=(k+1)\log(k+1)$ and $m,s_1,\cdots,s_T$ be set as in Lemma \ref{lem_adaptive_exponential}.
Then with probability $\geq 1-4\delta$ one has
\begin{equation}
\|\mat M-\mat S\mat S^\dagger\mat M\|_F^2 = \|\mat M-\mathcal P_{\mathcal U\cup\mathcal S_1\cup\cdots\cup\mathcal S_T}(\mat M)\|_F^2
\leq (1+\epsilon)\|\mat M-\mat M_k\|_F^2 \leq (1+\epsilon)\|\mat R\|_F^2.
\end{equation}
\label{cor_adaptive_exponential}
\end{cor}

To reconstruct the coefficient matrix $\mat X$ and to further bound the reconstruction error $\|\mat M-\mat S\mat X\|_F$, 
we apply the $\mat U(\mat U_{\Omega}^\top\mat U_{\Omega})^{-1}\mat U_{\Omega}$ operator on every column to build a low-rank approximation $\widehat{\mat M}$.
It was shown in \citep{akshay-nips,subspace-detection} that this operator recovers all components in the underlying subspace $\mathcal U$ with high probability,
and hence achieves a relative error bound for low-rank matrix approximation.
More specifically, we have Lemma \ref{lem_relative_matrix_approximation}, which is proved in Appendix \ref{appsec:alg2}.
\begin{lem}
Fix $\delta,\delta''>0$ and $\epsilon>0$.
Let $\mat S\in\mathbb R^{n_1\times s}$ and $\mat X\in\mathbb R^{s\times n_2}$ be the output of Algorithm \ref{alg_relative_cx}.
Suppose Corollary \ref{cor_adaptive_exponential} holds with probability $\geq 1-\delta$.
If $m$ satisfies 
\begin{equation}
m=\Omega(\epsilon^{-1}s\mu_0\log(n/\delta)\log(n/\delta'')),
\end{equation}
then with probability $\geq 1-\delta-\delta''$ we have 
\begin{equation}
\|\mat M - \widehat{\mat M}\|_F^2 \leq (1+\epsilon)\|\mat M-\mat S\mat S^\dagger\mat M\|_F^2.
\end{equation}
\label{lem_relative_matrix_approximation}
\end{lem}

Note that all columns of $\widehat{\mat M}$ are in the subspace $\mathcal U(S)$.
Therefore, $\mat S\mat X = \mat S\mat S^\dagger\widehat{\mat M} = \widehat{\mat M}$.
The proof of Eq. (\ref{eq_relative_cx}) is then completed by noting that $(1+\epsilon)^2\leq 1+3\epsilon$ whenever $\epsilon\leq 1$.
}
{
\subsection{Proof of Theorem \ref{thm_lscore_cx}}

Before presenting the proof, 
we first present a theorem cited from \citep{cur_relative}.
In general, Theorem \ref{thm_cx_relative_leverage} claims that if columns are selected with probability proportional
to their row-space leverage scores then the resulting column subset is a relative-error approximation of the original input matrix.
\begin{thm}[\citep{cur_relative}, Theorem 3]
Let $\mat M\in\mathbb R^{n_1\times n_2}$ be the input matrix and $k$ be a rank parameter.
Suppose a subset of columns $C=\{i_1,i_2,\cdots,i_s\}\subseteq [n_2]$ is selected such that
\begin{equation}
\Pr[i_t=j] = p_j \geq \frac{\beta\|\mat V_k^\top\vct e_j\|_2^2}{k},\quad\forall t\in\{1,\cdots,s\}, j\in\{1,\cdots,n_2\}.
\label{eq_approx_lscore_sampling}
\end{equation}
Here $\mat V_k\in\mathbb R^{n_2\times k}$ is the top-$k$ right singular vectors of $\mat M$.
If $s=\Omega(\beta^{-1}\epsilon^{-2}k^2\log(1/\delta))$ then with probability $\geq 1-\delta$ one has
\begin{equation}
\|\mat M-\mat C\mat C^\dagger\mat M\|_F \leq (1+\epsilon)\|\mat M-\mat M_k\|_F.
\end{equation}
\label{thm_cx_relative_leverage}
\end{thm}

In the sequel we use $\mathcal Q_S(\mat M)$ to denote the matrix formed by projecting each row of $\mat M$ to a row subspace $\mathcal S$
and $\mathcal P_C(\mat M)$ to denote the matrix formed by projecting each column of $\mat M$ to a column subspace $\mathcal C$.
Since $\mat M$ has incoherent column space, the uniform sampling distribution $p_j=1/n_1$ satisfies Eq. (\ref{eq_approx_lscore_sampling})
with $\beta=1/\mu_0$.
Consequently, by Theorem \ref{thm_cx_relative_leverage} the computed row space $\mathcal S$ satisfies
\begin{equation}
\|\mat M-\mathcal Q_S(\mat M)\|_F\leq (1+\epsilon)\|\mat M-\mat M_k\|_F
\label{eq_lscore_first_approx}
\end{equation}
with high probability when $m=\Omega(k^2/\beta\epsilon^2)=\Omega(\mu_0k^2/\epsilon^2)$.

Next, note that though we do not know $\mathcal Q_S(\mat M)$, we know its row space $\mathcal S$.
Subsequently, we can compute the exact leverage scores of $\mathcal Q_S(\mat M)$, i.e., $\|\mat S_k^\top\vct e_j\|_2^2$ for $j=1,2,\cdots,n_2$.
With the computed leverage scores, performing leverage score sampling on $\mathcal Q_S(\mat M)$ as in Algorithm \ref{alg_lscore_cx}
and applying Theorem \ref{thm_cx_relative_leverage} we obtain
\begin{equation}
\|\mathcal Q_S(\mat M)-\mathcal P_C(\mathcal Q_S(\mat M))\|_F \leq (1+\epsilon)\|\mathcal Q_S(\mat M)-\left[\mathcal Q_S(\mat M)\right]_k\|_F,
\label{eq_lscore_second_approx}
\end{equation}
where $\left[\mathcal Q_S(\mat M)\right]_k$ denotes the best rank-$k$ approximation of $\mathcal Q_S(\mat M)$.
Note that
\begin{equation}
\|\mathcal Q_S(\mat M)-\left[\mathcal Q_S(\mat M)\right]_k\|_F
\leq \|\mathcal Q_S(\mat M)-\mathcal Q_S(\mat M_k)\|_F
= \|\mathcal Q_S(\mat M-\mat M_k)\|_F
\leq \|\mat M-\mat M_k\|_F
\end{equation}
because $\mathcal Q_S(\mat M_k)$ has rank at most $k$.
Consequently, the selection error $\|\mat M-\mathcal P_C(\mat M)\|_F$ can be bounded as follows:
\begin{eqnarray*}
\|\mat M-\mathcal P_C(\mat M)\|_F
&\leq& \|\mat M-\mathcal Q_S(\mat M)\|_F + \|\mathcal Q_S(\mat M) - \mathcal P_C(\mathcal Q_S(\mat M))\|_F + \|\mathcal P_C(\mathcal Q_S(\mat M)) - \mathcal P_C(\mat M)\|_F\\
&\leq& \|\mat M-\mathcal Q_S(\mat M)\|_F + \|\mathcal Q_S(\mat M) - \mathcal P_C(\mathcal Q_S(\mat M))\|_F + \|\mathcal Q_S(\mat M) - \mat M\|_F\\
&\leq& 3(1+\epsilon)\|\mat M-\mat M_k\|_F.
\end{eqnarray*}

}

\section{Related work on column subset selection with missing data}\label{sec:related_work}

In this section we review two previously proposed algorithms for column subset selection with missing data.
Both algorithms are heuristic based and no theoretical analysis is available.
We also remark that both methods employ the passive sampling scheme as observation models.
In fact, they work for any subset of observed matrix entries.

\subsection{Block orthogonal matching pursuit (Block OMP)}\label{subsec:blockomp}

\begin{algorithm*}[t]
\caption{A block OMP algorithm for column subset selection with missing data}
\begin{algorithmic}[1]
\State \textbf{Input}: size of column subset $s$, observation mask $\mat W\in\{0,1\}^{n_1\times n_2}$.
\State \textbf{Initialization}: Set $C=\emptyset$, $\mathcal C=\emptyset$, $\mat Y=\mat W\circ\mat M$, $\mat Y^{(1)} = \mat Y$.
\For{$t=1,2,\cdots, s$}
	\State Compute $\mat D = \mat Y^\top(\mat W\circ\mat Y^{(t)})$. Let $\{\vct d_i\}_{i=1}^{n_2}$ be rows of $\mat D$.
	\State Column selection: $i_t = \argmax_{1\leq i\leq n_2}\|\vct d_i\|_2$; update: $C\gets C\cup\{i_t\}$, $\mathcal C\gets\Span(\mathcal C, \mat Y_{(i_t)})$.
	\State Back projection: $\mat Y^{(t+1)} = \mat Y^{(t)} - \mathcal P_{\mathcal C}(\mat Y^{(t)})$.
\EndFor
\State \textbf{Output}: the selected column indices $C\subseteq \{1,2,\cdots,n_2\}$.
\end{algorithmic}
\label{alg_blockomp}
\end{algorithm*}

A block OMP algorithm was proposed in \citep{missing-css-blockomp} for column subset selection with missing data.
Let $\mat W\in\{0,1\}^{n_1\times n_2}$ denote the ``mask'' of observed entries; that is,
\begin{equation*}
\mat W_{ij} = \left\{\begin{array}{ll}
1,& \text{if $\mat M_{ij}$ is observed};\\
0,& \text{if $\mat M_{ij}$ is not observed}.\\\end{array}\right.
\end{equation*}
We also use $\circ$ to denote the Hadamard product (entrywise product) between two matrices of the same dimension.

The pseudocode is presented in Algorithm \ref{alg_blockomp}.
Note that Algorithm \ref{alg_blockomp} has very similar framework compared with the iterative norm sampling algorithm:
both methods select columns in an iterative manner
and after each column is selected, the contribution of selected columns is removed from the input matrix
by projecting onto the complement of the subspace spanned by selected columns.
Nevertheless, there are some major differences.
First, in iterative norm sampling we select a column according to their residue norms
while in block OMP we base such selection on inner products between the original input matrix and the residue one.
In addition, due to the passive sampling nature Algorithm \ref{alg_blockomp}
uses the zero-filled data matrix to approximate subspace spanned by selected columns.
In contrast, iterative norm sampling computes this subspace exactly by active sampling.

\subsection{Group Lasso}\label{subsec:glasso}

The group Lasso formulation was originally proposed in \citep{missing-css-glasso} as a convex optimization
alternative for matrix column subset selection and CUR decomposition for fully-observed matrices.
It was briefly remarked in \citep{missing-css-blockomp} that group Lasso could be extended to the case when only partial observations are available.
In this paper we make such extension precise by proposing the following convex optimization problem:
\begin{equation}
\min_{\mat X\in\mathbb R^{n_1\times n_2}}\|\mat W\circ\mat M - (\mat W\circ\mat M)\mat X\|_F^2 + \lambda\|\mat X\|_{1,2},
\quad s.t.\;\;\diag(\mat X) = \vct 0.
\label{eq_glasso}
\end{equation}
Here in Eq. (\ref{eq_glasso}) $\mat W\in\{0,1\}^{n_1\times n_2}$ denotes the mask for observed matrix entries
and $\circ$ denotes the Hadamard (entrywise) matrix product.
$\|\mat X\|_{1,2}=\sum_{i=1}^{n_2}{\|\mat X_{(i)}\|_2}$ denotes the 1,2-norm of matrix $\mat X$,
which is the sum of $\ell_2$ norm of all rows in $\mat X$.
The nonzero rows in the optimal solution $\mat X$ correspond to the selected columns.

Eq. (\ref{eq_glasso}) could be solved using standard convex optimization methods, e.g., proximal gradient descent \citep{prox-glasso}.
However, to make Eq. (\ref{eq_glasso}) a working column subset selection algorithm one needs to carefully choose the regularization parameter $\lambda$
so that the resulting optimal solution $\mat X$ has no more than $s$ nonzero columns.
Such selection could be time-consuming and inexact.
As a workaround, we implement the solution path algorithm for group Lasso problems in \citep{solution-path-glasso}.

\subsection{Discussion on theoretical assumptions of block OMP and group Lasso}

We discuss theoretical assumptions required for block OMP and group Lasso approaches.
It should be noted that for the particular matrix column subset selection problem, neither \cite{missing-css-blockomp} or \cite{missing-css-glasso}
provides rigorous theoretical guarantee of approximation error of the selected column subsets.
However, it is informative to compare to typical assumptions that are used to analyze block OMP and group Lasso for regression problems in the existing literature \citep{yuan2006model,lounici2011oracle}.
In most cases,
certain ``restricted eigenvalue'' conditions on the design matrix $\mat X$,
 which roughly corresponds to a ``weak correlation'' condition among columns of a data matrix.
 This explains the worse performance of both methods on data sets that have highly correlated columns (e.g., many repeated columns), as we shown in later sections on experimental results.

\section{Experiments}\label{sec:experiments}

\begin{figure*}[t]
\centering
\includegraphics[width=12cm]{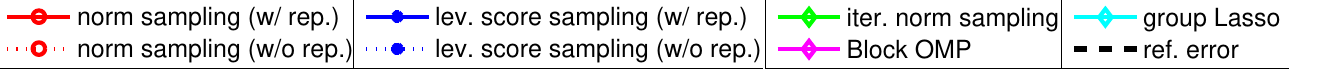}
\includegraphics[width=5cm]{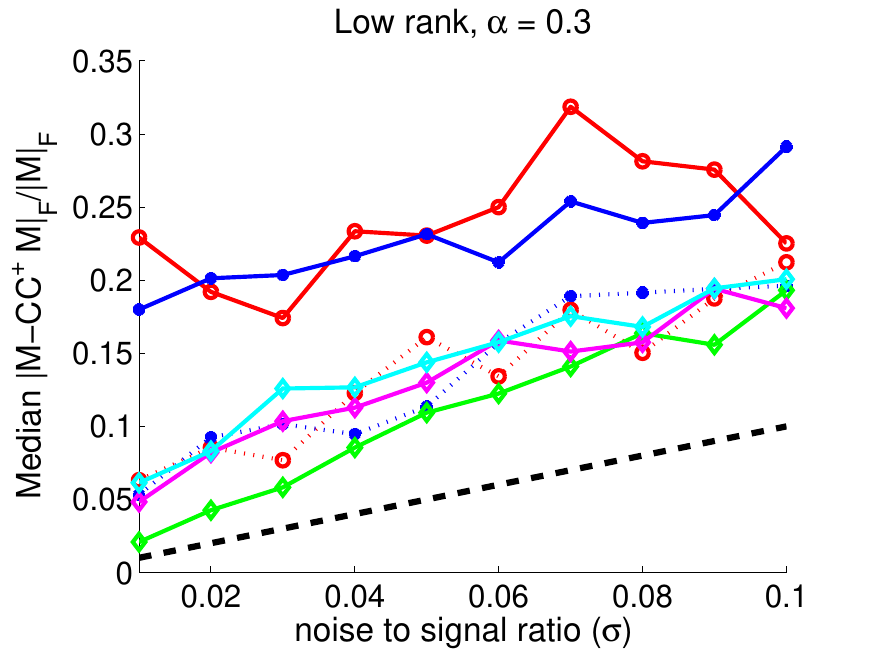}
\includegraphics[width=5cm]{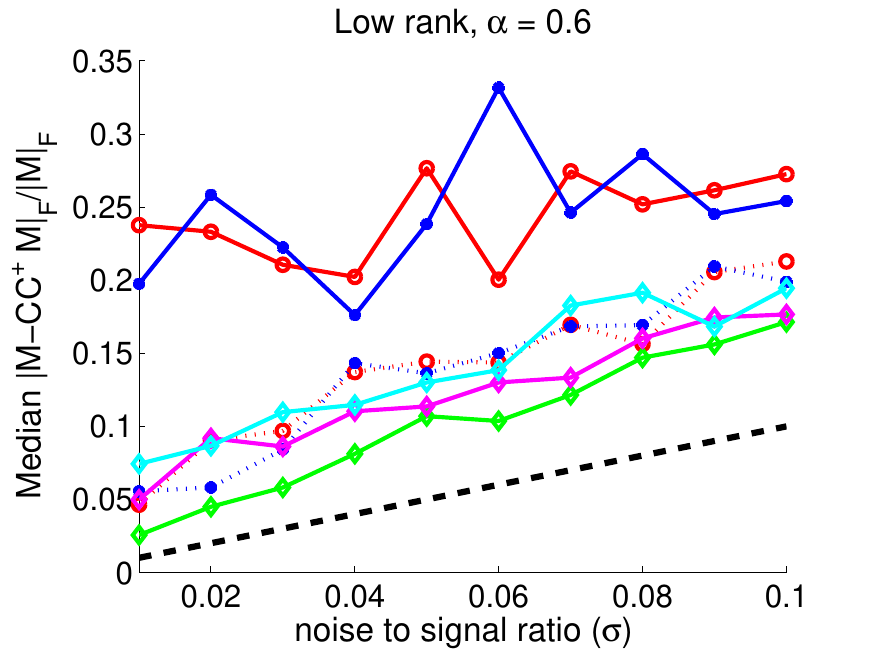}
\includegraphics[width=5cm]{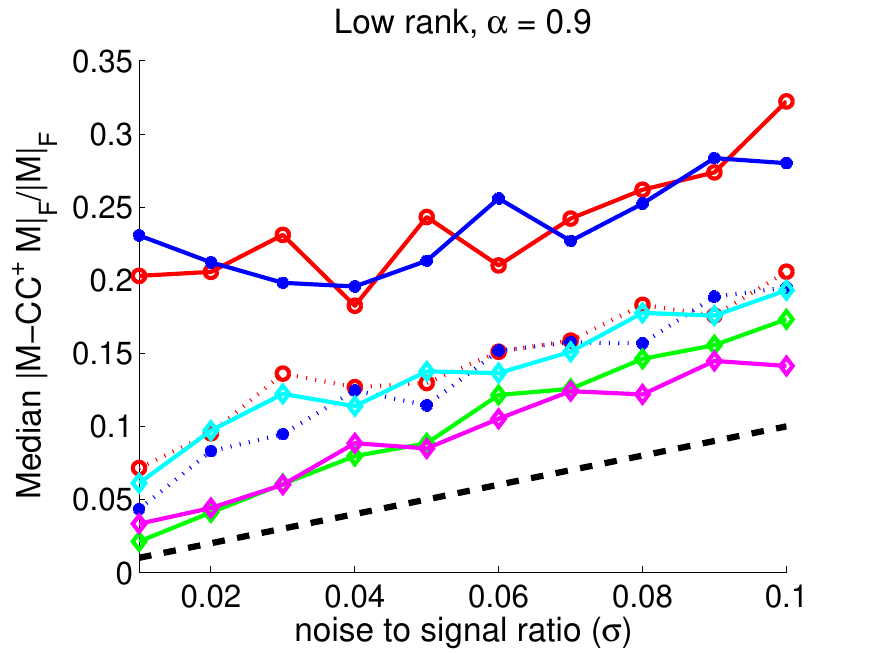}
\includegraphics[width=5cm]{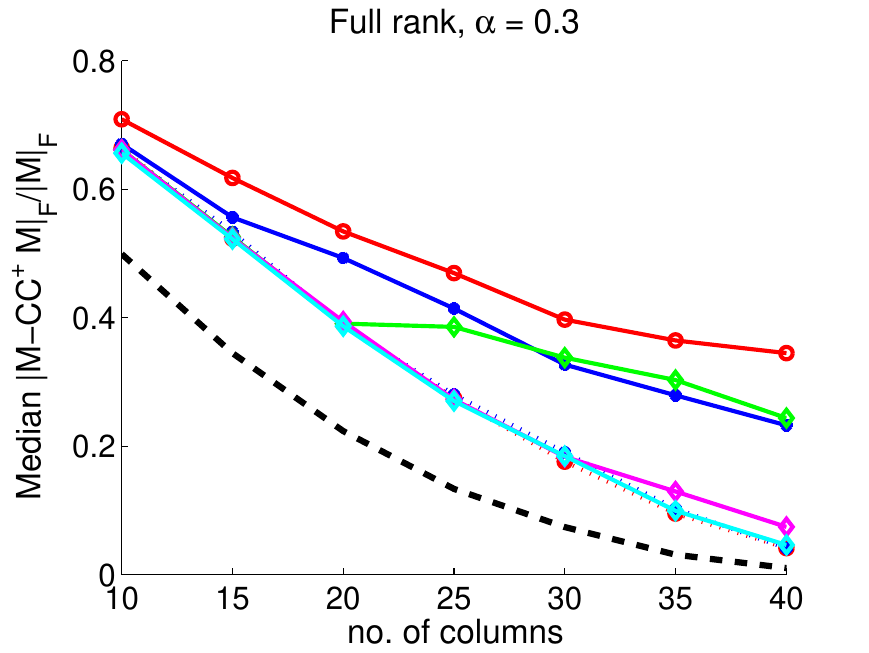}
\includegraphics[width=5cm]{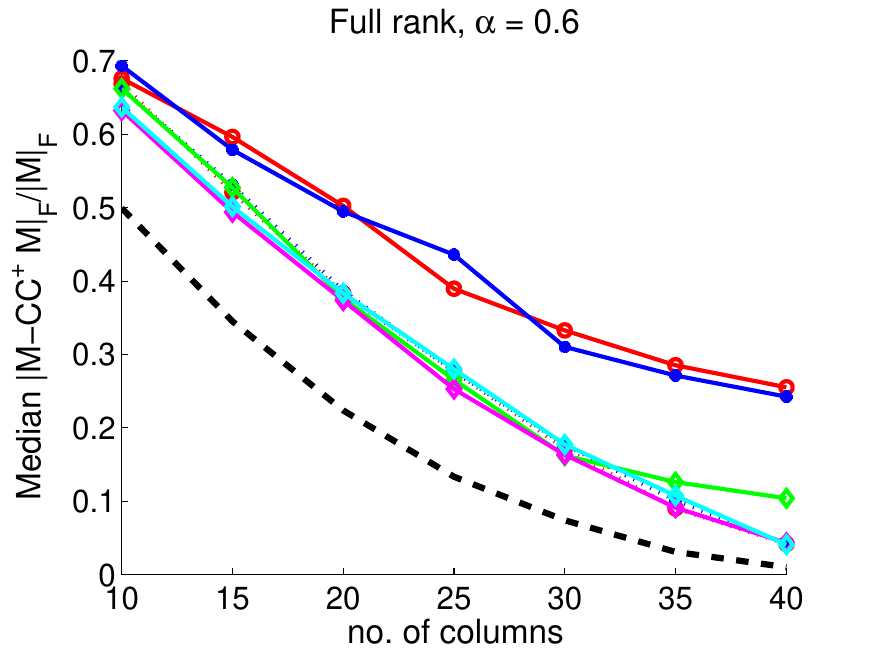}
\includegraphics[width=5cm]{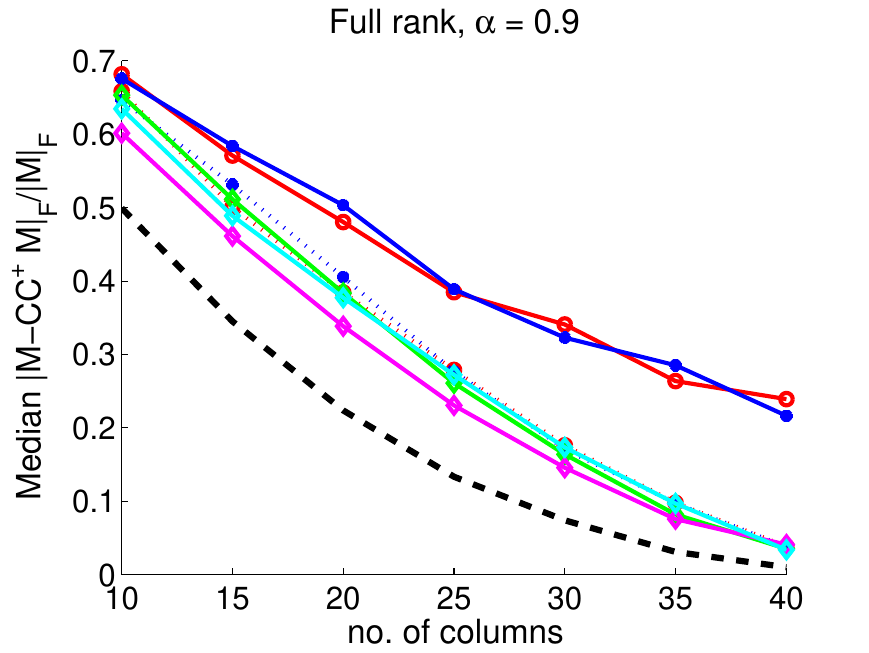}
\caption{Selection error on Gaussian random matrices. Top row: low-rank plus noise inputs, $s=k=15$;
bottom row: full-rank inputs.
The black dashed lines denote noise-to-signal ratio $\sigma$ in the first row and $\|\mat M-\mat M_k\|_F$ in the second row.
$\alpha$ indicates the observation rate (i.e., the number of observed entries divided by $n_1n_2$, the total number of matrix entries).
All algorithms are run for 8 times on each data set and the median error is reported.
We report the median instead of the mean because the performance of norm and leverage score sampling is quite variable.}
\label{fig_compare_random}
\end{figure*}

\begin{figure*}[t]
\centering
\includegraphics[width=12cm]{legend.pdf}
\includegraphics[width=5cm]{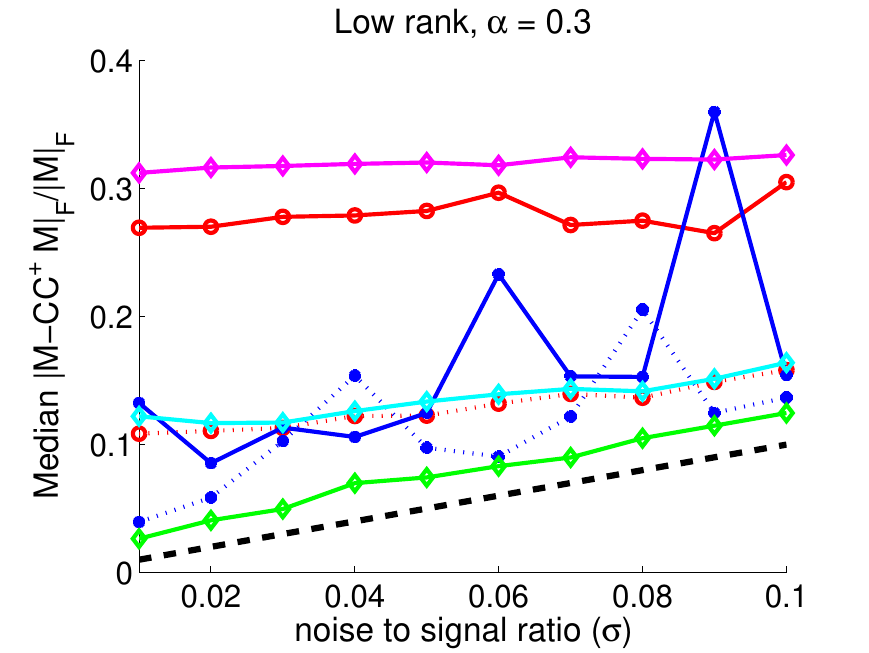}
\includegraphics[width=5cm]{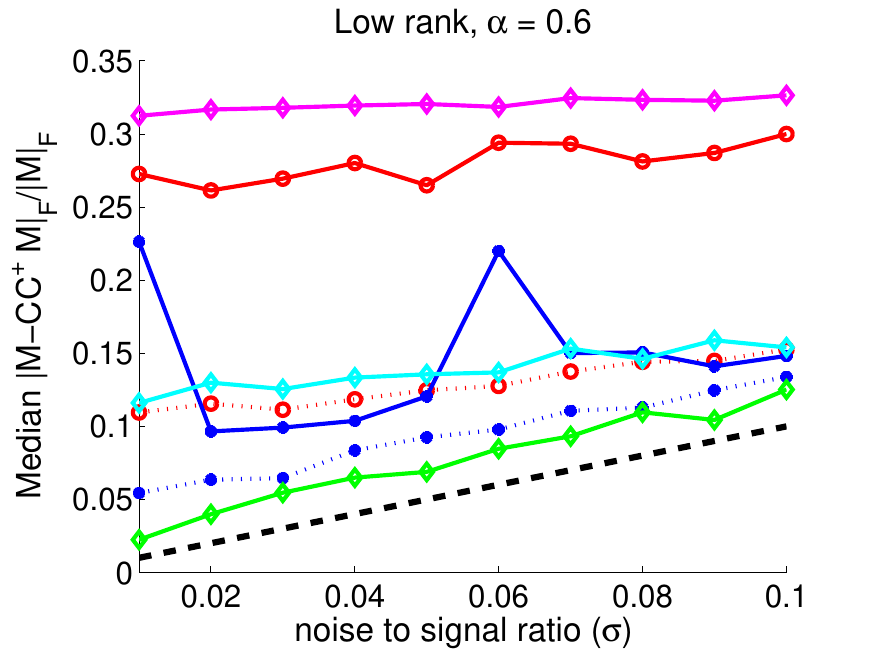}
\includegraphics[width=5cm]{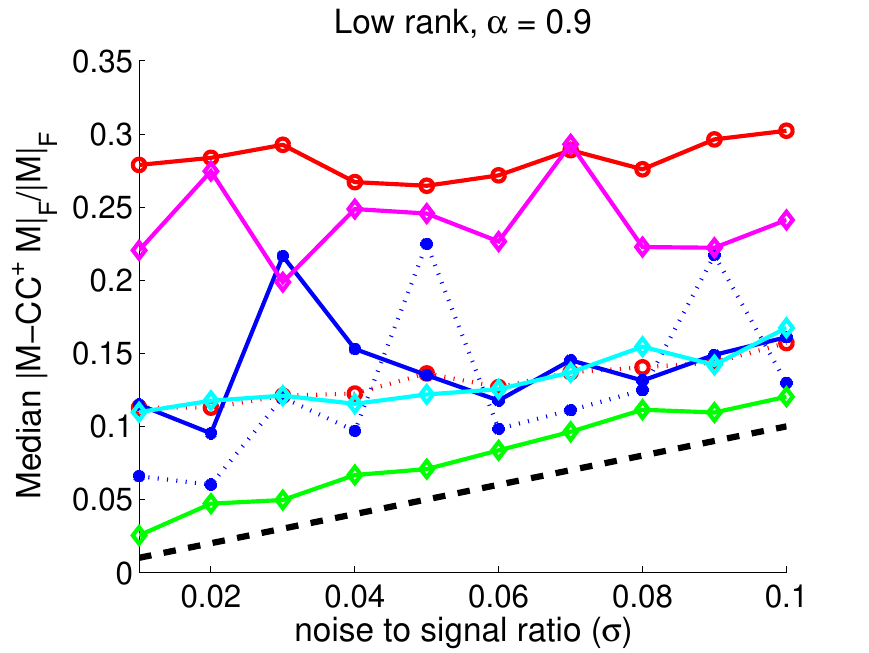}
\includegraphics[width=5cm]{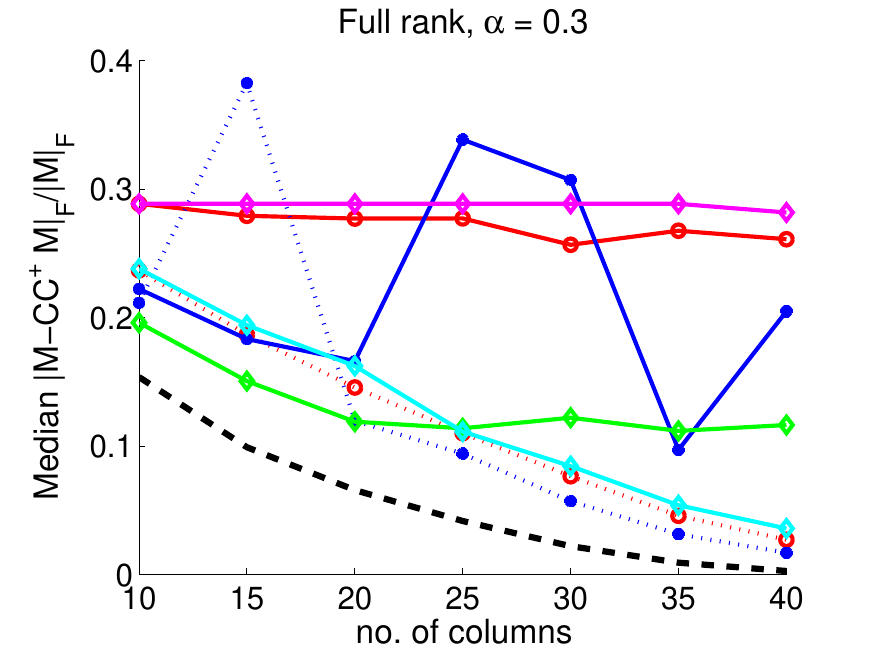}
\includegraphics[width=5cm]{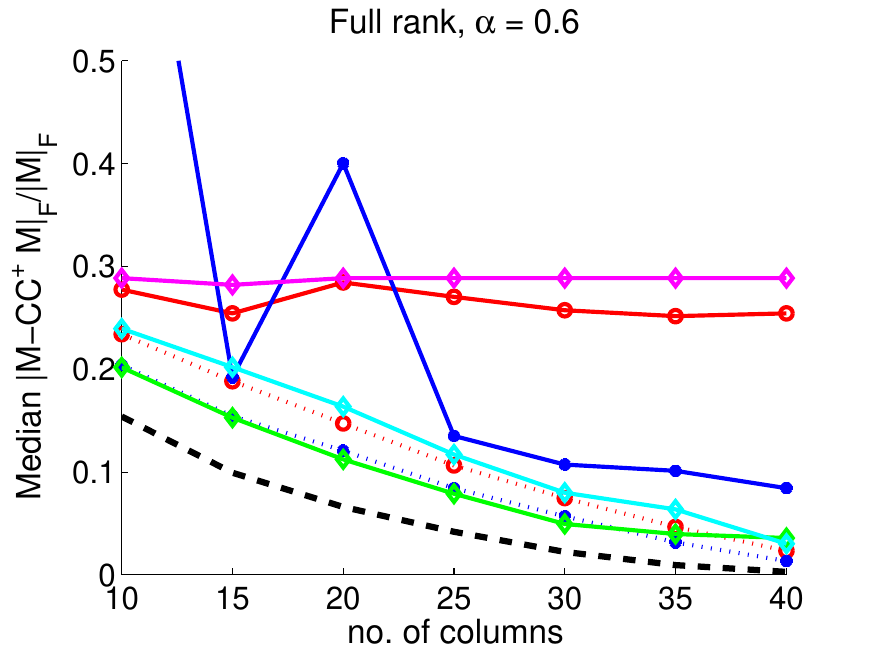}
\includegraphics[width=5cm]{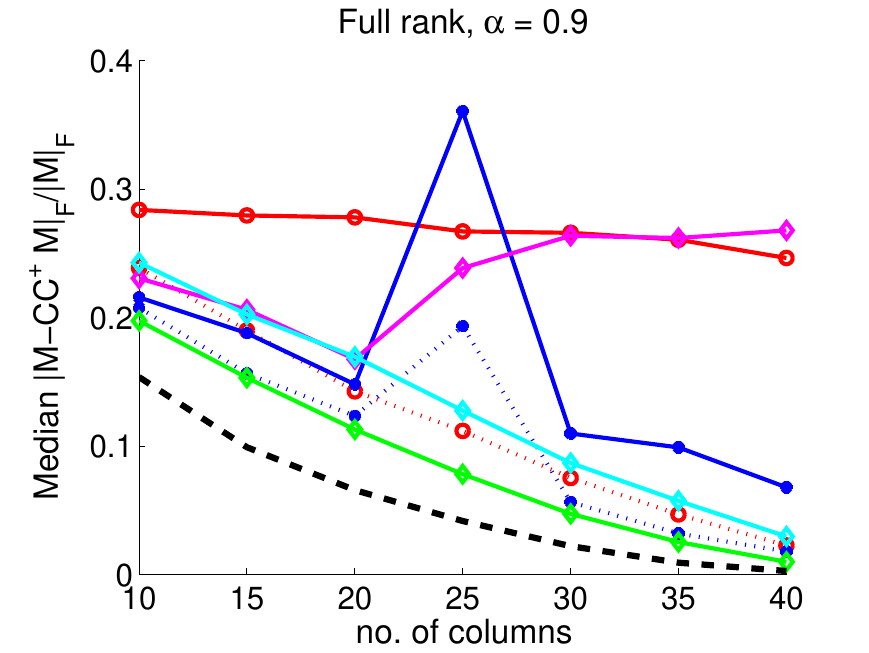}
\caption{Selection error on matrices with coherent columns. Top row: low-rank plus noise inputs, $s=k=15$;
bottom row: full-rank inputs.
$\alpha$ indicates the observation rate.
The black dashed lines denote noise-to-signal ratio $\sigma$ in the first row and $\|\mat M-\mat M_k\|_F$ in the second row.
All algorithms are run for 8 times on each data set and the median error is reported.}
\label{fig_compare_coherent}
\end{figure*}

\begin{figure*}[t]
\centering
\includegraphics[width=12cm]{legend.pdf}
\includegraphics[width=5cm]{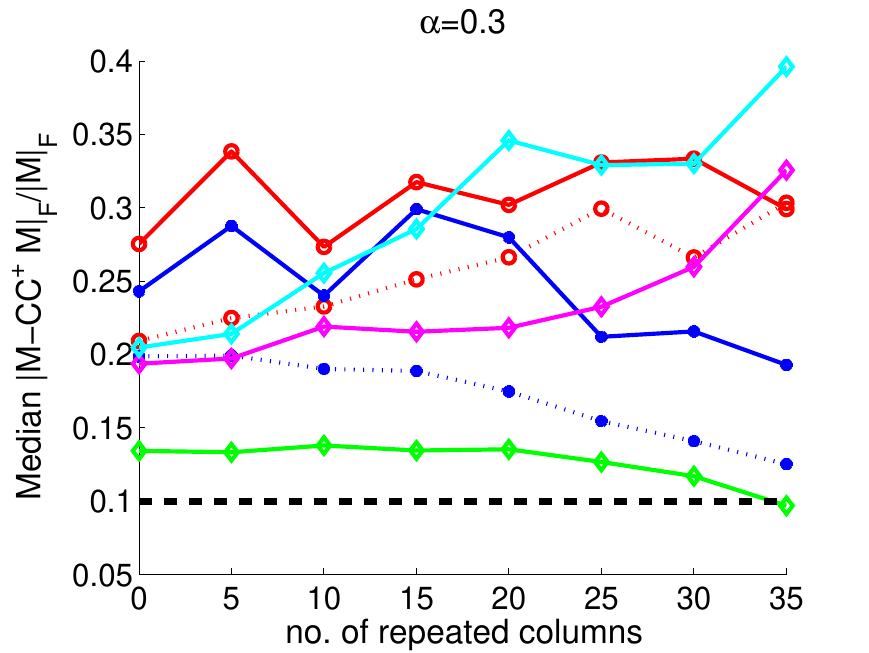}
\includegraphics[width=5cm]{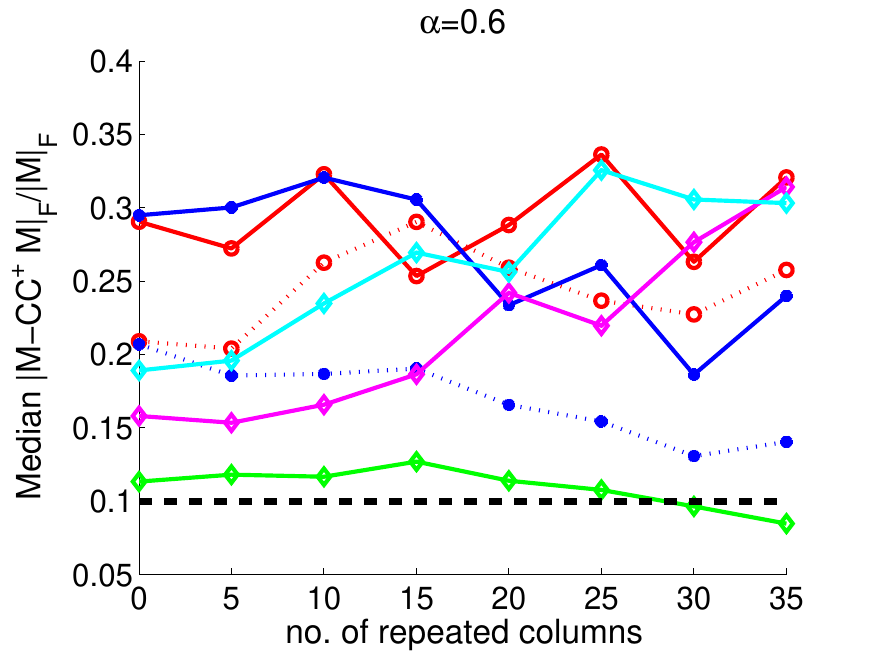}
\includegraphics[width=5cm]{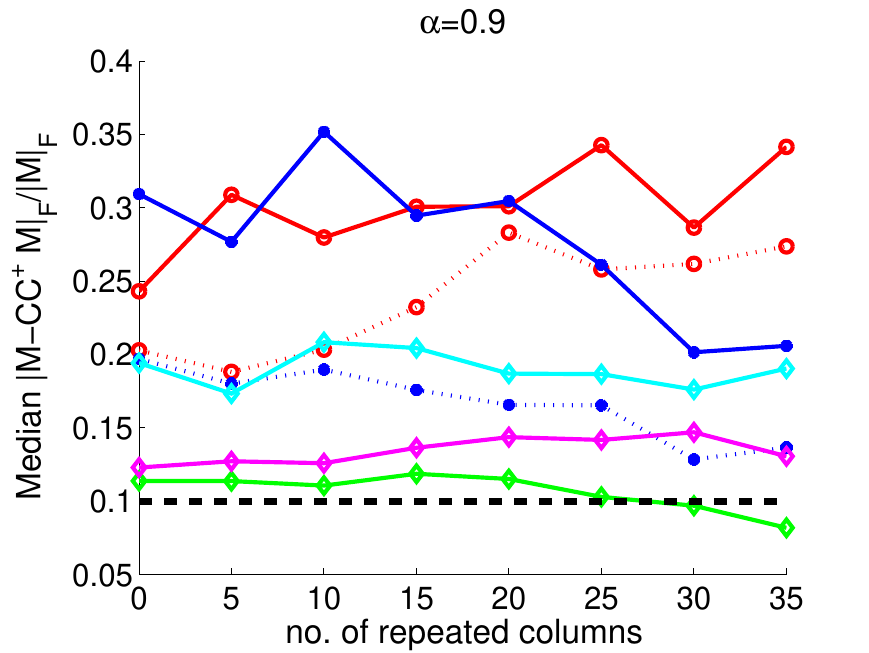}
\caption{Selection error on matrices with varying number of repeated columns.
Both $s$ and $k$ are set to 15 and the noise-to-signal ratio $\sigma$ is set to 0.1.
$\alpha$ indicates the observation rate.
All algorithms are run for 8 times on each data set and the median error is reported.}
\label{fig_compare_repeated}
\end{figure*}

\begin{figure}[t]
\centering
\includegraphics[width=5cm]{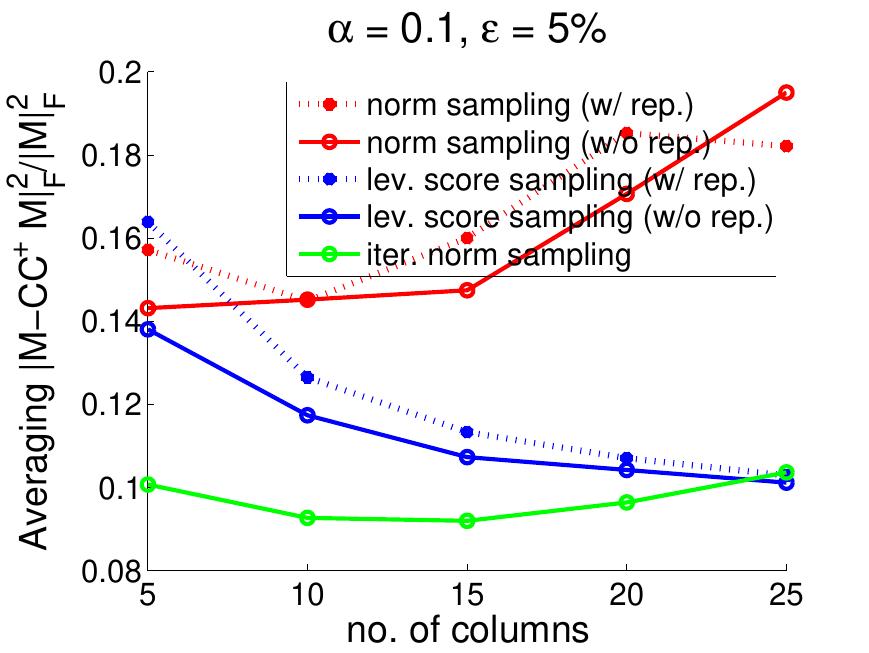}
\includegraphics[width=5cm]{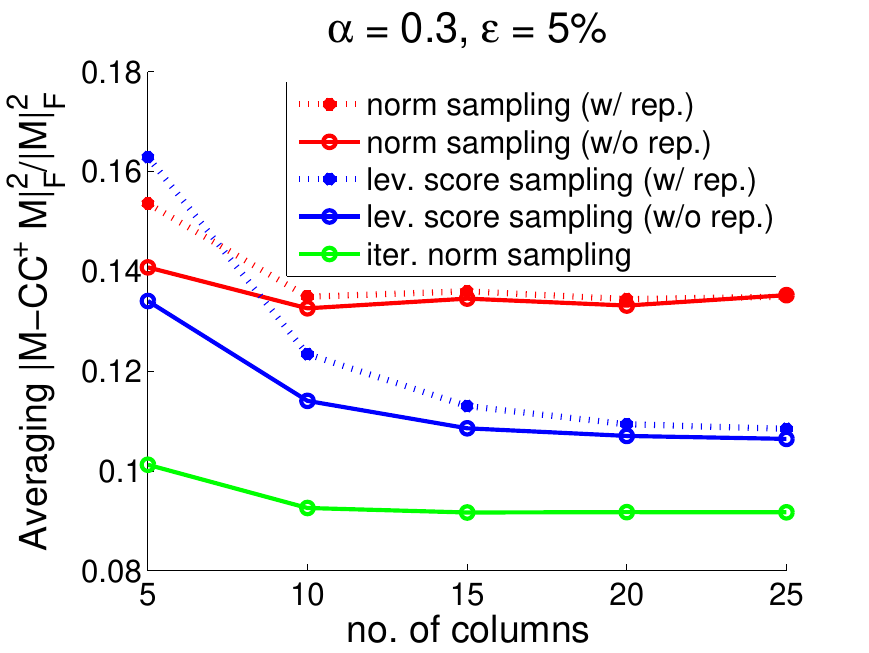}
\includegraphics[width=5cm]{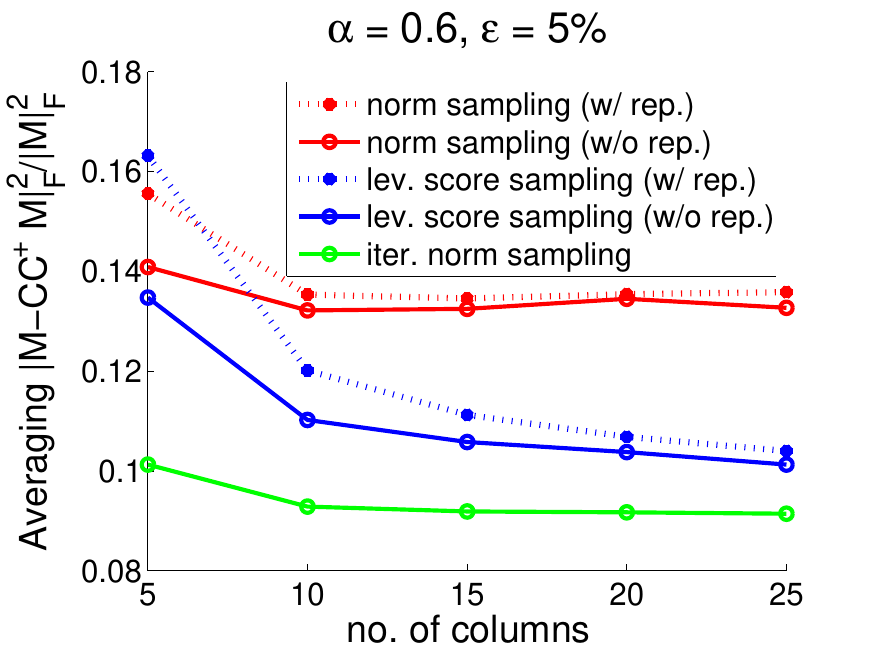}
\includegraphics[width=5cm]{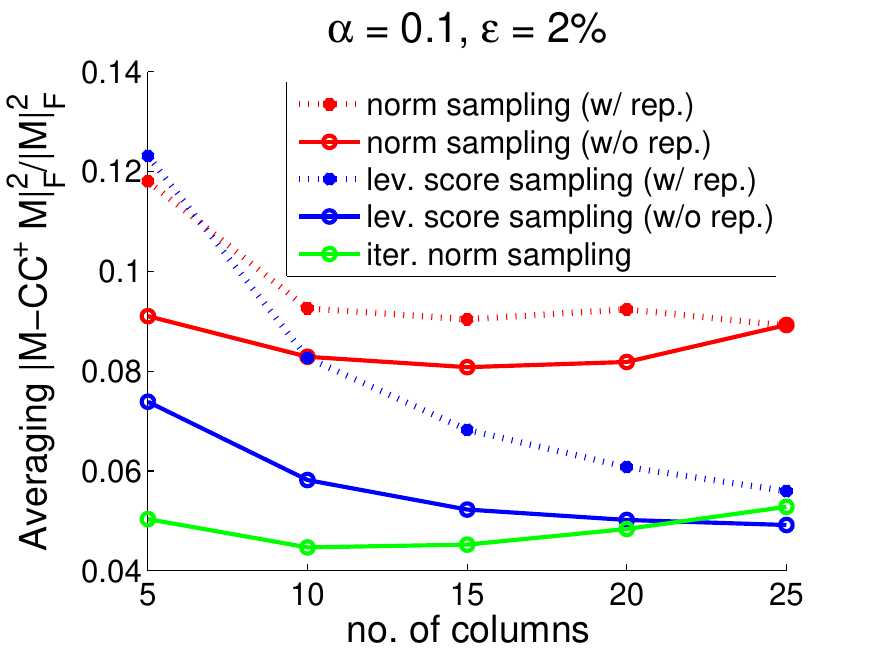}
\includegraphics[width=5cm]{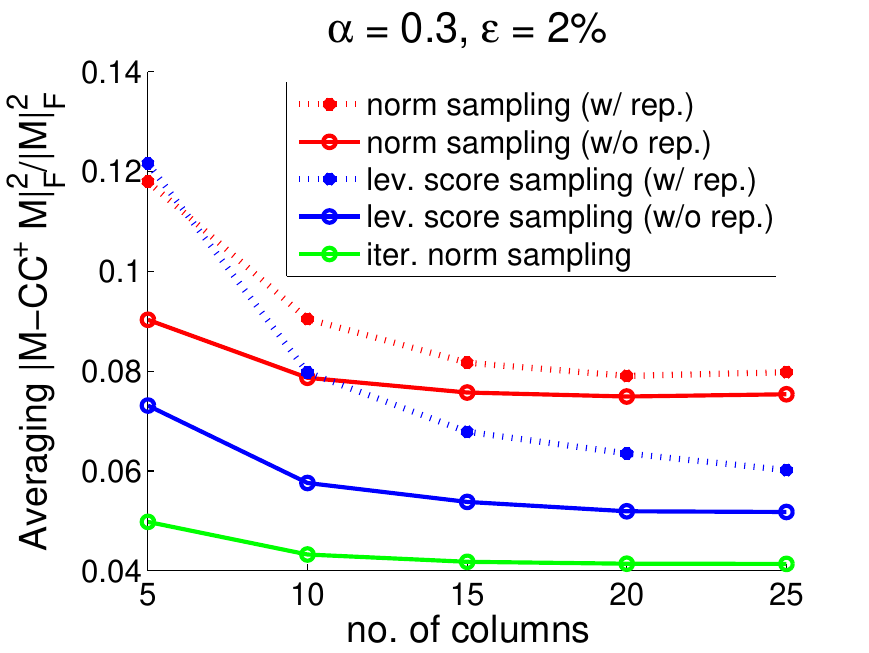}
\includegraphics[width=5cm]{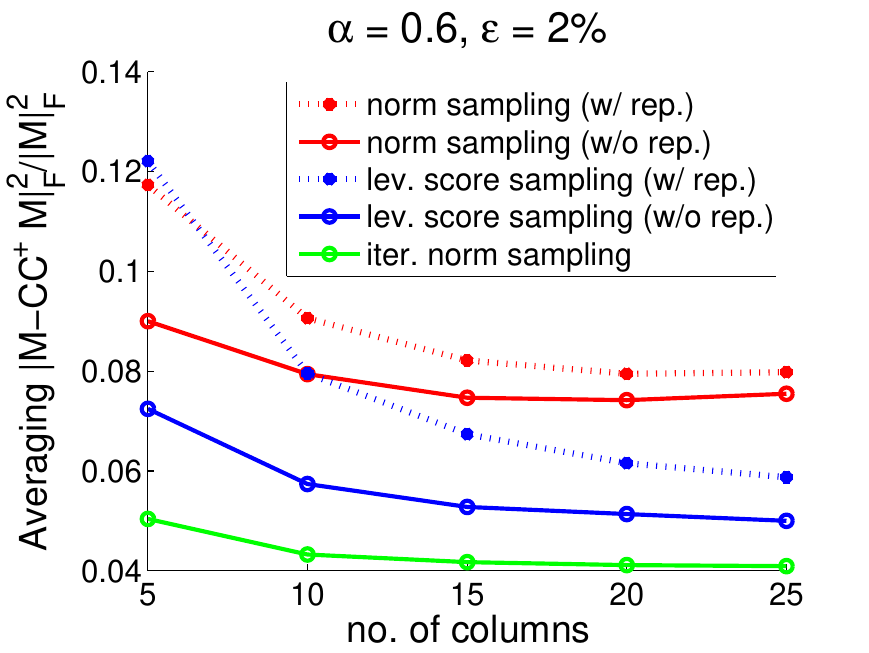}
\caption{Selection error or sampling based algorithm on Hapmap phase2 data set.
$\alpha$ indicates the observation rate.
Top row: top-$k$ PCA captures 95\% variance within each SNP window; bottom row: top-$k$ PCA captures 98\% variance within each SNP window.}
\label{fig_hapmap}
\end{figure}

\begin{figure}[t]
\centering
\includegraphics[width=6cm]{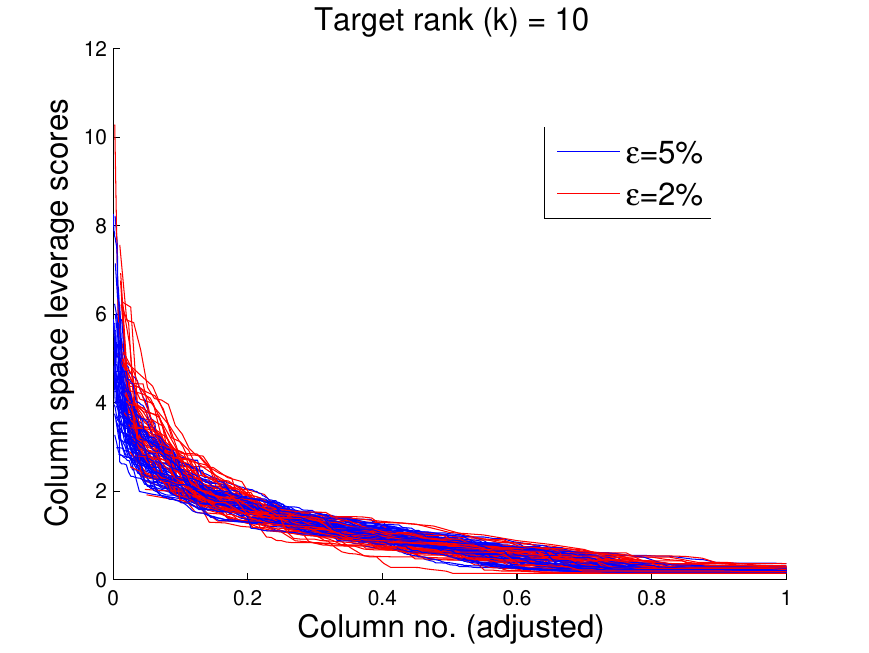}
\includegraphics[width=6cm]{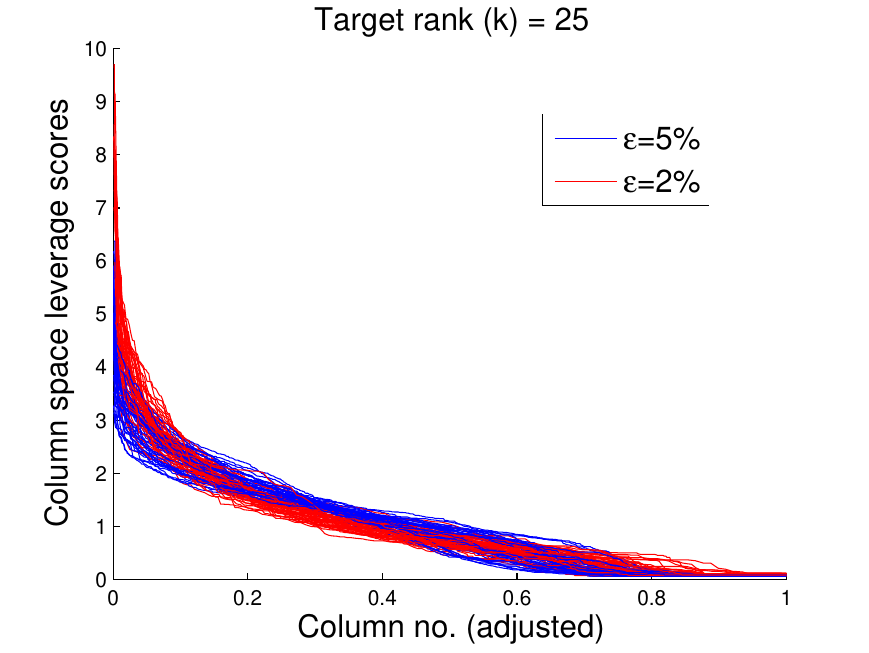}
\caption{Sorted column space leverage scores for different $\varepsilon$ and $k$ settings.
For each setting 50 windows are picked at random and their leverage scores are plotted.
Each plotted line is properly scaled along the X axis so that they have the same length even though actual window sizes vary.}
\label{fig_hapmap_levscore}
\end{figure}

\begin{figure}[p]
\centering
\begin{subfigure}[b]{1.0\textwidth}
\centering
\includegraphics[width=5cm]{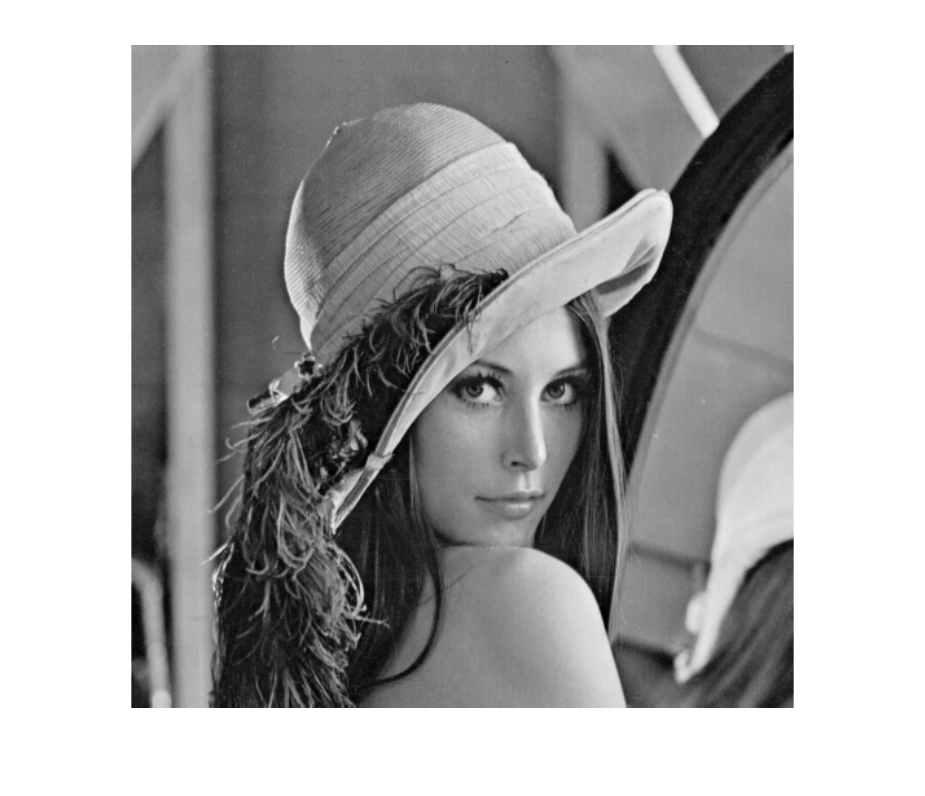}
\caption{The $512\times 512$ 8-bit gray scale Lena test image before compression.}
\label{fig_lena_original}
\end{subfigure}
\begin{subfigure}[b]{1.0\textwidth}
\includegraphics[width=5cm]{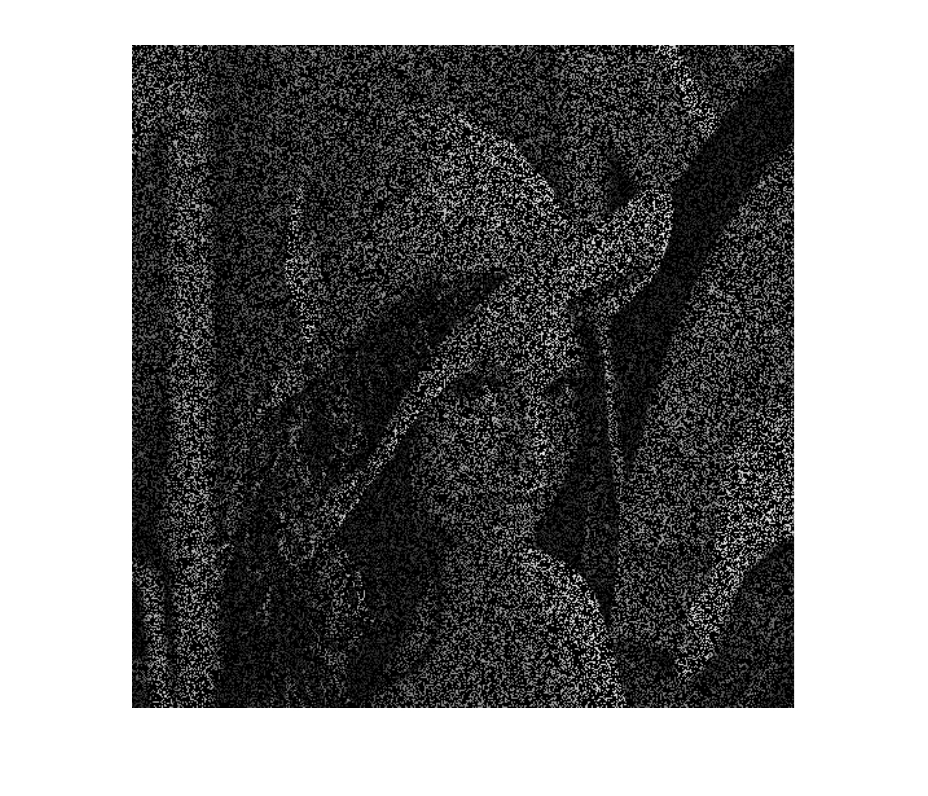}
\includegraphics[width=5cm]{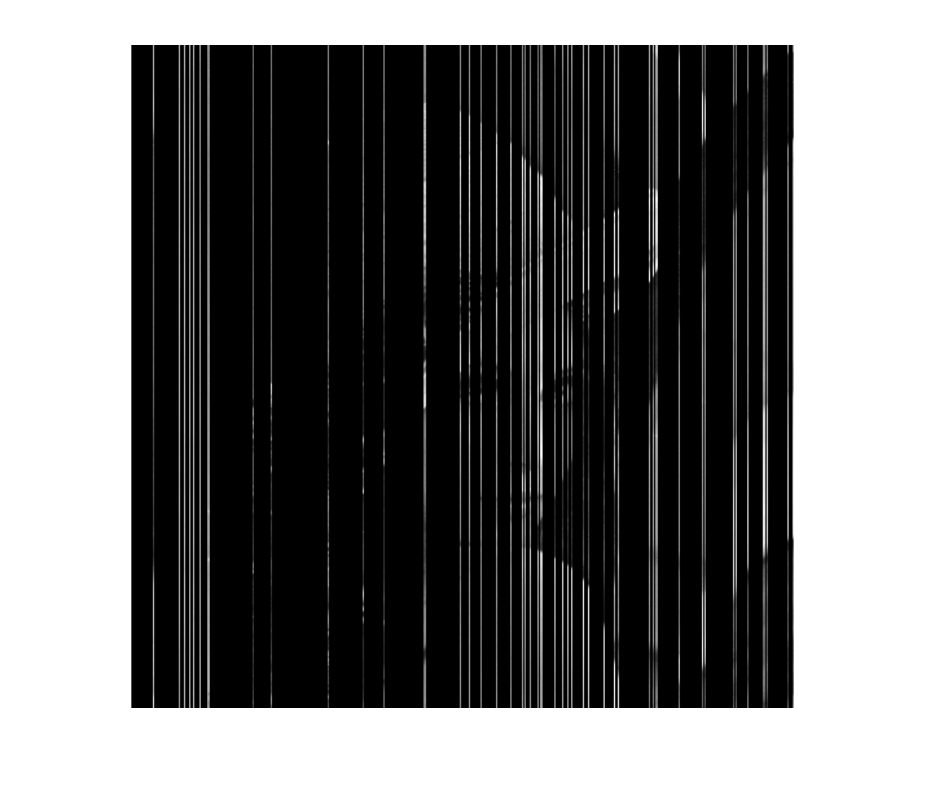}
\includegraphics[width=5cm]{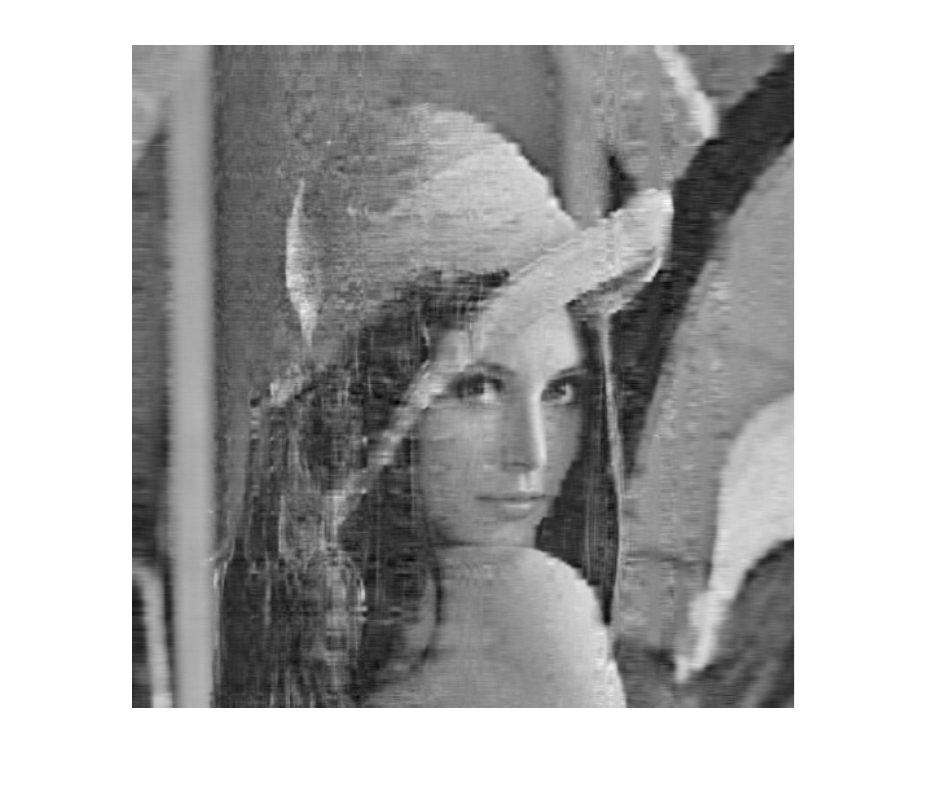}
\caption{Norm sampling (without replacement). Selection error $\|\mat M-\mat C\mat C^\dagger\mat M\|_F/\|\mat M\|_F = 0.106$.}
\label{fig_lena1}
\end{subfigure}
\begin{subfigure}[b]{1.0\textwidth}
\includegraphics[width=5cm]{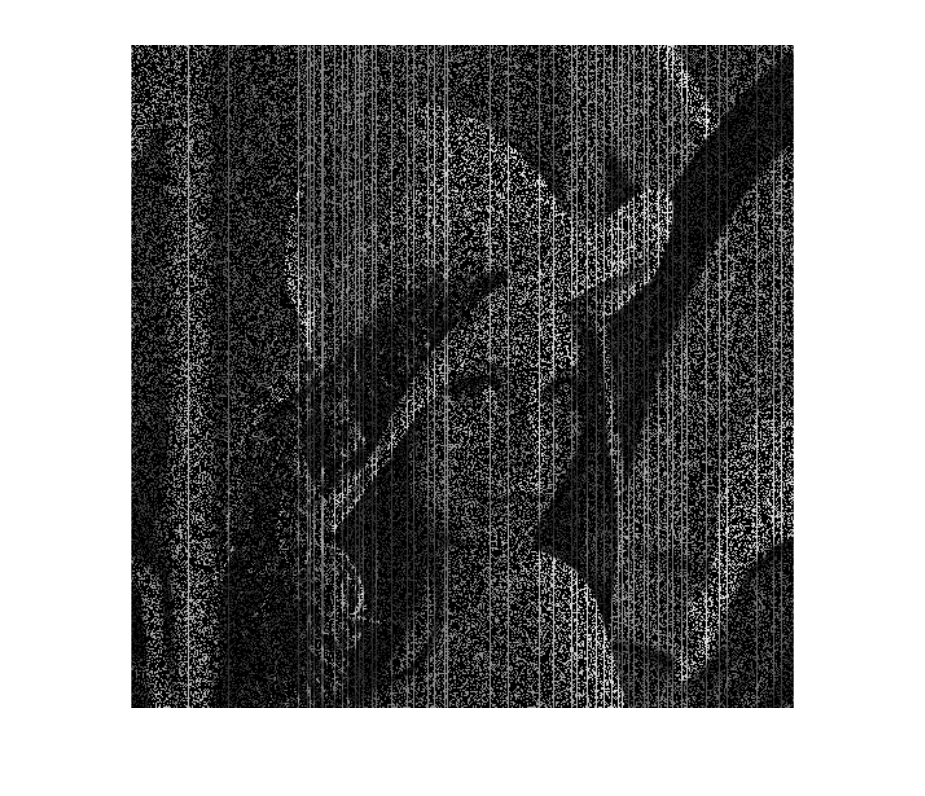}
\includegraphics[width=5cm]{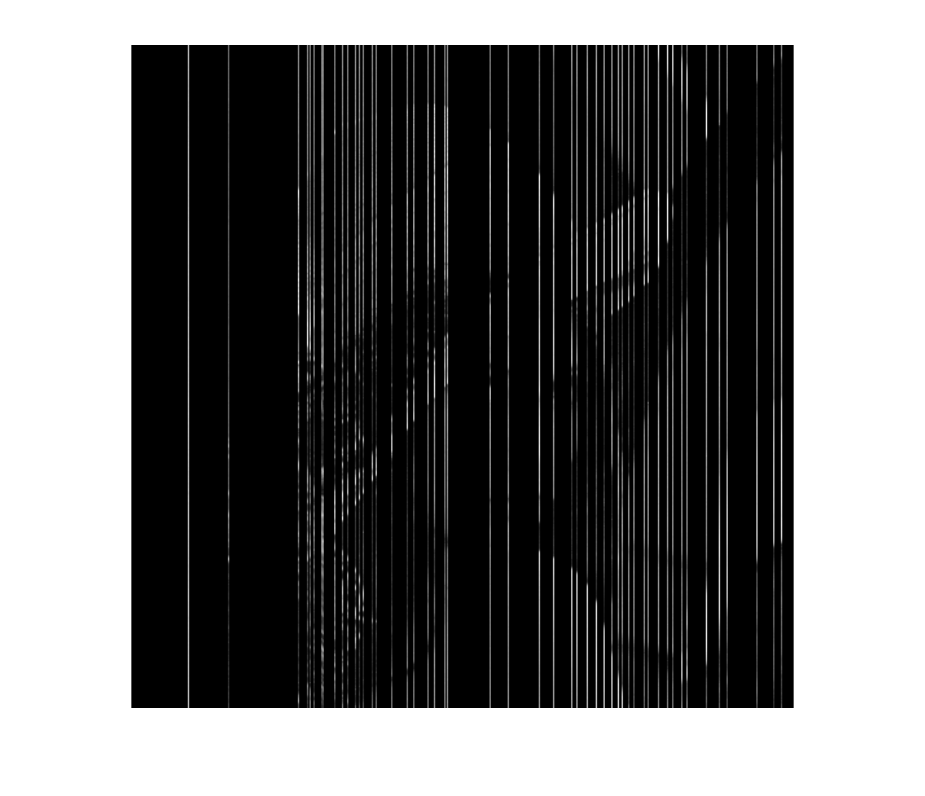}
\includegraphics[width=5cm]{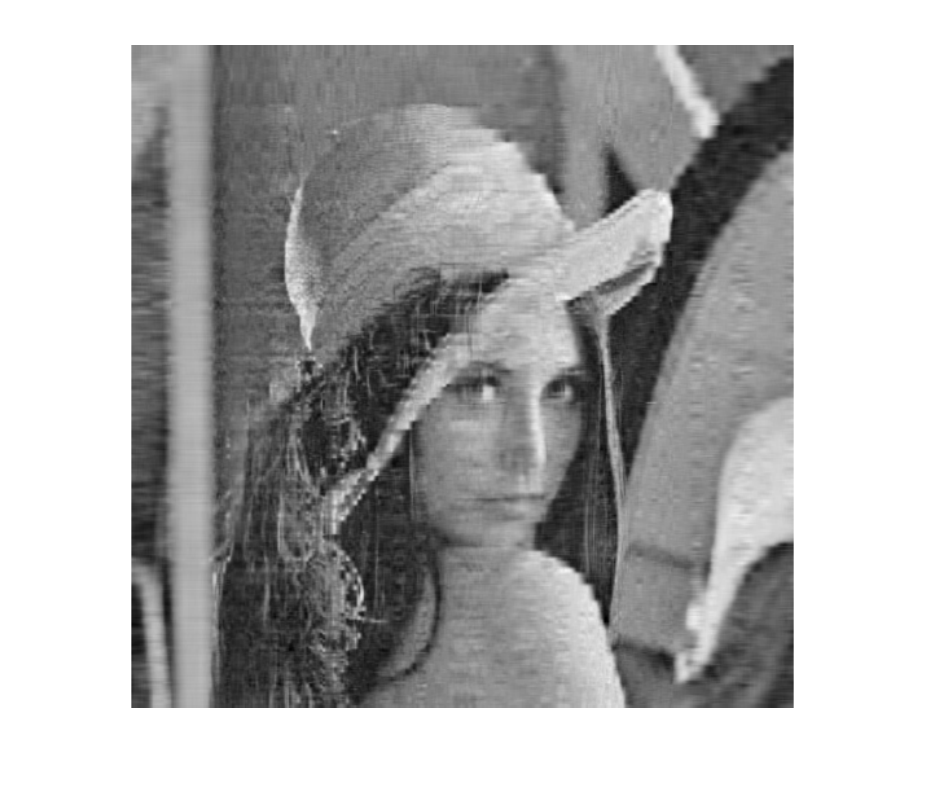}
\caption{Iterative norm sampling. Selection error $\|\mat M-\mat C\mat C^\dagger\mat M\|_F/\|\mat M\|_F = 0.088$.}
\label{fig_lena2}
\end{subfigure}
\begin{subfigure}[b]{1.0\textwidth}
\includegraphics[width=5cm]{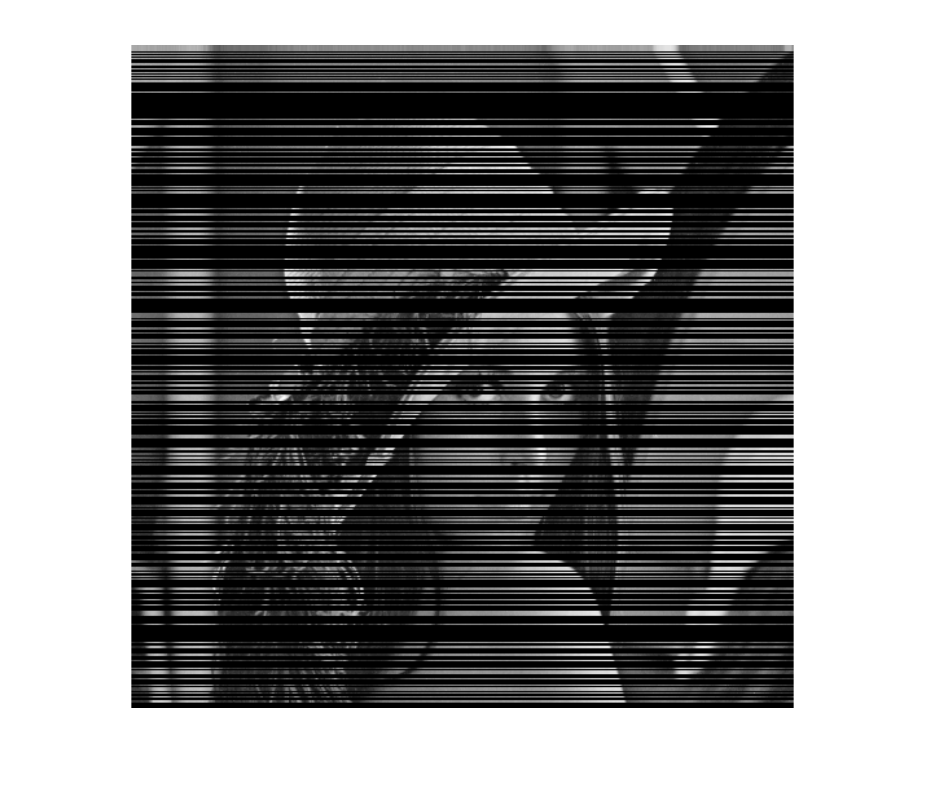}
\includegraphics[width=5cm]{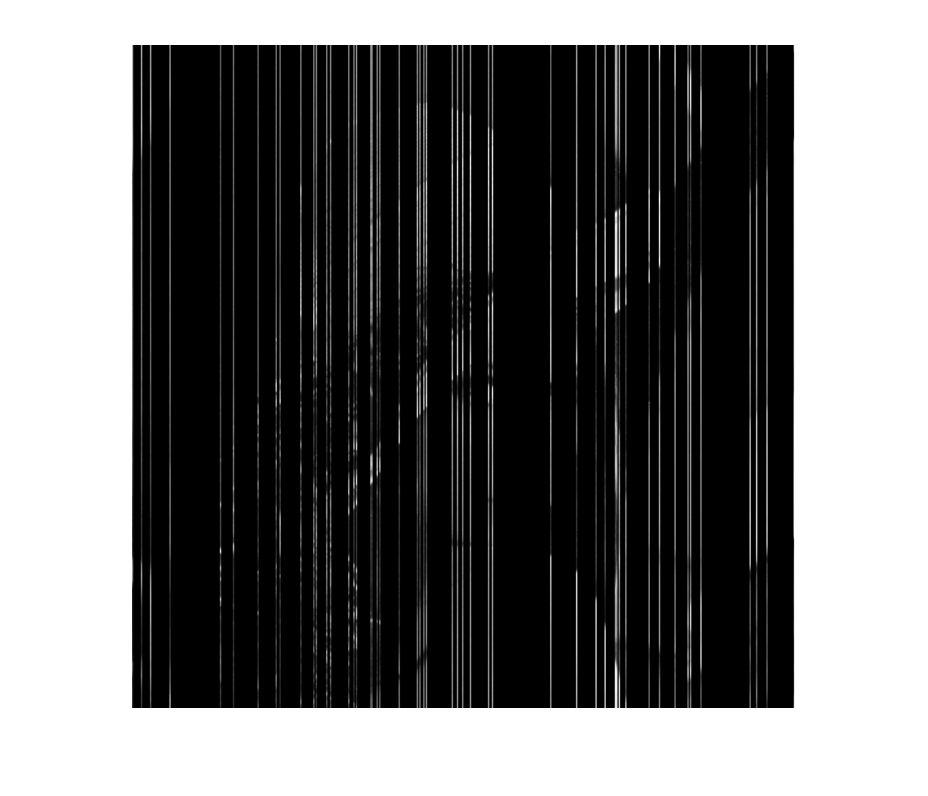}
\includegraphics[width=5cm]{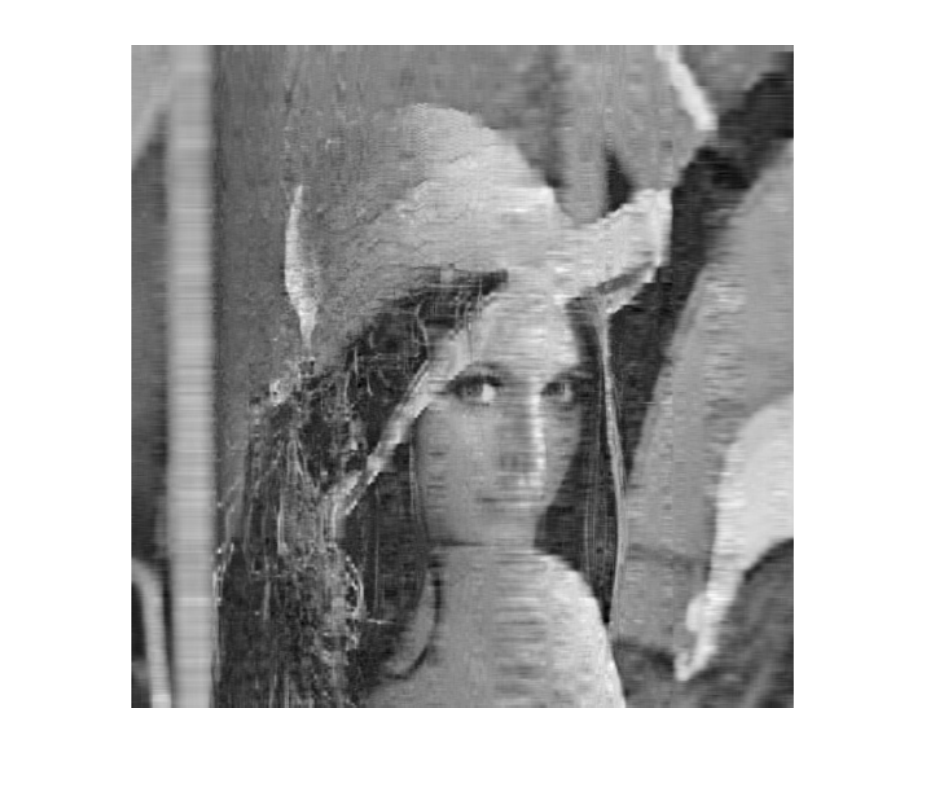}
\caption{Approx. leverage score sampling (without replacement). $\|\mat M-\mat C\mat C^\dagger\mat M\|_F/\|\mat M\|_F = 0.103$.}
\label{fig_lena3}
\end{subfigure}
\caption{Column-based image compression results on the Lena standard test image. 
Left: actively sampled image pixels;
middle: the selected columns;
right: the reconstructed images.
Number of selected columns is set to 50 and the pixel subsampling rate $\alpha$ is set to 0.3.}	
\label{fig_lena}
\end{figure}

\begin{figure}[t]
\centering
\includegraphics[width=5cm]{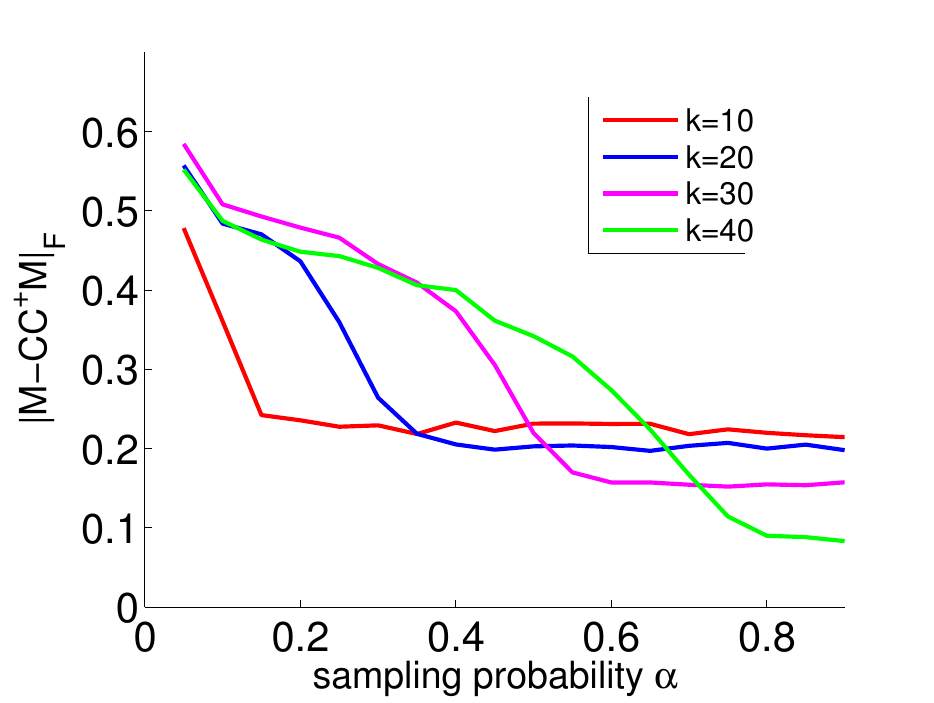}
\includegraphics[width=5cm]{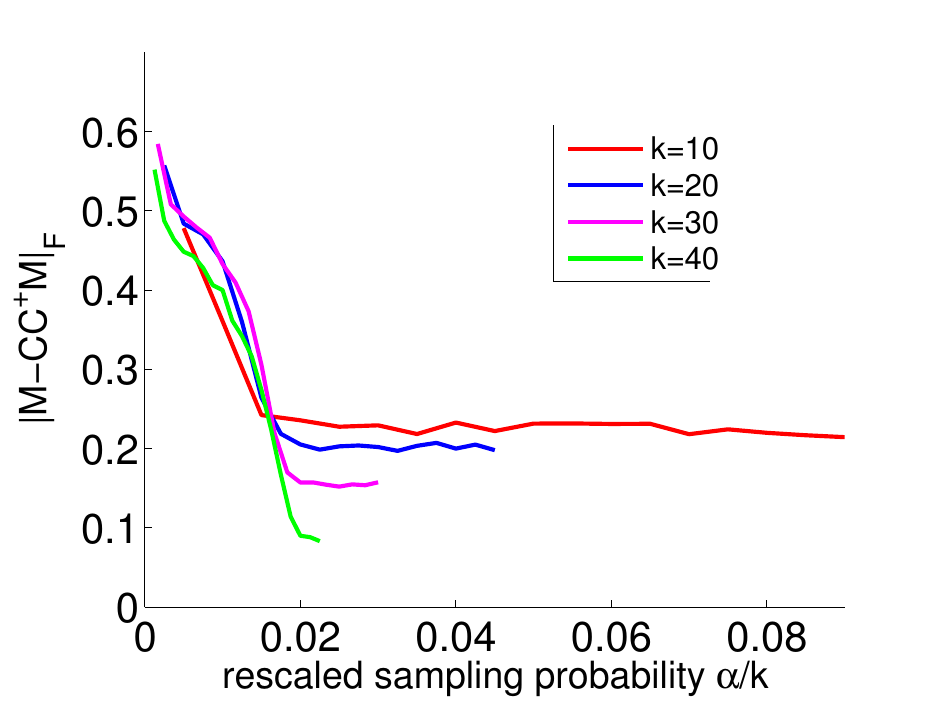}
\includegraphics[width=5cm]{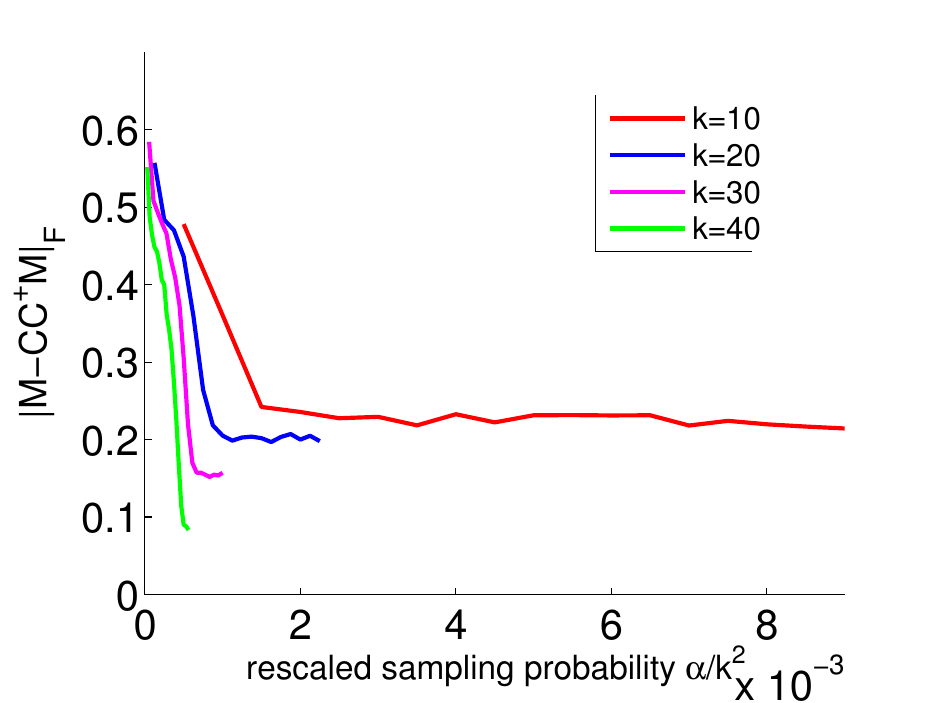}
\caption{Selection error $\|\mat M-\mat C\mat C^\dagger\mat M\|_F$ for the iterative norm sampling algorithm
as a function of $\alpha$ (left), $\alpha/k$ (middle) and $\alpha/k^2$ (right).
Error curves plotted under 4 different rank ($k$) settings.}
\label{fig_dependence_r}
\end{figure}

In this section we report experimental results on both synthetic and real-world data sets for our proposed column subset selection algorithms
as well as other competitive methods.
All algorithms are implemented in Matlab.
To make fair comparisons, all input matrices $\mat M$ are normalized so that $\|\mat M\|_F^2=1$.

\subsection{Synthetic data sets}

We first test the proposed algorithms on synthetic data sets.
The input matrix has dimension $n_1=n_2=n=50$.
To generate the synthetic data, we consider two different settings listed below:
\begin{enumerate}
\item \textbf{Random Gaussian matrices}: for random Gaussian matrices each entry $\mat M_{ij}$ are i.i.d. sampled from a normal distribution
$\nml(0,1)$. 
For low rank matrices, we first generate a random Gaussian matrix $\mat B\in\mathbb R^{n\times k}$ where $k$ is the intrinsic rank
and then form the data matrix $\mat M$ as $\mat M=\mat B\mat B^\top$.
I.i.d. Gaussian noise $\mat R$ with $\mat R_{ij}\sim\nml(0,\sigma^2)$ is then appended to the synthesized low-rank matrix.
We remark that data matrices generated in this manner have both incoherent column and row space with high probability.

\item \textbf{Matrices with coherent columns}: 
we took a simple procedure to generate matrices with coherent columns in order to highlight the power of proposed algorithms and baseline methods.
After generating a random Gaussian matrix $\mat M=\mat B\mat B^\top$, we pick a column $\vct x$ from $\mat M$ uniformly at random.
We then take $\tilde{\vct x}=10\vct x$ and repeat the column for 5 times.
As a result, the newly formed data matrix will have 5 identical columns with significantly higher norms compared to the other columns.
\end{enumerate}

In Figure \ref{fig_compare_random} we report the selection error $\|\mat M-\mat C\mat C^\dagger\mat M\|_F$
of proposed and baseline algorithms on random Gaussian matrices and in Figure \ref{fig_compare_coherent}
we report the same results on matrices with coherent columns.
Results on both low-rank plus noise and high-rank inputs are reported.
For low-rank matrices, both the intrinsic rank $k$ and the number of selected columns $s$ are set to 15.
Each algorithm is run for 8 times on the same input and the median selection error is reported.
For norm sampling and approximate leverage score sampling, 
we implement two variants: in the \emph{sampling with replacement} scheme the algorithm samples each column from a sampling distribution
(based on either norm or leverage score estimation) with replacement;
while in the \emph{sampling without replacement} scheme a column is never sampled twice.
Note that all theoretical results in Section \ref{sec:main_result} are proved for sampling with replacement algorithms.

From Figure \ref{fig_compare_random} we observe that all algorithms perform similarly,
with the exception of two sampling with replacement algorithms and iterative norm sampling when both rank and missing rate are high.
\footnote{We discuss on the poor performance of with replacement algorithms in Section \ref{subsec:replacement}.}
For the latter case, we conjecture that the degradation of performance is due to inaccurate norm estimation of column residues;
in fact, the iterative norm sampling only provably works when the input matrix has a low-rank plus noise structure (see Theorem \ref{thm_relative_cx}).
On the other hand, when either the target rank or the missing rate is not too high iterative norm sampling works just as good;
it is particularly competitive when the true rank of the input matrix is low (see the top row of Figure \ref{fig_compare_random}).

When the input matrix has coherent columns, as shown in Figure \ref{fig_compare_coherent},
it becomes easier to observe performance gaps among different algorithms.
The block OMP algorithm completely fails in such cases and the selection error for group Lasso also increases considerably.
This is due to the fact that both algorithms observe matrix entries by sampling uniformly at random and hence
could be poorly informed when the underlying matrix is highly coherent.
On the other hand, both leverage score sampling and iterative norm sampling are more robust to column coherence.
The coherence among columns also makes the separation between norm sampling and volume sampling clearer in Figure \ref{fig_compare_coherent}.
In particular, there is a significant gap between the two sampling with replacement curves and 
the norm sampling algorithm degrades to its worst-case additive error bound (see Theorem \ref{thm_additive_cx}).
The gap between the sampling without replacement curves is smaller since the coherent column is only repeated for 5 times in the design
and so an algorithm can not be ``too wrong'' if it samples columns without replacement.

To further investigate how the proposed and baseline algorithms adapt to different levels of coherence,
we report in Figure \ref{fig_compare_repeated} the selection error on noisy low-rank matrices with varying number of repeated columns.
Matrices with more repeated columns have higher coherence level.
We can see that there is a clear separation of two groups of algorithms:
the first group includes norm sampling, block OMP and group Lasso,
whose error increases as the matrix becomes more coherent.
Also, design matrix assumptions (e.g., restricted isometry) are violated for group Lasso.
This suggests that these algorithms only have additive error bounds, or adapt poorly to column coherence of the underlying data matrix.
On the other hand, the selection error of volume sampling and iterative norm sampling remains stable or slightly decreases.
This is consistent with our theoretical results that both volume sampling and iterative norm sampling enjoy relative error bounds.



\subsection{Application to tagging Single Nucleotide Polymorphisms (tSNPs) selection}\label{sec:snp}

\begin{table}[t]
\centering
\caption{Averaging SNP window sizes for different $\varepsilon$ values and number of selected columns per window.}
\begin{tabular}{cccccc}
\hline
\abovespace\belowspace
& \textsc{5 columns}& \textsc{10 columns}& \textsc{15 columns}&\textsc{20 columns}& \textsc{25 columns}\\
\hline
\abovespace
$\varepsilon=95\%$& 63.4& 248.9& 516.3& 891.0& 1405.7\\
\belowspace
$\varepsilon=98\%$& 18.8& 62.1& 123.4& 203.8& 309.7\\
\hline
\end{tabular}
\label{tab_window_length}
\end{table}

We apply our proposed methods on real-world genetic data sets.
We consider the tagging Single Nucleotide Polymorphisms (tSNP) selection task as described in \citep{tsnp-selection,tsnp-selection-css}.
The task aims at selecting a small set of SNPs in human genes such that the selected SNPs (called tagging SNPs) capture the genetic information 
within a specific genome region.
More specifically, given an $n_1\times n_2$ matrix with each row corresponding to the genome expression for an individual,
we want to select $k$ columns (typically $k\ll n_2$) corresponding to $k$ tagging SNPs that best capture the entire SNP matrix across different individuals.
Matrix column subset selection methods have been successfully applied to the tSNP selection problem \citep{tsnp-selection-css}.

In this section we demonstrate that our proposed algorithms could achieve the same objective while allowing many missing entries in the raw data matrix.
We also compare the selection error of the proposed methods under different missing rate and number of tSNP settings.
We did not apply Block OMP and group Lasso because the former cannot handle coherent data matrices and the latter does not scale well.
The data set we used is the HapMap Phase 2 data set \citep{hapmap}.
For demonstration purposes, we use gene data for the first chromosome of  a joint east Asian population consisting of Han Chinese in Beijing (CHB)
and Japanese in Tokyo (JPT).
The data matrix consists of 89 rows (individuals) and 311,854 columns (SNPs).
Each matrix entry has two letters $b_1b_2$ describing a specific gene expression for an individual.

We follow the same step as described in \citep{snp-pca} to preprocess the data.
We first convert the raw data matrix into a numerical matrix $\mat M$ with +1/0/-1 entries as follows:
let $B_1$ and $B_2$ be the bases that appear for the $j$th SNP.
Fix an individual $i$ with its gene expression $b_1b_2$.
If $b_1b_2=B_1B_1$ then $\mat M_{ij}$ is set to -1;
else if $b_1b_2 = B_2B_2$ then $\mat M_{ij}$ is set 1;
otherwise $\mat M_{ij}$ is set to 0.
We further split the SNPs into multiple consecutive ``windows'' so that within each window $w$
the SVD reconstruction error $\|\mat M^{(w)}-\mat M^{(w)}_k\|_F^2/\|\mat M^{(w)}\|_F^2$
is no larger than $\varepsilon$ with $\varepsilon$ set to 5\% and 2\%.
We refer the readers to Figure 1 in \citep{snp-pca} for details of the preprocessing steps.
Averaging window length (i.e., number of SNPs within each window) are shown in Table \ref{tab_window_length}
for different $k$ and $\varepsilon$ settings.
After preprocessing, column subset selection algorithms are performed for each SNP window and the selection error
is averaged across all windows, as reported in Figure \ref{fig_hapmap}.
The number of selected columns per window ($k$) ranges from 5 to 25
and the sampling budget $\alpha$ ranges from 10\% to 60\%.

In Figure \ref{fig_hapmap} we observe that iterative norm sampling and approximate leverage score sampling
outperform norm sampling by a large margin.
This is because the truncated data matrix within each window is very close to an exact low-rank matrix
and hence relative error algorithms achieve much better performance than additive error ones.
In addition, approximate leverage score sampling significantly outperforms norm sampling under both the with replacement
and without replacement schemes.
This shows that the heterogeneity of human SNPs cannot be captured merely by their norms
because the norm is simply the proportion of heterozygous within a population and provides little information about its importance across the entire chromosome.
The spikiness of leverage score distribution is empirically verified in Figure \ref{fig_hapmap_levscore}.
Finally, we remark that sampling without replacement is much better than sampling with replacement and should always be preferred in practice.
We discuss this aspect in Section \ref{subsec:replacement}.

\subsection{Application to column-based image compression}\label{sec:lena}

In this section we show how active sampling can be applied to column-based image compression without observing entire images.
Given an image, we first actively subsample a small number of pixels from the original image.
We then select a subset of columns based on the observed pixels and reconstruct the entire image by projecting each column to the space
spanned by the selected column subsets.

In Figure \ref{fig_lena} we depicted the final compressed image as well as intermediate steps (e.g., subsampled pixels and selected columns)
on the $512\times 512$ 8-bit gray scale Lena standard test image.
We also report the mean and standard deviation of selection error across 10 runs under different settings of target column subset sizes in Table \ref{tab_lena}.

Table \ref{tab_lena} shows that the iterative norm sampling algorithm consistently outperforms norm sampling and so is the leverage score sampling method
when the target column subset size is large, which implies small oracle error $\|\mat M-\mat M_k\|_F^2$.
To get an intuitive sense of why this is the case, we refer the readers to the selected columns for each of the sampling algorithm as shown in Figure \ref{fig_lena} (the middle column).
It can be seen that the norm sampling algorithm (Figure \ref{fig_lena1}) oversamples columns in relatively easy regions (e.g., the white bar on the left side and the smooth part of the face)
because these regions have large pixel values (i.e., they are whiter than the other pixels) and hence have larger column norms.
In contrast, the iterative norm sampling algorithm (Figure \ref{fig_lena2})
focuses most sampled columns on the tassel and hair parts which are complicated and cannot be well approximated by other columns.
This shows that the iterative norm sampling method has the power to adapt to highly heterogeneous columns and produce better approximations.
Finally, we remark that though both leverage score sampling and iterative norm sampling have relative error guarantees,
in practice the iterative norm sampling performs much better than leverage score sampling for matrices whose rank is not very high.

\section{Discussion}\label{sec:discussion}

We discuss on several aspects of the proposed algorithms and their analysis.

\subsection{Limitation of passive sampling}\label{subsec:passive}

\begin{table}[t]
\centering
\caption{Relative selection error $\|\mat M-\mat C\mat C^\dagger\mat M\|_F/\|\mat M\|_F$ on the standard Lena test image ($512\times 512$) for
 norm sampling (\textsc{Norm}), iterative norm sampling (\text{Iter. norm}) and approximate leverage score sampling (\textsc{Lev. score}).
 Results also compared to a uniform sampling baseline (\textsc{Uniform}) and the truncated SVD lower bound (\textsc{SVD}).
 The percentage of observed entries $\alpha$ is set to $\alpha=30\%$.
 Number of columns used for reconstruction varies from 25 to 100.}
\vskip -0.05cm
\begin{tabular}{cccccc}
\hline
\abovespace\belowspace
& \textsc{Uniform}& \textsc{Norm}& \textsc{Iter. norm}& \textsc{Lev. score}& \textsc{SVD}\\
\hline
\abovespace
\textsc{25 columns}& $.151\pm .009$& $.147\pm .004$& $\mathbf{.136\pm .004}$& $.148\pm .007$& $.092$\\
\textsc{50 columns}& $.104\pm .004$ & $.103\pm .003$& $\mathbf{.092\pm .001}$& $.105\pm .003$& $.059$\\
\belowspace
\textsc{100 columns}& $.064\pm .002$& $.065\pm .001$& $\mathbf{.053\pm .001}$& $.061\pm .002$& $.032$\\
\hline
\end{tabular}
\label{tab_lena}
\end{table}

In most cases the observed entries of a partially observable matrix are sampled according to some sampling schemes.
We say a sampling scheme is \emph{passive} when the sampling distribution (i.e., probability of observing a particular matrix entry) is fixed a priori
and does not depend on the data matrix.
On the other hand, an \emph{active} sampling scheme adapts its sampling distribution according to previous observations and requests
unknown data points in a feedback driven way.
We mainly focus on active sampling methods in this paper (both Algorithm \ref{alg_additive_cx} and \ref{alg_relative_cx} perform active sampling).
However, Algorithm \ref{alg_lscore_cx} only requires passive sampling because the sampling distribution of rows is the uniform distribution
and is fixed a priori.

Passive sampling is known to work poorly for coherent matrices \citep{power-adaptivity,complete-any-matrix}.
In this section, we make the following three remarks on the power of passive sampling for column subset selection:

\paragraph{Remark 1}
The $\|\mat M-\mat C\mat X\|_{\xi}$ reconstruction error bound for column subset selection is hard for passive sampling.
In particular, it can be shown that no passive sampling algorithm achieves relative reconstruction error bound with high probability
unless it observes $\Omega(n_1n_2)$ entries of an $n_1\times n_2$ matrix $\mat M$.
This holds true even if $\mat M$ is assumed to be exact low rank and has incoherent column space.

This remark can be formalized by noting that when $\mat M$ is exact low rank then relative reconstruction error implies exact recovery of $\mat M$,
or in other words, matrix completion.
Here we cite the hardness result in \citep{power-adaptivity} for completing coherent matrix by passive sampling.
Similar results could also be obtained by applying Theorem 6 in \citep{complete-any-matrix}.
\begin{thm}[Theorem 2, \citep{power-adaptivity}]
Let $\mathcal X$ denote all $n_1\times n_2$ matrices whose rank is no more than $k$ and column space has incoherence $\mu_0$ as defined in Eq. (\ref{eq_mu_subspace}).
Fix $m<n_1n_2$ and let $\mathcal Q$ denote all passive sampling distributions over $m$ samples of $n_1n_2$ matrix entries.
Let $\mathcal F=\{f:\mathbb R^m\to\mathcal X\}$ be the collection of (possibly random) matrix completion algorithms.
We then have
\begin{equation}
R_{\complete}^* := \inf_{f\in\mathcal F}\inf_{q\in\mathcal Q}\sup_{\mat X\in\mathcal X}\Pr_{\Omega\sim q;f}[f(\Omega, \mat X_{\Omega}) \neq \mat X] \geq \frac{1}{2}-\left\lceil\frac{m}{(1-\frac{k-1}{k\mu_0})n_1}\right\rceil\frac{1}{2(n-k)},
\end{equation}
where $n=\max(n_1,n_2)$.
As a remark, when $\mu_0$ is a constant then $R_{\complete}^*=\Omega(1)$ whenever $m=o(n_1(n_2-k))$.
\label{thm_mat_complete}
\end{thm}

\paragraph{Remark 2}
For the $\|\mat M-\mat C\mat C^\dagger\mat M\|_{\xi}$ selection error (with only column indices $C$ output by an CSS algorithm),
it is possible for a passive sampling algorithm to achieve a relative error bound with high probability.
In fact, Algorithm \ref{alg_lscore_cx} and Theorem \ref{thm_lscore_cx} precisely accomplish this.
In addition, when the input matrix is exact low rank, Theorem \ref{thm_lscore_cx} implies that there exists a passive sampling algorithm
that outputs a small subset of columns which span the entire column subspace of a row-coherent matrix with high probability.
This result shows column subset selection is easier than matrix completion when only indices of the selected column subset are required.
It does not violate Theorem \ref{thm_mat_complete}, however, because knowing which columns span the column space of an input matrix
does not imply we can complete the matrix without further samples.

\paragraph{Remark 3}
Although Remark 2 and Theorem \ref{thm_lscore_cx} shows that it is possible to achieve relative $\|\mat M-\mat C\mat C^\dagger\mat M\|_F$ error bound for row coherent matrices via passive sampling,
we show in this section that passive sampling is insufficient under a slightly weaker notion of column incoherence. 
In particular, instead of assuming $\mu(\mathcal U)\leq \mu_0$ on the column space as in Eq. (\ref{eq_mu_subspace}),
we assume $\mu(\vct x_i)\leq\mu_1$ where $\mu_1$ is independent of $k$ for every column $\vct x_i$ as in Eq. (\ref{eq_mu_vector}).
Note that if $\rank(\mathcal U)=k$ and $\vct x_i\in\mathcal U$ then $\mu(\vct x_i)\leq k\mu(\mathcal U)$.
So for exact low rank matrices the vector-based incoherence assumption in Eq. (\ref{eq_mu_vector}) is weaker than the subspace-based incoherence assumption in Eq. (\ref{eq_mu_subspace}).
We then have the following theorem, which is proved in Appendix \ref{appsec:lowerbound}.
\begin{thm}
Let $\mathcal X'$ denote all $n_1\times n_2$ matrices whose rank is no more than $k$ and incoherence $\mu_1\geq1+\frac{1}{n_1-1}$ as defined in Eq. (\ref{eq_mu_vector}) for each column.
Fix $m<n_1n_2$ and let $\mathcal Q$ denote all passive sampling distributions over $m$ samples of $n_1n_2$ matrix entries.
Let $\mathcal F'=\{f:\mathbb R^m\to[n_2]^k\}$ be the collection of (possibly random) column subset selection algorithms.
We then have
\begin{equation}
R_{\css}^* := \inf_{f\in\mathcal F'}\inf_{q\in\mathcal Q}\sup_{\mat X\in\mathcal X'}\Pr_{\Omega\sim q;f}[\mat X\neq\mat X_C\mat X_C^\dagger\mat X]\geq \frac{1}{2}-\frac{m}{2n_1(n_2-k)},
\end{equation}
where $C=f(\mat X,\mat X_{\Omega})$ is the output column subset of $f$.
As a remark, the failure probability $R_{\css}^*$ satisfies $R_{\css}^*=\Omega(1)$ whenever $m=o(n_1(n_2-k))$.
\label{thm_lb_css}
\end{thm}

Theorem \ref{thm_lb_css} combined with Theorem \ref{thm_mat_complete} shows a separation of hardness between column subset selection and matrix completion.
It also formalizes the intuitive limited power of passive sampling over coherent matrices.

\subsection{Time complexity}\label{subsec:discuss_time}

\begin{table}[b]
\centering
\caption{Time complexity of proposed and baseline algorithms. 
$k$ denotes the intrinsic rank and $s$ denotes the number of selected columns.
Dependency on failure probability $\delta$ and other poly-logarithmic dependency is omitted.}
\scalebox{0.9}{
\begin{tabular}{lccccc}
\hline
\abovespace\belowspace
Algorithm& \textsc{Norm}& \textsc{Iter. norm}\textsuperscript{*}& \textsc{Lev. score}& \textsc{Block OMP}\textsuperscript{*}& \textsc{gLasso}\textsuperscript{$\dagger$}\\
\hline
\abovespace\belowspace
Time Complexity& $O(\alpha n^2)$& $O(\alpha^2sn^3)$& $O(\textrm{svd}(\alpha n,n,k))$& $O(\alpha^2sn^3)$& $O(T(n^3+s^2n^2))$\\
\hline
\multicolumn{6}{l}{\footnotesize\textsuperscript{*}Assume $\alpha n>s$ and $\alpha^2n > 1$.}\\
\multicolumn{6}{l}{\footnotesize\textsuperscript{$\dagger$}Using solution path implementation; $T$ is the desired number of $\lambda$ values.}
\end{tabular}
}
\label{tab_time}
\end{table}

In this section we report the theoretical time complexity of our proposed algorithms as well as the optimization based methods for comparison
in Table \ref{tab_time}.
We assume the input matrix $\mat M$ is square $n\times n$ and we are using $s$ columns to approximate the top-$k$ component of $\mat M$.
Let $\alpha=m/n^2$ be the percentage of observed data.
$\mathtt{svd}(a,b,c)$ denotes the time for computing the top-$c$ truncated SVD of an $a\times b$ matrix.

Suppose the observation ratio $\alpha$ is a constant and the $\mathtt{svd}$ operation takes quadratic time.
Then the time complexity for all algorithms can be sorted as
\begin{multline}
\textsc{Norm}; O(n^2) < \textsc{Lev. score}; O(kn^2) < \textsc{Iter. norm}, \textsc{Block OMP}; O(sn^3) \\
< \textsc{gLasso}, O(T(n^3+s^2n^2)).
\label{eq_time_sorted}
\end{multline}
Perhaps not surprisingly, in Section \ref{sec:snp} and \ref{sec:lena} on real-world data sets we show the reverse holds for selection error for the first three algorithms in Eq. (\ref{eq_time_sorted}).

\subsection{Sample complexity, column subset size and selection error}

We remark on the connection of sample complexity (i.e., number of observed matrix entries),
size of column subsets and reconstruction error for column subset selection.
For column subset selection when the target column subset size is fixed the sample complexity acts more like a threshold:
if not enough number of matrix entries are observed then the algorithm fails since the column norms are not accurately estimated,
but when a sufficient number of observations are available the reconstruction error does not differ much.
Such phase transition was also observed in other matrix completion/approximation tasks as well, for example, in \citep{power-adaptivity}.
In fact, the guarantee in Eq. (\ref{eq_additive_css}), for example, is exactly the same as in \citep{norm2-css} under the fully observed setting, i.e., $m_1=n_1$.

The bottom three plots in Figure \ref{fig_compare_coherent} are an excellent illustration of this phenomenon. 
When $\alpha=0.3$ the selection error of Algorithm \ref{alg_relative_cx} is very high, which means the algorithm does not have enough samples.
However, for $\alpha=0.6$ and $\alpha=0.9$ the performance of Algorithm \ref{alg_relative_cx} is very similar.

\subsection{Sample complexity of the iterative norm sampling algorithm}

{
We try to verify the sample complexity dependence on the intrinsic matrix rank $k$ 
for the iterative norm sampling algorithm (Algorithm \ref{alg_relative_cx}).
To do this, we run Algorithm \ref{alg_relative_cx} under various settings of intrinsic dimension $k$
and the sampling probability $\alpha$ (which is basically proportional to the expected number of per-column samples $m$).
We then plot the selection error $\|\mat M-\mat C\mat C^\dagger\mat M\|_F$ against $\alpha$, $\alpha/k$ and $\alpha/k^2$ in Figure \ref{fig_dependence_r}.

Theorem \ref{thm_relative_cx} states that the dependence of $m$ on $k$ should be $m=\widetilde O(k^2)$ ignoring logarithmic factors.
However, in Figure \ref{fig_dependence_r} one can observe that when the selection error is plotted against $\alpha/k$ the different curves coincide.
This suggests that the actual dependence of $m$ on $k$ should be close to linear instead of quadratic.
It is an interesting question whether we can get rid of the use of union bounds over all $n_2$-choose-$k$ column subsets in the proof of Theorem \ref{thm_relative_cx}
in order to get a near linear dependence over $k$.
Note that the curves converge to different values for different $k$ settings because selection error decreases when more columns are used to reconstruct the input matrix.
}

\subsection{Sampling with and without replacement}\label{subsec:replacement}

In the experiments we observe that for norm sampling (Algorithm \ref{alg_additive_cx}) and approximate leverage score sampling (Algorithm \ref{alg_lscore_cx})
the two column sampling schemes, i.e., sampling with and without replacement,
makes a big difference in practice (e.g., see Figure \ref{fig_compare_random}, \ref{fig_compare_coherent}, and \ref{fig_hapmap}).
In fact, sampling without replacement always outperforms sampling with replacement because under the latter scheme there is a positive probability of sampling the same column more than once.
Though we analyzed both algorithm under the sampling with replacement scheme, in practice sampling without replacement should always be used
since it makes no sense to select a column more than once.
Finally, we remark that for iterative norm sampling (Algorithm \ref{alg_relative_cx}) a column will never be picked more than once
since the (estimated) projected norm of an already selected column is zero with probability 1.

\acks{
We would like to thank Akshay Krishnamurthy for helpful discussion on the proof of Theorem \ref{thm_relative_cx}
and James Duyck for a solution path implementation of group Lasso for column subset selection.
This work is supported in part by grants NSF-1252412
and AFOSR-FA9550-14-1-0285.
}

\begin{appendices}
\section{Analysis of the active norm sampling algorithm}\label{appsec:alg1}

\begin{proof}[Proof of Lemma \ref{lem_l2css}]
This lemma is a direct corollary of Theorem 2 from \citep{norm2-css}.
First, let $P_i = \hat c_i/\hat f$ be the probability of selecting the $i$-th column of ${\mat M}$.
By assumption, we have $P_i \geq \frac{1-\alpha}{1+\alpha}\|\vct x_i\|_2^2/\|\mat M\|_F^2$.
Applying Theorem 2
\footnote{The original theorem concerns random samples of rows; it is essentially the same for random samples of columns.}
from \citep{norm2-css} we have that with probability at least $1-\delta$, 
there exists an orthonormal set of vectors $\vct y^{(1)}, \cdots, \vct y^{(k)}\in\mathbb R^{n_1}$ in $\Span(\mat C)$ such that
\begin{equation}
\left\|\mat M-\left(\sum_{j=1}^k{\vct y^{(j)}\vct y^{(j)^\top}}\right)\mat M\right\|_F^2 \leq \|\mat M-\mat M_k\|_F^2 + \frac{(1+\alpha)k}{(1-\alpha)\delta s}\|\mat M\|_F^2.
\end{equation}
Finally, to complete the proof, note that every column of $\left(\sum_{j=1}^k{\vct y^{(j)}\vct y^{(j)^\top}}\right)\mat M$
can be represented as a linear combination of columns in $\mat C$;
furthermore, 
\begin{equation}
\|\mat M - \mathcal P_C(\mat M)\|_F = \min_{\mat X\in\mathbb R^{k\times n_2}}{\|\mat M - \mat C\mat X\|_F}
\leq \left\|\mat M-\left(\sum_{j=1}^k{\vct y^{(j)}\vct y^{(j)^\top}}\right)\mat M\right\|_F.
\end{equation}
\end{proof}

\begin{proof}[Proof of Theorem \ref{thm_additive_cx}]
First, set $m_1 = \Omega(\mu_0\log(n_2/\delta_1))$ we have that with probability $\geq 1-\delta_1$ the inequality
\begin{equation*}
(1-\alpha)\|\vct x_i\|_2^2\leq\hat c_i\leq(1+\alpha)\|\vct x_i\|_2^2
\end{equation*}
holds with $\alpha=0.5$ for every column $i$, using Lemma \ref{lem_norm_estimation_additive}.
Next, putting $s\geq 6k/\delta_2\epsilon^2$ and applying Lemma \ref{lem_l2css} we get
\begin{equation}
\|\mat M-\mathcal P_C(\mat M)\|_F \leq \|\mat M-\mat M_k\|_F + \epsilon\|\mat M\|_F
\label{eq_mainthm_1}
\end{equation}
with probability at least $1-\delta_2$.
Finally, note that when $\alpha\leq 1/2$ and $n_1\leq n_2$ the bound in Lemma \ref{lem_approx} is dominated by
\begin{equation}
\|\mat M-\widehat{\mat M}\|_2 \leq \|\mat M\|_F\cdot O\left(\sqrt{\frac{\mu_0}{m_2}}\log\left(\frac{n_1+n_2}{\delta}\right)\right).
\end{equation}
Consequently, for any $\epsilon'>0$ if $m_2 = \Omega((\epsilon')^{-2}\mu_0\log^2((n_1+n_2)/\delta_3)$ we have with probability $\geq 1-\delta_3$
\begin{equation}
\|\mat M-\widehat{\mat M}\|_2 \leq \epsilon'\|\mat M\|_F.
\end{equation}
The proof is then completed by taking $\epsilon' = \epsilon/\sqrt{s}$:
\begin{eqnarray*}
\|\mat M-\mat C{\mat X}\|_F&=& \|\mat M-\mathcal P_C(\widehat{\mat M})\|_F\\
&\leq& \|\mat M-\mathcal P_C(\mat M)\|_F + \|\mathcal P_C(\mat M-\widehat{\mat M})\|_F\\
&\leq& \|\mat M-\mat M_k\|_F + \epsilon\|\mat M\|_F + \sqrt{s}\|\mathcal P_C(\mat M-\widehat{\mat M})\|_2\\
&\leq& \|\mat M-\mat M_k\|_F + \epsilon\|\mat M\|_F + \sqrt{s}\cdot \epsilon'\|\mat M\|_F\\
&\leq& \|\mat M-\mat M_k\|_F + 2\epsilon\|\mat M\|_F.
\end{eqnarray*}
\end{proof}

\section{Analysis of the iterative norm sampling algorithm}\label{appsec:alg2}

\begin{proof}[Proof of Lemma \ref{lem_preserve_incoherence}]

We first prove Eq. (\ref{eq_preserve_subspace_incoherence}).
{Observe that $\dim(\mathcal U(C))\leq s$}. 
Let $\mat R_C = (\mat R^{(C(1))},\cdots,\mat R^{(C(s))})\in\mathbb R^{n_1\times s}$ denote the selected $s$ columns in the noise matrix $\mat R$ 
and let $\mathcal R(C) = \Span(\mat R_C)$ denote the span of selected columns in $\mat R$.
By definition, $\mathcal U(C) \subseteq \mathcal U\cup \mathcal R(C)$,
where $\mathcal U=\Span(\mat A)$ denotes the subspace spanned by columns in the deterministic matrix $\mat A$.
Consequently, we have the following bound on $\|\mathcal P_{\mathcal U(C)}\vct e_i\|$ (assuming each entry in $\mat R$ follows a zero-mean Gaussian distribution with $\sigma^2$ variance):
\begin{eqnarray*}
\|\mathcal P_{\mathcal U(C)}\vct e_i\|_2^2
&\leq& \|\mathcal P_{\mathcal U}\vct e_i\|_2^2 + \|\mathcal P_{\mathcal U^\perp\cap\mathcal R(C)}\vct e_i\|_2^2\\
&\leq& \|\mathcal P_{\mathcal U}\vct e_i\|_2^2 + \|\mathcal P_{\mathcal R(C)}\vct e_i\|_2^2\\
&\leq& \frac{k\mu_0}{n_1} + \|\mat R_C\|_2^2\|(\mat R_C^\top\mat R_C)^{-1}\|_2^2\|\mat R_C^\top\vct e_i\|_2^2\\
&\leq& \frac{k\mu_0}{n_1} + \frac{(\sqrt{n_1}+\sqrt{s}+\epsilon)^2\sigma^2}{(\sqrt{n_1}-\sqrt{s}-\epsilon)^4\sigma^4}\cdot \sigma^2(s + 2\sqrt{s\log(2/\delta)} + 2\log(2/\delta)).
\end{eqnarray*}
For the last inequality we apply Lemma \ref{lem_gaussian_spectrum} to bound the largest and smallest singular values of $\mat R_C$
and Lemma \ref{lem_gaussian_2norm} to bound $\|\mat R_C^\top\vct e_i\|_2^2$, because $\mat R_C^\top\vct e_i$ follow i.i.d. Gaussian distributions
with covariance $\sigma^2\mat I_{s\times s}$.
{If $\epsilon$ is set as $\epsilon = \sqrt{2\log(4/\delta)}$ then the last inequality holds with probability at least $1-\delta$.
Furthermore, when $s\leq n_1/2$ and $\delta$ is not exponentially small (e.g., $\sqrt{2\log(4/\delta)} \leq \frac{\sqrt{n_1}}{4}$), the fraction 
$\frac{(\sqrt{n_1}+\sqrt{s}+\epsilon)^2}{(\sqrt{n_1}-\sqrt{s}-\epsilon)^4}$ 
is approximately $O(1/n_1)$. }
As a result, with probability $1-n_1\delta$ the following holds:
\begin{multline}
\mu(\mathcal U(C)) = \frac{n_1}{s}\max_{1\leq i\leq n_1}{\|\mathcal P_{\mathcal U(C)}\vct e_i\|_2^2}\\
\leq \frac{n_1}{s}\left(\frac{k\mu_0}{n_1} + O\left(\frac{s+\sqrt{s\log(1/\delta)}+\log(1/\delta)}{n_1}\right)\right)
= O\left(\frac{k\mu_0 + s + \sqrt{s\log(1/\delta)} + \log(1/\delta)}{s}\right).
\end{multline}
Finally, putting $\delta' = n_1/\delta$ we prove Eq. (\ref{eq_preserve_subspace_incoherence}).

Next we try to prove Eq. (\ref{eq_preserve_project_incoherence}).
Let $\vct x$ be the $i$-th column of $\mat M$ and write $\vct x = \vct a + \vct r$,
where $\vct a = \mathcal P_{\mathcal U}(\vct x)$ and $\vct r=\mathcal P_{\mathcal U^\perp}(\vct x)$.
Since the deterministic component of $\vct x$ lives in $\mathcal U$ and the random component of $\vct x$ is a vector with each entry sampled from i.i.d. zero-mean Gaussian distributions,
we know that $\vct r$ is also a zero-mean random Gaussian vector with i.i.d. sampled entries.
Note that $\mathcal U(C)$ does not depend on the randomness over $\{\mat M^{(i)}:i\notin C\}$.
Therefore, in the following analysis we will assume $\mathcal U(C)$ to be a fixed subspace $\widetilde{\mathcal U}$ with dimension at most $s$.

The projected vector $\vct x'=\mathcal P_{\widetilde{\mathcal U}^\perp}\vct x$ can be written as $\tilde{\vct x} = \tilde{\vct a}+\tilde{\vct r}$,
where $\tilde{\vct a} = \mathcal P_{\widetilde{\mathcal U}^\perp}\vct a$ and $\tilde{\vct r}=\mathcal P_{\widetilde{\mathcal U}^\perp}\vct r$.
By definition, $\tilde{\vct a}$ lives in the subspace $\mathcal U\cap\widetilde{\mathcal U}^\perp$.
So it satisfies the incoherence assumption
\begin{equation}
\mu(\tilde{\vct a}) = \frac{n_1\|\tilde{\vct a}\|_{\infty}^2}{\|\tilde{\vct a}\|_2^2} \leq k\mu(\mathcal U) \leq k\mu_0.
\end{equation}
{On the other hand, because $\tilde{\vct r}$ is an orthogonal projection of some random Gaussian variable,
$\tilde{\vct r}$ is still a Gaussian random vector,
which lives in $\mathcal U^\perp\cap\widetilde{\mathcal U}^\perp$ with rank at least $n_1-k-s$.}
Subsequently, we have
\begin{eqnarray*}
\mu(\tilde{\vct x}) &=& n_1\frac{\|\tilde{\vct x}\|_{\infty}^2}{\|\tilde{\vct x}\|_2^2} \leq 3n_1\frac{\|\tilde{\vct a}\|_{\infty}^2+\|\tilde{\vct r}\|_{\infty}^2}{\|\tilde{\vct a}\|_2^2 + \|\tilde{\vct r}\|_2^2}\\
&\leq& 3n_1\frac{\|\tilde{\vct a}\|_{\infty}^2}{\|\tilde{\vct a}\|_2^2} + 3n_1\frac{\|\tilde{\vct r}\|_{\infty}^2}{\|\tilde{\vct r}\|_2^2}\\
&\leq& 3k\mu_0 + \frac{6\sigma^2n_1\log(2n_1n_2/\delta)}{\sigma^2(n_1-k-s)-2\sigma^2\sqrt{(n_1-k-s)\log(n_2/\delta)}}. 
\end{eqnarray*}
For the second inequality we use the fact that $\frac{\sum_i{a_i}}{\sum_i{b_i}} \leq \sum_i{\frac{a_i}{b_i}}$ whenever $a_i,b_i\geq 0$.
For the last inequality we use Lemma \ref{lem_gaussian_infty} on the enumerator and Lemma \ref{lem_gaussian_2norm} on the denominator.
Finally, note that when $\max(s,k)\leq n_1/4$ and $\log(n_2/\delta) \leq n_1/64$ the denominator can be lower bounded by $\sigma^2 n_1/4$;
subsequently, we can bound $\mu(\tilde{\vct x})$ as
\begin{equation}
\mu(\tilde{\vct x}) \leq 3k\mu_0+\frac{24\sigma^2n_1\log(2n_1n_2/\delta)}{\sigma^2n_1}
 \leq 3k\mu_0 + 24\log(2n_1n_2/\delta).
\end{equation}
Taking a union bound over all $n_2-s$ columns yields the result.

\end{proof}

To prove the norm estimation consistency result in Lemma \ref{lem_norm_estimation}
we first cite a seminal theorem from \citep{power-adaptivity}
which provides a tight error bound on a subsampled projected vector
in terms of the norm of the true projected vector.
\begin{thm}
Let $\mathcal U$ be a $k$-dimensional subspace of $\mathbb R^n$ and $\vct y=\vct x+\vct v$,
where $\vct x\in\mathcal U$ and $\vct v\in\mathcal U^\perp$.
Fix $\delta' > 0$, $m \geq \max\{\frac{8}{3}k\mu(\mathcal U)\log\left(\frac{2k}{\delta'}\right), 4\mu(\vct v)\log(1/\delta')\}$
and let $\Omega$ be an index set with entries sampled uniformly with replacement with probability $m/n$.
Then with probability at least $1-4\delta'$:
\begin{equation}
\frac{m(1-\alpha)-k\mu(\mathcal U)\frac{\beta}{1-\gamma}}{n}\|\vct v\|_2^2 
\leq \|\vct y_{\Omega}-\mathcal P_{U_{\Omega}}\vct y_{\Omega}\|_2^2
\leq (1+\alpha)\frac{m}{n}\|\vct v\|_2^2,
\end{equation}
where $\alpha = \sqrt{2\frac{\mu(\vct v)}{m}\log(1/\delta')} + 2\frac{\mu(\vct v)}{3m}\log(1/\delta')$,
$\beta = (1+2\sqrt{\log(1/\delta')})^2$
and $\gamma = \sqrt{\frac{8k\mu(\mathcal U)}{3m}\log(2k/\delta')}$.
\label{thm_subspace_detection}
\end{thm}

We are now ready to prove Lemma \ref{lem_norm_estimation}.
\begin{proof}[Proof of Lemma \ref{lem_norm_estimation}]
By Algorithm \ref{alg_relative_cx}, we know that $\dim(\mathcal S_t) = t$ with probability 1.
Let $\vct y = \mat M^{(i)}$ denote the $i$-th column of $\mat M$ and let $\vct v = \mathcal P_{\mathcal S_t}\vct y$
be the projected vector.
We can apply Theorem \ref{thm_subspace_detection} to bound the estimation error between $\|\vct v\|$ and $\|\vct y_{\Omega}-\mathcal P_{\mathcal S_t(\Omega)}\vct y_{\Omega}\|$.

First, when $m$ is set as in Eq. (\ref{eq_m_lowerbound}) it is clear that the conditions
$m\geq \frac{8}{3}t\mu(\mathcal U)\log\left(\frac{2t}{\delta'}\right) = \Omega(k\mu_0\log(n/\delta)\log(k/\delta'))$ 
and $m \geq 4\mu(\vct v)\log(1/\delta') = \Omega(k\mu_0\log(n/\delta)\log(1/\delta'))$
are satisfied.
We next turn to the analysis of $\alpha$, $\beta$ and $\gamma$.
More specifically, we want $\alpha = O(1)$, $\gamma  = O(1)$ and $\frac{t\mu(\mathcal U)}{m}\beta = O(1)$.

For $\alpha$, $\alpha = O(1)$ implies $m = \Omega(\mu(\vct v)\log(1/\delta')) = \Omega(k\mu_0\log(n/\delta)\log(1/\delta'))$.
Therefore, by carefully selecting constants in $\Omega(\cdot)$ we can make $\alpha\leq 1/4$.

For $\gamma$, $\gamma = O(1)$ implies $m = \Omega(t\mu(\mathcal U)\log(t/\delta')) = \Omega(k\mu_0\log(n/\delta)\log(k/\delta'))$.
By carefully selecting constants in $\Omega(\cdot)$ we can make $\gamma \leq 0.2$.

For $\beta$, $\frac{t\mu(\mathcal U)}{m}\beta = O(1)$ implies $m=O(t\mu(\mathcal U)\beta) = O(k\mu_0\log(n/\delta)\log(1/\delta'))$.
By carefully selecting constants we can have $\beta \leq 0.2$.
Finally, combining bounds on $\alpha$, $\beta$ and $\gamma$ we prove the desired result.
\end{proof}

Before proving Lemma \ref{lem_approximate_volume_sampling}, we first cite a lemma from \citep{volume-sampling-css}
that connects the volume of a simplex to the permutation sum of singular values.
\begin{lem}[\citep{volume-sampling-css}]
Fix $\mat A\in\mathbb R^{m\times n}$ with $m\leq n$.
Suppose $\sigma_1,\cdots,\sigma_m$ are singular values of $\mat A$.
Then
\begin{equation}
\sum_{S\subseteq[n],|S|=k}{\vol(\Delta(S))^2} = \frac{1}{(k!)^2}\sum_{1\leq i_1<i_2<\cdots<i_k\leq m}{\sigma_{i_1}^2\sigma_{i_2}^2\cdots\sigma_{i_k}^2}.
\end{equation}
\label{lem_volume_perm}
\end{lem}

Now we are ready to prove Lemma \ref{lem_approximate_volume_sampling}.
\begin{proof}[Proof of Lemma \ref{lem_approximate_volume_sampling}]
Let $\mat M_k$ denote the best rank-$k$ approximation of $\mat M$ and assume the singular values of $\mat M$ are $\{\sigma_i\}_{i=1}^{n_1}$.
Let $C = \{i_1,\cdots,i_k\}$ be the selected columns.
Let $\tau\in\Pi_k$, where $\Pi_k$ denotes all permutations with $k$ elements.
By $\mathcal H_{\tau,t}$ we denote the linear subspace spanned by $\{\mat M^{(\tau(i_1))},\cdots,\mat M^{(\tau(i_t))}\}$
and let $d(\mat M^{(i)}, \mathcal H_{\tau,t})$ denote the distance between column $\mat M^{(i)}$ and subspace $\mathcal H_{\tau,t}$.
We then have 
\begin{eqnarray*}
\hat p_C &\leq& \sum_{\tau\in\Pi_k}{\left(\frac{5}{2}\right)^k\frac{\|\mat M^{(\tau(i_1))}\|_2^2}{\|\mat M\|_F^2}\frac{d(\mat M^{(\tau(i_2))},\mathcal H_{\tau,1})^2}{\sum_{i=1}^{n_2}{d(\mat M^{(i)}, \mathcal H_{\tau,1})^2}}\cdots\frac{d(\mat M^{(\tau(i_k))},\mathcal H_{\tau,k-1})^2}{\sum_{i=1}^{n_2}{d(\mat M^{(i)}, \mathcal H_{\tau,k-1})^2}}}\\
&\leq& 2.5^k\cdot \frac{\sum_{\tau\in\Pi_k}{\|\mat M^{(\tau(i_1))}\|^2d(\mat M^{(\tau(i_2))}, \mathcal H_{\tau,1})^2\cdots d(\mat M^{(\tau(i_k))}, \mathcal H_{\tau,k-1})^2}}{\|\mat M\|_F^2\|\mat M-\mat M_1\|_F^2\cdots\|\mat M-\mat M_{k-1}\|_F^2}\\
&=& 2.5^k\cdot \frac{\sum_{\tau\in\Pi_k}{(k!)^2\vol(\Delta(C))^2}}{\|\mat M\|_F^2\|\mat M-\mat M_1\|_F^2\cdots\|\mat M-\mat M_{k-1}\|_F^2}\\
&=& 2.5^k\cdot \frac{(k!)^3\vol(\Delta(C))^2}{\sum_{i=1}^{n_1}{\sigma_i^2}\sum_{i=2}^{n_1}{\sigma_i^2}\cdots \sum_{i=k}^{n_1}{\sigma_i^2}}\\
&\leq& 2.5^k\cdot \frac{(k!)^3\vol(\Delta(C))^2}{\sum_{1\leq i_1<i_2<\cdots<i_k\leq n_1}{\sigma_{i_1}^2\sigma_{i_2}^2\cdots\sigma_{i_k}^2}}\\
&=& 2.5^k\cdot \frac{k!\vol(\Delta(C))^2}{\sum_{T:|T|=k}{\vol(\Delta(T))^2}} = 2.5^kk!p_C.
\end{eqnarray*}
For the first inequality we apply Eq. (\ref{eq_norm_estimation_uniform})
and for the second to last inequality we apply Lemma \ref{lem_volume_perm}.
\end{proof}

Lemma \ref{lem_adaptive_exponential} can be proved by applying Theorem \ref{thm_adaptive_norm_lemma} for $T$ rounds,
given the norm estimation accuracy bound in Proposition \ref{prop_norm_estimation_ver3}.
\begin{proof}[Proof of Lemma \ref{lem_adaptive_exponential}]
First note that 
$$
\|\mat M-\mathcal P_{\mathcal U\cup\mathcal S_1\cup\cdots\cup\mathcal S_T}(\mat M)\|_F^2
\leq \|\mat M-\mathcal P_{\mathcal U\cup\mathcal S_1\cup\cdots\cup\mathcal S_T, k}(\mat M)\|_F^2.
$$
Applying Theorem \ref{thm_adaptive_norm_lemma} with $\frac{1+\alpha}{1-\alpha} = \frac{5}{2}$, we have
\begin{eqnarray*}
&&\mathbb E\left[\|\mat M-\mathcal P_{\mathcal U\cup\mathcal S_1\cup\cdots\cup\mathcal S_T}(\mat M)\|_F^2\right]\\
&\leq& \|\mat M-\mat M_k\|_F^2 + \frac{5k}{2s_T}\mathbb E\left[\|\mat M-\mathcal P_{\mathcal U\cup\mathcal S_1\cup\cdots\mathcal S_{T-1}}(\mat M)\|_F\right]^2\\
&\leq& \|\mat M-\mat M_k\|_F^2 + \frac{5k}{2s_T}\left(\|\mat M-\mat M_k\|_F^2 + \frac{5k}{2s_{T-1}}\mathbb E\left[\|\mat M-\mathcal P_{\mathcal U\cup\mathcal S_1\cup\cdots\mathcal S_{T-2}}(\mat M)\|_F^2\right]\right)\\
&\leq& \cdots\\
&\leq& \left(1 + \frac{5}{2}\frac{k}{s_T} + \left(\frac{5}{2}\right)^2\frac{k^2}{s_Ts_{T-1}} + \cdots + \left(\frac{5}{2}\right)^{T-1}\frac{k^{T-1}}{s_{T-1}\cdots s_1}\right)\|\mat M-\mat M_k\|_F^2\\
&& + \left(\frac{5}{2}\right)^T\frac{k^T}{s_Ts_{T-1}\cdots s_1}\|\mat M-\mathcal P_{\mathcal U}(\mat M)\|_F^2\\
&\leq& \left(1+\frac{\epsilon}{4\delta}+\frac{\epsilon}{20\delta}+\cdots\right)\|\mat M-\mat M_k\|_F^2 + \frac{\epsilon/2}{2^T\delta}\|\mat E\|_F^2\\
&\leq& \left(1+\frac{\epsilon}{2\delta}\right)\|\mat M-\mat M_k\|_F^2 + \frac{\epsilon/2}{2^T\delta}\|\mat E\|_F^2.
\end{eqnarray*}
Finally applying Markov's inequality we complete the proof.
\end{proof}

To prove the reconstruction error bound in Lemma \ref{lem_relative_matrix_approximation}
we need the following two technical lemmas, cited from \citep{akshay-nips,subspace-detection}.
\begin{lem}[\citep{akshay-nips}]
Suppose $\mathcal U\subseteq \mathbb R^{n}$ has dimension $k$ and $\mat U\in\mathbb R^{n\times k}$ is the orthogonal matrix associated with $\mathcal U$.
Let $\Omega\subseteq [n]$ be a subset of indices each sampled from i.i.d. Bernoulli distributions with probability $m/n_1$.
Then for some vector $\vct y\in\mathbb R^n$,
with probability at least $1-\delta$:
\begin{equation}
\|\mat U_{\Omega}^\top\vct y_{\Omega}\|_2^2 \leq \beta\frac{m}{n_1}\frac{k\mu(\mathcal U)}{n_1}\|\vct y\|_2^2,
\end{equation}
where $\beta$ is defined in Theorem \ref{thm_subspace_detection}.
\label{lem_uy_bound}
\end{lem}

\begin{lem}[\citep{subspace-detection}]
With the same notation in Lemma \ref{lem_uy_bound} and Theorem \ref{thm_subspace_detection}.
With probability $\geq 1-\delta$ one has
\begin{equation}
\|(\mat U_{\Omega}^\top\mat U_{\Omega})^{-1}\| \leq \frac{n_1}{(1-\gamma)m},
\end{equation}
provided that $\gamma < 1$.
\label{lem_uinv_bound}
\end{lem}

Now we can prove Lemma \ref{lem_relative_matrix_approximation}.
\begin{proof}[Proof of Lemma \ref{lem_relative_matrix_approximation}]
Let $\mathcal U = \mathcal U(S)$ and $\mat U\in\mathbb R^{n_1\times s}$ be the orthogonal matrix associated with $\mathcal U$. 
Fix a column $i$ and let $\vct x = \mat M^{(i)} = \vct a+\vct r$, where $\vct a\in\mathcal U$ and $\vct r\in\mathcal U^\perp$.
What we want is to bound $\|\vct x-\mat U(\mat U_{\Omega}^\top\mat U_{\Omega})^{-1}\mat U_{\Omega}^\top\vct x_{\Omega}\|_2^2$ in terms of $\|\vct r\|_2^2$.

Write $\vct a = \mat U\vct{\tilde a}$.
By Lemma \ref{lem_uinv_bound}, if $m$ satisfies the condition given in the Lemma then with probability over $1-\delta-\delta''$
we know $(\mat U_{\Omega}^\top\mat U_{\Omega})$ is invertible and furthermore, $\|(\mat U_{\Omega}^\top\mat U_{\Omega})^{-1}\|_2\leq 2n_1/m$.
Consequently,
\begin{equation}
\mat U(\mat U_{\Omega}^\top\mat U_{\Omega})^{-1}\mat U_{\Omega}^\top\vct a_{\Omega} = \mat U(\mat U_{\Omega}^\top\mat U_{\Omega})^{-1}\mat U_{\Omega}^\top\mat U_{\Omega}\tilde{\vct a}
= \mat U\tilde{\vct a} = \vct a.
\end{equation}
That is, the subsampled projector preserves components of $\vct x$ in subspace $\mathcal U$.

Now let's consider the noise term $\vct r$.
By Corollary \ref{cor_preserve_incoherence_uniform} with probability $\geq 1-\delta$ we can bound the incoherence level of $\vct y$ as 
$\mu(\vct y) = O(s\mu_0\log(n/\delta))$.
The incoherence of subspace $\mathcal U$ can also be bounded as $\mu(\mathcal U) =O(\mu_0\log(n/\delta))$.
Subsequently, given $m=\Omega(\epsilon^{-1}s\mu_0\log(n/\delta)\log(n/\delta''))$ we have (with probability $\geq 1-\delta-2\delta''$)
\begin{eqnarray*}
& &\|\vct x-\mat U(\mat U_{\Omega}^\top\mat U_{\Omega})^{-1}\mat U_{\Omega}^\top(\vct a+\vct r)|_2^2\\
&=& \|\vct a + \vct r -\mat U(\mat U_{\Omega}^\top\mat U_{\Omega})^{-1}\mat U_{\Omega}^\top(\vct a+\vct r)\|_2^2\\
&=& \|\vct r - \mat U(\mat U_{\Omega}^\top\mat U_{\Omega})^{-1}\mat U_{\Omega}^\top\vct r\|_2^2\\
&\leq& \|\vct r\|_2^2 + \|(\mat U_{\Omega}^\top\mat U_{\Omega})^{-1}\|_2^2\|\mat U_{\Omega}^\top\vct r\|_2^2\\
&\leq& (1+O(\epsilon))\|\vct r\|_2^2.
\end{eqnarray*}
For the second to last inequality we use the fact that $\vct r\in\mathcal U^{\perp}$.
By carefully selecting constants in Eq. (\ref{eq_m_lowerbound_uniform}) we can make
\begin{equation}
\|\vct x-\mat U(\mat U_{\Omega}^\top\mat U_{\Omega})^{-1}\mat U_{\Omega}^\top\vct x\|_2^2 \leq (1+\epsilon)\|\mathcal P_{\mathcal U^\perp}\vct x\|_2^2.
\end{equation}
Summing over all $n_2$ columns yields the desired result.
\end{proof}

\section{Proof of lower bound for passive sampling}\label{appsec:lowerbound}

\begin{proof}[Proof of Theorem \ref{thm_lb_css}]
Let $\widetilde{\mathcal X}=\{\mat X_1,\cdots,\mat X_T\}\subseteq \mathcal X'$
be a finite subset of $\mathcal X'$ which we specify later.
Let $\pi$ be any prior distribution over $\widetilde{\mathcal X}$.
We then have the following chain of inequalities:
\begin{eqnarray}
R_{\css}^* &=& \inf_{f\in\mathcal F'}\inf_{q\in\mathcal Q}\sup_{\mat X\in\mathcal X'}\Pr_{\Omega\sim q;f}[\mat X\neq\mat X_C\mat X_C^\dagger\mat X]\nonumber\\
&\geq& \inf_{f\in\mathcal F'}\inf_{q\in\mathcal Q}\Pr_{\Omega\sim q;\mat X\sim\pi;f}[\mat X\neq\mat X_C\mat X_C^\dagger\mat X]\label{eq_lbcss_eq1}\\
&\geq& \inf_{f\in\mathcal F'}\min_{|\Omega|=m}\Pr_{\mat X\sim\pi;f}[\mat X\neq\mat X_C\mat X_C^\dagger\mat X]\label{eq_lbcss_eq2}.
\end{eqnarray}
Here Eq. (\ref{eq_lbcss_eq1}) uses the fact that the maximum dominates any expectation over the same set
and for Eq. (\ref{eq_lbcss_eq2}) we apply Yao's principle, which asserts that the worst-case performance of a randomized algorithm
is better (i.e., lower bounded) by the averaging performance of a deterministic algorithm.
Hence, when the input matrix $\mat X$ is randomized by a prior $\pi$ it suffices to consider only deterministic sampling schemes,
which corresponds to a subset of matrix entries $\Omega$ fixed a priori, with size $|\Omega|=m$.

We next construct the subset $\widetilde{\mathcal X}$ and let $\pi$ be the uniform distribution over $\widetilde{\mathcal X}$.
Let $\vct x_1,\cdots,\vct x_{k-2}\in\mathbb R^{n_1}$ be an arbitrary set of linear independent column vectors with $[\vct x_i]_1=0$ for all $i=1,2,\cdots,k=2$
and $\mu(\vct x_1),\cdots,\mu(\vct x_{k-2}) = 1+\frac{1}{n_1-1}$.
This can be done by setting all nonzero entries in $\vct x_1,\cdots,\vct x_{k-2}$ to $\pm 1$.
In addition, we define $\vct y:=(1,1,\cdots,1)$ and $\vct e_j=(0,\cdots,0,1,0,\cdots,0)$ with the only nonzero entry at the $j$th position.
Next, define $\widetilde{\mat X}=\{\mat X^{i,j}\}_{i=k-1,j=1}^{n_2,n_1}$ with
\begin{equation}
\mat X^{i,j}_{(\ell)} = \left\{\begin{array}{ll}
\vct x_{\ell}& \text{if }\ell \leq k-2,\\
\vct y - 2\vct e_j& \text{if }\ell = i,\\
\vct y& \text{otherwise.}
\end{array}\right.
\end{equation}

It follows by definition that $\rank(\mat X^{i,j})=k$ and $\mu(\mat X^{i,j}_{(\ell)})\leq \mu_1=1+\frac{1}{n_1-1}$ for all $i,j$ and $\ell$.
Furthermore, for fixed $i$ and $j$ one necessary condition for $\mat X=\mat X_C\mat X_C^\dagger\mat X$ is $\{1,2,\cdots,k-2,i\}\subseteq C$.
Therefore, if for distinct $i_1,i_2,i_3,i_4$ and some $j_1,j_2,j_3,j_4$ one has $\mat X^{i_1,j_1}_{\Omega} = \cdots = \mat X^{i_4,j_4}_{\Omega}$
then the best a column subset selection algorithm $f$ could do is random guessing and hence $\Pr[\mat X\neq\mat X_C\mat X_C^\dagger\mat X]\geq 1/2$.
Consequently, for fixed $\Omega$ one has
\begin{equation}
\inf_{f\in\mathcal F'}\Pr_{\mat X\sim\pi;f}[\mat X\neq\mat X_C\mat X_C^\dagger\mat X]
\geq \frac{1}{2} - \frac{1}{2}\bigg|\left\{\mat X^{i,j}: \mat X^{i',j'}_\Omega\neq \mat X^{i,j}, \forall i'\neq i,j'\in[n_1]\right\}\bigg|.
\label{eq_lb_key}
\end{equation}

The final step is to bound the size of the set $E=\{\mat X^{i,j}: \mat X^{i',j'}_\Omega\neq \mat X^{i,j}, \forall i'\neq i,j'\in[n_1]\}$.
Note that if $\mat X_{\Omega}$ is $+1$ on all entries $(i,j)$ with $i>k-2$ then $\mat X\notin E$ because for every $\mat X'\in\widetilde{\mathcal X}$, $\mat X'_{\Omega}=\mat X_{\Omega}$.
Consequently,
\begin{equation}
\big|E\big| \leq \frac{|\Omega|}{n_1(n_2-k+2)} \leq \frac{m}{n_1(n_2-k)}.
\label{eq_lb_E}
\end{equation}
Plugging Eq. (\ref{eq_lb_E}) into Eq. (\ref{eq_lb_key}) we complete the proof of Theorem \ref{thm_lb_css}.

\end{proof}

\section{Some concentration inequalities}\label{appsec:concentration}

\begin{lem}[\citep{gaussian-2norm-bound}]
Let $X\sim\chi_d^2$. Then with probability $\geq 1-2\delta$ the following holds:
\begin{equation}
-2\sqrt{d\log(1/\delta)} \leq X-d\leq 2\sqrt{d\log(1/\delta)} + 2\log(1/\delta).
\end{equation}
\label{lem_gaussian_2norm}
\end{lem}

\begin{lem}
Let $X_1,\cdots,X_n\sim\nml(0,\sigma^2)$. Then with probability $\geq 1-\delta$ the following holds:
\begin{equation}
\max_i{|X_i|}\leq \sigma\sqrt{2\log(2n/\delta)}.
\end{equation}
\label{lem_gaussian_infty}
\end{lem}

\begin{lem}[\citep{gaussian-spectrum-bound}]
Let $\mat X$ be an $n\times t$ random matrix with i.i.d. standard Gaussian random entries.
If $t<n$ then for every $\epsilon\geq 0$ with probability $\geq 1-2\exp(-\epsilon^2/2)$ the following holds:
\begin{equation}
\sqrt{n}-\sqrt{t}-\epsilon \leq \sigma_{\min}(\mat X) \leq \sigma_{\max}(\mat X)\leq \sqrt{n}+\sqrt{t}+\epsilon.
\end{equation}
\label{lem_gaussian_spectrum}
\end{lem}

\begin{lem}[Noncommutative Bernstein Inequality, \citep{nc-bernstein-1,simpler-matrix-completion}]
Let $\mat X_1,\cdots,\mat X_m$ be independent zero-mean square $n\times n$ random matrices.
Suppose $\rho_k^2 = \max(\|\mathbb E[\mat X_k\mat X_k^\top]\|_2, \|\mathbb E[\mat X_k^\top\mat X_k]\|_2)$
and $\|\mat X_k\|_2\leq M$ with probability 1 for all $k$.
Then for any $t>0$, 
\begin{equation}
\Pr\left[\left\|\sum_{k=1}^m{\mat X_k}\right\|_2 > t\right] \leq 2n\exp\left(-\frac{t^2/2}{\sum_{k=1}^m{\rho_k^2} + Mt/3}\right).
\label{eq_noncommutative_bernstein}
\end{equation}
\label{lem_noncommutative_bernstein}
\end{lem}

\end{appendices}


\bibliography{css-sampling}

\end{document}